\documentclass[article]{siamonline0516}

\usepackage{amsfonts}
\usepackage{graphicx}
\usepackage{epstopdf}
\usepackage{algorithmic}
\usepackage{upgreek}
\usepackage{multirow}
\usepackage{booktabs}

\ifpdf
  \DeclareGraphicsExtensions{.eps,.pdf,.png,.jpg}
\else
  \DeclareGraphicsExtensions{.eps}
\fi

\newcommand{\TheTitle}{Template shape estimation: correcting an asymptotic bias} 
\newcommand{\TheAuthors}{Nina Miolane, Susan Holmes, Xavier Pennec}

\headers{\TheTitle}{\TheAuthors}

\title{{\TheTitle}\thanks{Accepted for publication in SIAM Journal of Imaging Science. Submitted to the editors on January 23, 2017. This work was funded by the virtual lab Inria@SiliconValley.}}

\author{
  Nina Miolane\thanks{INRIA, Asclepios project-team, 2004 Route des Lucioles, 06902 Sophia Antipolis, France, \email{nina.miolane@inria.fr}}
  \and
  Susan Holmes\thanks{Stanford University, Department of Statistics, Sequoia Hall, Serra Mall, Stanford CA}
  \and
  Xavier Pennec\footnotemark[2]
}

\usepackage{amsopn}

\ifpdf
\hypersetup{
  pdftitle={\TheTitle},
  pdfauthor={\TheAuthors}
}
\fi

\begin{document}

 
\maketitle

\begin{abstract}
We use tools from geometric statistics to analyze the usual estimation procedure of a template shape. This applies to shapes from landmarks, curves, surfaces, images etc. We demonstrate the asymptotic bias of the template shape estimation using the stratified geometry of the shape space. We give a Taylor expansion of the bias with respect to a parameter $\sigma$ describing the measurement error on the data. We propose two bootstrap procedures that quantify the bias and correct it, if needed. They are applicable for any type of shape data. We give a rule of thumb to provide intuition on whether the bias has to be corrected. This exhibits the parameters that control the bias' magnitude. We illustrate our results on simulated and real shape data.
\end{abstract}



\section*{Introduction}

The shape of a set of points, the shape of a signal, the shape of a surface, or the shapes in an image can be defined as the remainder after we have filtered out the position and the orientation of the object \cite{Kendall1984}. Statistics on shapes appear in many fields. Paleontologists combine shape analysis of monkey skulls with ecological and biogeographic data to understand how the \textit{skull shapes} have changed in space and time during evolution \cite{Elewa2012}. Molecular Biologists study how \textit{shapes of proteins} are related to their function. Statistics on misfolding of proteins is used to understand diseases, like Parkinson's disease \cite{Li2008}. Orthopaedic surgeons analyze \textit{bones' shapes} for surgical pre-planning \cite{Darmante2014}. In Signal processing, the \textit{shape of neural spike trains} correlates with arm movement \cite{Kurtek2011}. In Computer Vision, classifying \textit{shapes of handwritten digits} enables automatic reading of texts \cite{Allassonniere2015a}. In Medical Imaging and more precisely in Neuroimaging, studying \textit{brain shapes} as they appear in the MRIs facilitates discoveries on diseases, like Alzheimer \cite{Lorenzi2011}.

What do these applications have in common? Position and orientation of the skulls, proteins, bones, neural spike trains, handwritten digits or brains do not matter for the studies' goal: only \textit{shapes} matter. Mathematically, the study analyses the statistical distributions of \textit{the equivalence classes of the data} under translations and rotations. They project the data in a quotient space, called the \textit{shape space}.

The simplest - and most widely used - method for summarizing shapes is the computation of the mean shape. Almost all neuroimaging studies start with the computation of the mean brain shape \cite{Evans2012} for example. One refers to the mean shape with different terms depending on the field: mean configuration, mean pattern, template, atlas, etc. The mean shape is an average of \textit{equivalence classes of the data}: one computes the mean after projection of the data in the shape space. One may wonder if the projection biases the statistical procedure. This is a legitimate question as any bias introduced in this step would make the conclusions of the study less accurate. If the mean brain shape is biased, then neuroimaging's inferences on brain diseases will be too. This paper shows that a bias is indeed introduced in the mean shape estimation under certain conditions.

\subsection*{Related work}
We review papers on the shape space's geometry as a quotient space, and existing results on the mean shape's bias.

\textbf{Shapes of landmarks: Kendall analyses} The theory for shapes \textit{of landmarks} was introduced by Kendall in the 1980's \cite{Kendall1977}. He considered shapes of $k$ labeled landmarks in $\mathbb{R}^m$. The size-and-shape space, written $S\Sigma_m^k$, takes also into account the overall size of the landmarks' set. The shape space, written $\Sigma_m^k$, quotients by the size as well. Both $S\Sigma_m^k$ and $\Sigma_m^k$ have a Riemannian geometry, whose metrics are given in \cite{Le1993}. These studies model the probability distribution of the data directly in the shape space $\Sigma_m^k$. They do not consider that the data are observed in the space of landmarks $\left(\mathbb{R}^m\right)^k$ and projected in the shape space $\Sigma_m^k$. The question of bias is not raised.

We emphasize the distinction between "form" and "shape". "Form" relates to the quotient of the object by rotations and translations only. "Shape" denotes the quotient of the object by rotations, translations, and scalings. Kendall shape spaces refer to "shape": the scalings are quotiented by constraining the size of the landmarks' set to be 1.

\textbf{Shapes of landmarks: Procrustean analyses}
Procrustean analysis is related to Kendall shape spaces but it also considers shapes of landmarks \cite{Goodall1991, Dryden1998, Gower2004}. Kendall analyses project the data in the shape space by explicitly computing their coordinates in $\Sigma_m^k$. In contrast, Procrustean analyses keep the coordinates in $\left(\mathbb{R}^m\right)^k$: they project the data in the shape space by "aligning" or "registering" them. Orthogonal Procrustes analysis "aligns" the sets of landmarks by rotating each set to minimize the Euclidean distance to the other sets. Procrustean analysis considers the fact that the data are observed in the space $\left(\mathbb{R}^m\right)^k$ but does not consider the geometry of the shape space. 

The mean "shape" was shown to be consistent for shapes of landmarks in 2D and 3D in \cite{Lele1993, Le1998}. Such studies have a generative model with a scaling component $\alpha$ and a size constraint in the mean "shape" estimation procedure, which prevents the shapes from collapsing to 0 during registration. In constrast, the mean "form" - i.e. without considering scalings - is shown to be inconsistent in \cite{Lele1993} with an reducto ad absurdum proof. However, this proof does not give any geometric intuition about how to control or correct the phenomenon. More recently, similar inconsistency effects have been observed in \cite{Dryden2015}, showing that implementing ordinary Procrustes analysis without taking into account noise on the landmarks may compromise inference. The authors propose a conditional scoring method for matching configurations in order to guarantee consistency.

\textbf{Shapes of curves}
Curve data are projected in their shape space by an alignment step \cite{Joshi2006}, in the spirit of a Procrustean analysis. The bias of the mean shape is discussed in the literature. Unbiasedness was shown for shapes of signals in \cite{Kurtek2011} but under the simplifying assumption of no measurement error on the data. Some authors provide examples of bias when there is measurement error \cite{Allassonniere2007}. Their experiments show that the mean signal shape may converge to pure noise when the measurement error on simulated signals increases. The bias is proven in \cite{Bigot2011} for curves estimated from a finite number of points in the presence of error. But again, no geometric intuition nor correction strategy is given.

\textbf{Abstract shape spaces} 
\cite{Huckemann2010} studies statistics on abstract shape spaces: the shapes are defined as equivalence classes of objects in a manifold $M$ under the isometric action of a Lie group $G$. This unifies the theory for shapes of landmarks, of curves and of surfaces described above. \cite{Huckemann2010} introduces a generalization of Principal Component Analysis to such shape spaces and does not compute the mean shape as the 0-dimensional principal subspace. Therefore, the bias on the mean shape is not considered. 

But in the same abstract setting, \cite{Miolane2015b} shows the bias of the mean shape, in the special case of a finite-dimensional flat manifold $M$. The authors emphasize how the bias depends on the noise $\sigma$ on the measured objects, more precisely on the ratio of $\sigma$ with respect to the overall size of the objects. \cite{Allassonniere2016} also presents a case study for an infinite dimensional flat manifold $M$ quotiented by translations, where the noise $\sigma$ is one of the crucial variables controlling the bias. However, the case of general curved manifolds $M$ has not been investigated yet.

\subsection*{Contributions and outline}

We are still missing a \textit{global geometric} understanding of the bias. Which variables control its magnitude? Is it restricted to the mean shape or does it appear for other statistical analyses? How important is it in practice: do we even need to correct it? If so, how can we correct it? Our paper addresses these questions. We use a geometric framework that unifies the cases of landmarks, curves, images etc.

\paragraph{Contributions} We make three contributions. First, we show that statistics on shapes are biased when the data are measured with error. We explicitly compute the bias in the case of the mean shape. Formulated in the Procrustean terminology, our result is: the Generalized Procrustes Analysis (GPA) estimator of mean "form" is asymptotically biased, because we do not consider scalings. Second, we offer an interpretation of the bias through the geometry of the shape space. In applications, this aids in deciding when the bias can be neglected in contrast with situations when it must be corrected. Third, we leverage our understanding to suggest several correction approaches.

\paragraph{Outline} The paper has four Sections. Section 1 introduces the geometric framework of shape spaces. Section 2 presents our first two contributions: the proof and geometric interpretation of the bias. Section 3 describes the procedures to correct the bias. Section 4 validates and illustrates our results on synthetic and real data.

\section{Geometrization of template shape estimation}\label{sec:geom}

\subsection{Two running examples}

We introduce two simple examples of shape spaces that we will use to provide intuition. 

First, we consider two landmarks in the plane $\mathbb{R}^2$ (Figure~\ref{fig:simple} (a)). The landmarks are parameterized each with 2 coordinates. For simplicity we consider that one landmark is fixed at the origin on $\mathbb{R}^2$. Thus the system is now parameterized by the 2 coordinates of the second landmark only, e.g. in polar coordinates $(r,\theta)$. We are interested in the shape of the 2 landmarks, i.e. in their distance which is simply $r$.

Second, we consider two landmarks on the sphere $S^2$ (Figure~\ref{fig:simple} (b)). One of the landmark is fixed at the north pole of $S^2$. The system is now parameterized by the 2 coordinates of the second landmark only, i.e. $(\theta, \phi)$, where $\theta$ is the latitude and $\phi$ the longitude. The shape of the two landmarks is the angle between them and is simply $\theta$.

 \begin{figure}[!htbp]
  \centering
     \def\svgwidth{0.8\textwidth}
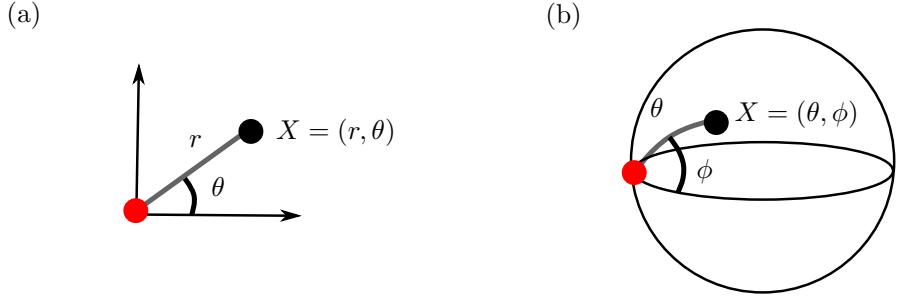
  \caption{Two landmarks, one in red and one in black, on the plane $\mathbb{R}^2$ (a) and on the sphere $S^2$ (b). The landmark in red is fixed at the origin of the coordinates. The system is entirely represented by the coordinates $X$ of the landmark in black.}\label{fig:simple}
\end{figure}

\subsection{Differential Geometry of shapes}

\subsubsection{The shape space is a quotient space}

The data are objects $\{X_i\}_{i=1}^n$ that are either sets of landmarks, curves, images, etc. We consider that each object $X_i$ is a point in a Riemannian manifold $M$. In this paper, we restrict ourselves to finite dimensional manifolds. We have $M=\mathbb{R}^2$ in the plane example: a flat manifold of dimension 2. We have $M=S^2$ in the sphere example: a manifold of constant positive curvature and of dimension 2. 

By definition, the objects' shapes are their equivalence classes $\{[X_i]\}_{i=1}^n$ under the action of some finite dimensional Lie group $G$: $G$ is a group of continuous transformations that models what does not change the shape. The action of $G$ on $M$ will be written with "$\cdot$". In our examples, the rotations are the transformations that leave the shape of the systems invariant. Let us take $g$ a rotation. The action of $g$ on the landmark $X$ is illustrated by a blue arrow in Figures~\ref{fig:leadingCases1} (a) for the plane and (d) for the sphere. We observe that the action does not change the shape of the systems: the distance between the two landmarks is preserved in (a), the angle between the two landmarks is preserved in (d). The equivalence class of $X_i$ is also called its orbit and written $O_{X_i}$. The equivalence class/orbit of $X$ is illustrated with the blue dotted circle in Figure~\ref{fig:leadingCases1} (a) for the plane example and in Figure~\ref{fig:leadingCases1} (d) for the sphere example. The orbit of $X$ in $M$ is the submanifold of all objects in $M$ that have the same shape as $X$. The curvature of the orbit as a submanifold of $M$ is the key point of the results in Section~\ref{sec:quant}.

The \textit{shape space} is by definition the space of orbits. This is a quotient space denoted $Q=M/G$. One orbit in $M$, i.e. one circle in Figure~\ref{fig:leadingCases1} (b) or (e), corresponds to a point in $Q$. The shape space is $Q=\mathbb{R}_+$ in the plane example. This is the space of all possible distances between the two landmarks, see Figure~\ref{fig:leadingCases1} (c). The shape space is $Q=[0,\pi]$ in the sphere example. This is the space of all possible angles between the two landmarks, see Figure~\ref{fig:leadingCases1} (f).

\subsubsection{The shape space is a metric space}

We consider that the action of $G$ on $M$ is \textit{isometric with respect to the Riemannian metric of $M$}. This implies that the distance $d_M$ between two objects in $M$ does not change if we transform both objects in the same manner. In the plane example, rotating the landmark $X_1$ and another landmark $X_2$ with the same angle does not change the distance between them.

The distance in $M$ induces a quasi-distance $d_Q$ in $Q$: $d_{Q}(O_{X_1}, O_{X_2})= \underset{g \in G}{\operatorname{inf}} d_{M}(g \cdot X_1, X_2)$ \cite{Huckemann2010}. The Lie group the action being isometric, the quasi-distance is in fact a distance. In the case of the The distance between the shapes of $X_1$ and $X_2$ is computed by first registering/aligning $X_1$ onto $X_2$ by the minimizing $g$, and then using the distance in the ambient space $M$. In the plane example, the distance between two shapes is the difference in distances between the landmarks. One can compute it by first aligning the landmarks, say on the first axis of $\mathbb{R}^2$, then one uses the distance in $\mathbb{R}^2$.

\subsubsection{The shape space has a dense set of principal shapes}

The \textit{isotropy group of $X_i$} is the subgroup of transformations of $G$ that leave $X_i$ invariant. For the plane example, every $X_i \neq (0,0)$ has isotropy group the identity and $(0,0)$ has isotropy group the whole group of 2D rotations. Objects on the same orbit, i.e. objects that have the same shape, have conjugate isotropy groups. 

\textit{Principal orbits or principal shapes} are orbits or shapes with smallest isotropy group conjugation class. In the plane example, $\mathbb{R}^2 \setminus (0,0)$ is the set of objects with principal shapes. Indeed, every $X$ in $\mathbb{R}^2 \setminus (0,0)$ belongs to a circle centered at $(0,0)$ and has isotropy group the identity. The set of principal shapes corresponds to $\mathbb{R}^*_+$ in the shape space and is colored in blue on Figure~\ref{fig:leadingCases1} (c). \textit{Singular orbits or singular shapes} are orbits or shapes with larger isotropy group conjugation class. In the plane example, $(0,0)$ is the only object with singular shape. It corresponds to $0$ in $\mathbb{R}_+$ and is colored in red in Figure~\ref{fig:leadingCases1} (c). 

Principal orbits form an open and dense subset of $M$, denoted $M^*$. This means that there are objects with non-degenerated shapes almost everywhere. In the plane example, $\mathbb{R}^2 \setminus (0,0)$ is dense in $\mathbb{R}^2$. In the sphere example, $S^2 \setminus 
\{(0,0),(\pi,0)\}$ is dense in $S^2$, where $(0,0)$ denotes the north pole and $(\pi,0)$ the south pole of $S^2$. Likewise, principal shapes form an open and dense subset in $Q$, denoted $Q^*$. In the plane example, $\mathbb{R}_+^*$ is dense in $\mathbb{R}_+$. In the sphere example, $]0,\pi[$ is dense in $[0,\pi]$. 

The dense set $M^*$ makes the projection in the quotient space a Riemannian submersion \cite{Huckemann2010}, which we use to embed $Q^* = M^*/G$ in $M^*$. In other words, the tangent space of the quotient space is identified everywhere with the vertical space with respect to the (isometric) Lie group action. Regular shapes of $Q^*$ are embedded in the space of objects with regular shapes $M^*$. The computations in Section~\ref{sec:quant} will be carried out on the dense set $M^*$ of principal orbits. The curvature of these principal orbits - i.e. of the blue circles of Figures~\ref{fig:leadingCases1}(b) and (e) - will be the main geometric parameter responsible for the asymptotic bias studied in this paper. We note that the curvature of principal orbits is closely related to the presence of singular orbits: principal orbits wrap around the singular orbits. In the plane example, any blue circle - i.e. any principal orbit - wraps around its center, the red dot $(0,0)$  - which is the singular orbit, see Figure~\ref{fig:leadingCases1}(b).

%

 \begin{figure}[!htbp]
  \centering
     \def\svgwidth{0.85\textwidth}
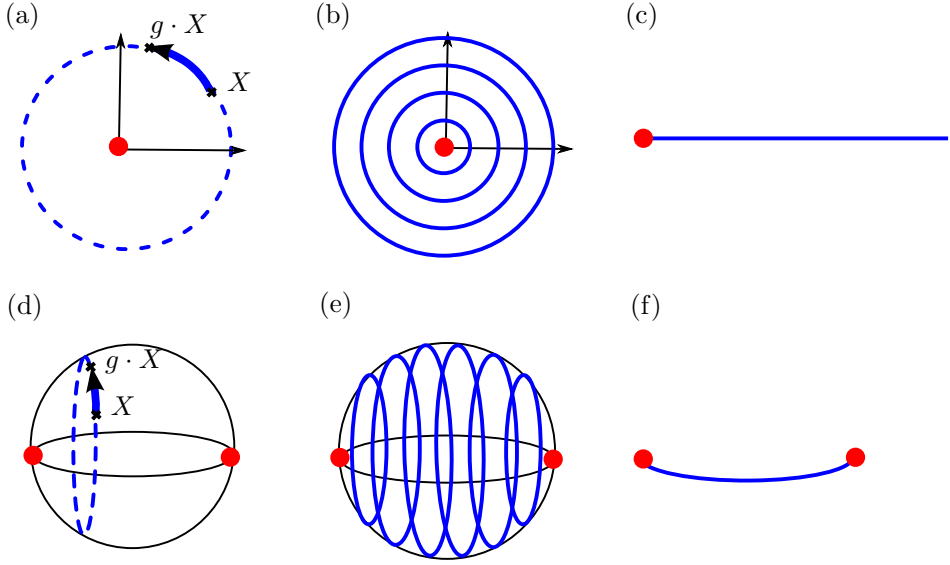
  \caption{First line: Action of rotations on $\mathbb{R}^2$, with (a): action of rotation $g \in SO(2)$ on point $X \in \mathbb{R}^2$ and orbit of $X$ in blue dotted line; (b) Stratification of $\mathbb{R}^2$ into principal orbit type (blue) and singular orbit type (red); (c) shape space $\mathbb{R}_+ = \mathbb{R}^2/SO(2)$ with a singularity (red dot). Second line: Action of $SO(2)$ on $S^2$ with (d): action of rotation $g \in SO(2)$ on point $X \in S^2$ and orbit of $X$ in blue dotted line; (e) Stratification of $S^2$ into principal orbit type (blue) and singular orbit type (red) (f) shape space $[0,\pi]= S^2/SO(2)$ with two singularities (red dots).}
   \label{fig:leadingCases1}
\end{figure}

We have focused on an intuitive introduction of the concepts. We refer to \cite{Postnikov2001,Alekseevsky2003,Huckemann2010} for mathematical details. From now on, the mathematical setting is the following: we assume a proper, effective and isometric action of a finite dimensional Lie group $G$ on a finite dimensional complete Riemannian manifold $M$.

\subsection{Geometrization of generative models of shape data}

We recall that the data are the $\{X_i\}_{i=1}^n$ that are sets of landmarks, curves, images, etc. In the general case, one can interpret the data $X_i$'s as random realizations of the generative model:
\begin{equation}\label{eq:genModelfull}
X_i = \text{Exp}(g_i \cdot Y_i,\epsilon_i) \qquad i=1...n,
\end{equation}
where $\text{Exp}(p,u)$ denotes the Riemannian exponential of $u$ at point $p$. The $Y_i$, $g_i$, $\epsilon_i$ are respectively i.i.d. realizations of random variables that are drawn independently. 

In this paper as well as often in the literature \cite{Allassonniere2007, Allassonniere2016, Bigot2010, Bigot2011, Kurtek2011}, we consider mainly the following simpler generative model:
\begin{equation}\label{eq:genModel}
X_i = \text{Exp}(g_i \cdot Y,\epsilon_i) \qquad i=1...n,
\end{equation}
\textit{where $Y$ is a parameter which we call the template shape}. The following three step formulation of the generative models~\ref{eq:genModelfull} and (\ref{eq:genModel}) gives technical details and their interpretation in terms of shapes.

\paragraph{Step 1: Generate the shape $Y_i \in M^*/G \subset M^*$} In the full generative model~(\ref{eq:genModelfull}), we assume that there is a probability density of shapes in the Riemannian manifold $Q^*=M^*/G$, with respect to the measure on $Q^*$ induced by the Riemannian measure of $M^*$. The $Y_i$'s are i.i.d. samples drawn from this distribution. For example, it can be a Gaussian - or one of its generalization to manifolds \cite{Pennec2006} - as illustrated in Figure~\ref{fig:genModel1} on the shape spaces for the plane and sphere examples. This is the variability that is meaningful for the statistical study, whether we are analyzing shapes of skulls, proteins, bones, neural spike trains, handwritten digits or brains. 

We mainly assume in this paper the simpler generative model~(\ref{eq:genModel}) with parameter: the template shape $Y \in M^*/G$. In other words, we assume that the probability distribution is singular and more precisely that it is simply a Dirac at $Y$. This is the most common assumption within the model~(\ref{eq:genModelfull}) \cite{Allassonniere2007, Allassonniere2015b, Kurtek2011, Bigot2010}. We point out that $Y$ is a point of the shape space $M^*/G$, which is embedded in the object space by $M^*/G \subset M^*$, see previous subsection.

 \begin{figure}[!htbp]
  \centering
       \def\svgwidth{0.7\textwidth}
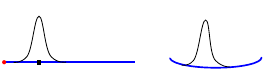
  \caption{Step 1 of generative model of Equation~(\ref{eq:genModelfull}) for the plane example (a) and the sphere example (b). The black curve illustrates the probability distribution function on shape space. This is a distribution on  $r \in \mathbb{R}_+$ for the plane example (a) and on $\theta \in [0,\pi]$ for the sphere example. The black square represents its expectation. For the simpler generative model of Equation~(\ref{eq:genModel}), the probability distribution boils down to a single point at $Y$ i.e. at the black square.}
   \label{fig:genModel1}
\end{figure}

\paragraph{Step 2: Generate its position/parameterization $g_i \in G$, to get $g_i \cdot Y \in M^*$} We cannot observe shapes in $Q=M/G$. We rather observe objects in $M$, that are shapes posed or parameterized in a certain way. We assume that there is a probability distribution on the positions or parameterizations of $G$, or equivalently a probability distribution on principal orbits with respect to their intrinsic measure. We assume that the distribution does not depend on the shape $Y_i$ that has been drawn. The $g_i$'s are i.i.d. from this distribution. For example, it can be a Gaussian - or one of its generalization to manifolds \cite{Pennec2006} - as illustrated in Figure~\ref{fig:genModel2} on the shape spaces for the plane and sphere examples. 

The drawn $g_i$ is used to pose/parameterize the shape $Y_i$ drawn in Step 1 (in the case of model of Equation~(\ref{eq:genModelfull})), where $Y_i=Y$ (in the case of model of Equation~(\ref{eq:genModel})). The shape is posed/parameterized through the isometric action of $G$ on $Q^* \subset M^*$, to get the object $g_i \cdot Y_i \in M^*$, or the object $g_i \cdot Y \in M^*$ in the case of the simpler model of Equation~(\ref{eq:genModel}).

 \begin{figure}[!htbp]
  \centering
       \def\svgwidth{0.7\textwidth}
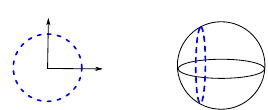
  \caption{Step 2 of generative model of Equation~(\ref{eq:genModel}) for the plane example (a) and the sphere example (b). The blue dotted curve illustrates the orbit of the shape drawn in Step 1. The black curve illustrates the probability distribution function on this orbit. This is a distribution in angle $\theta \in [0,2\pi]$ for the plane example (polar coordinates) and in angle $\phi \in [0,2\pi]$ for the sphere example (spherical coordinates).}
   \label{fig:genModel2}
\end{figure}

\paragraph{Step 3: Generate the noise $\epsilon_i \in T_{g_i \cdot Y_i}M$} The observed $X_i$'s are results of noisy measurements. We assume that there is a probability distribution function on $T_{g_i \cdot Y_i}M$ representing the noise. We further assume that this is a Gaussian - or one of its generalization to manifolds \cite{Pennec2006} - centered at $g_i \cdot Y_i$, the origin of the tangent space $T_{g_i \cdot Y_i}M$, and with standard deviation $\sigma$, see Figure~\ref{fig:genModel3}. The parameter $\sigma$ will be extremely important in the developments of Section~\ref{sec:quant}, as we will compute Taylor expansions around $\sigma=0$.

 \begin{figure}[!htbp]
  \centering
       \def\svgwidth{0.7\textwidth}
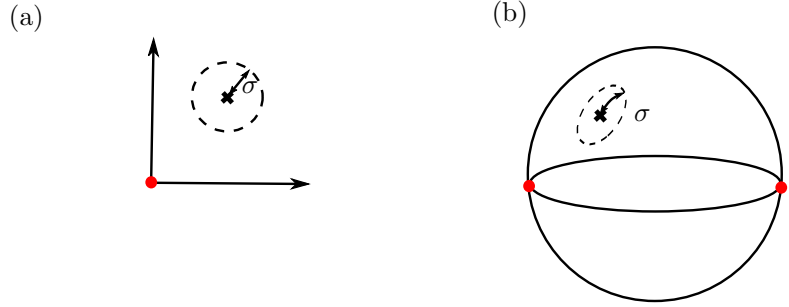
  \caption{Step 3 of generative model of Equation~(\ref{eq:genModel}) for the plane example (a) and the sphere example (b). The dotted curve represents the isolevel at $\sigma$ of the Gaussian distribution function on the ambient space.}
  \label{fig:genModel3}
\end{figure}

Other generative models may be considered in the literature. We find in \cite{Allassonniere2015b} the model: $X_i= g_i \cdot \text{Exp}(Y_i,\epsilon_i)$, where the Riemannian exponential $\text{Exp}$ is also performed in $M$ through the embedding $ Y_i \in Q \subset M$. In \cite{Kurtek2011}, we find the model without noise: $X_i= g_i \cdot Y_i$.

\subsection{Learning the variability in shapes: estimating the template shape}

Our goal is to unveil the variability \textit{of shapes in $Q=M/G$} while we in fact observe the noisy \textit{objects $X_i$'s in $M$}. We focus on the case where the variability in the shape space is assumed to be a Dirac at $Y$. Our goal is thus to estimate the template shape $Y$, which is a parameter of the generative model.

\paragraph{Estimating the template shape with the Fr\'echet mean in the shape space} We describe the procedure usually performed in the literature \cite{Kurtek2011, Allassonniere2007, Allassonniere2016, Bigot2010, Bigot2011}. One initializes the estimate with $\hat Y =X_1$. Then, one iterates the following two steps until convergence:
\begin{equation}
  \left\{
      \begin{aligned}
(i) \quad & \hat g_i = \underset{g \in G}{\text{argmin }} d_M(\hat Y, g \cdot X_i), \quad\forall i \in \{1,...,n\},\\
(ii) \quad &\hat{Y} = \underset{Y \in M}{\text{argmin }} \sum_{i=1}^n d_M(Y, \hat g_i \cdot X_i)^2 .
      \end{aligned}
    \right.
    \label{eq:procedure}
\end{equation}

This procedure has a very intuitive interpretation. Step (i) is the projection of each object $X_i$ in the shape space $Q$, as illustrated in Figure~\ref{fig:Estimator} (a)-(i) and (b)-(i) with the blue arrows. We assume that each minimizer $\hat g_i$ exists and is attained. In practice, this is true for example when the Lie group is compact, like the Lie group of rotations. We take $X_1$, $X_2$, $X_3$ three objects in $\mathbb{R}^2$ in Figure~\ref{fig:Estimator} (a)-(i) and on $S^2$ in Figure~\ref{fig:Estimator} (b)-(i). One filters out the position/parameterization component, i.e. the coordinate on the orbit. One projects the objects $X_1$, $X_2$, $X_3$ in the shape space $Q$ using the blue arrows.

Step (ii) is the computation of the mean of the registered data $\hat g_i \cdot X_i$, i.e. of the objects' shapes, as illustrated in Figure~\ref{fig:Estimator}(a)-(i) and (b)-(i) where $\hat Y$ is shown in orange. Again, we assume that the minimizer $\hat Y$ exists and is attained. In practice, this will be the case as we will consider a low level of noise in Step 3 of the generative model. The registered data $\hat g_i \cdot X_i$ will be concentrated on a small neighborhood of diameter of order $\sigma$. As a consequence, their Fr\'echet mean in the Riemannian manifold $Q^*$ is guaranteed to exist and be unique \cite{Emery1991}.

The procedure of Equations~(\ref{eq:procedure}) (i)-(ii) decreases at each step the following cost, which is bounded below by zero:
\begin{equation}
\text{Cost}(g_1, ..., g_n, Y) = \sum_{i=1}^n  d_M^2(Y, g \cdot X_i).
\end{equation}
Under the assumptions that both steps (i) and (ii) attained their minimizers, we are guaranteed convergence to a local minimum. We further assume that the procedure converges to the global minimum. The estimator computed with the procedure is then:
\begin{equation}\label{eq:frechet}
\hat Y = \underset{Y \in M}{\text{argmin }} \sum_{i=1}^n \underset{g \in G}{\text{min }} d_M^2(Y, g \cdot X_i).
\end{equation}
The term $\underset{g \in G}{\text{min }} d_M^2(Y, g \cdot X_i)$ in Equation~\ref{eq:frechet} is the distance in the shape space between the shapes of $Y$ and $X_i$. Thus, we recognize in Equation~\ref{eq:frechet} the Fr\'echet mean on the shape space. The Fr\'echet mean is a definition of mean on manifolds \cite{Pennec2006}: it is the point that minimizes the squared distances to the data in the shape space. All in all, one projects the probability distribution function of the $X_i$'s from $M$ to $M/G$ and computes its "expectation", in a sense made precise later.

 \begin{figure}[!htbp]
  \centering
     \def\svgwidth{1\textwidth}
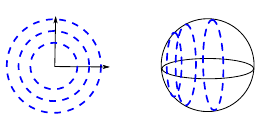
  \caption{Steps (i) and (ii) of procedure of template shape estimation described in Equations~(\ref{eq:procedure}) (i)-(ii) for the plane example (a) and the sphere example (b). The 3 black plus signs in $\mathbb{R}^2$ (a) or $S^2$ (b) represent the 3 data. The 3 dotted blue curves are their orbits. In Step (i), the $X_i$'s are registered, i.e. their projected in the shape space: 3 curved blue arrows represent their registration with the minimizers $\hat g_i$. The 3 black crosses in $\mathbb{R}_+$ (positive x-axis) (a) or $[0,\pi]$ (b) represent the registered data. In Step (ii), the template shape estimate $\hat Y$ is computed as the Fr\'echet mean of the registered data and is shown in orange.}
   \label{fig:Estimator}
\end{figure}

We implemented the generative model and the estimation procedure on the plane and the sphere in shiny applications available online: \url{https://nmiolane.shinyapps.io/shinyPlane} and \url{https://nmiolane.shinyapps.io/shinySphere}. We invite the reader to look at the web pages and play with the different parameters of the generative model. Figure~\ref{fig:shinyPlane1} shows screen shots of the applications.

\begin{figure}[h!]
\centering
     \def\svgwidth{0.85\textwidth}
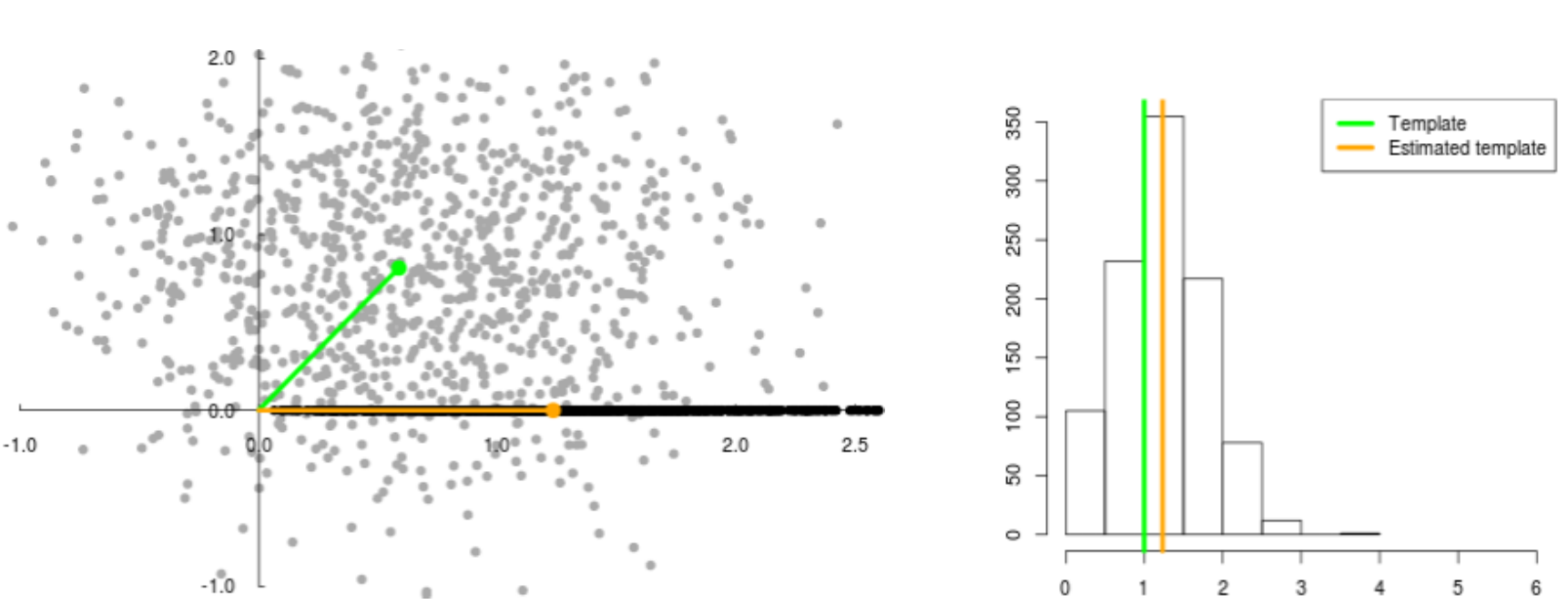
\centering
     \def\svgwidth{0.95\textwidth}
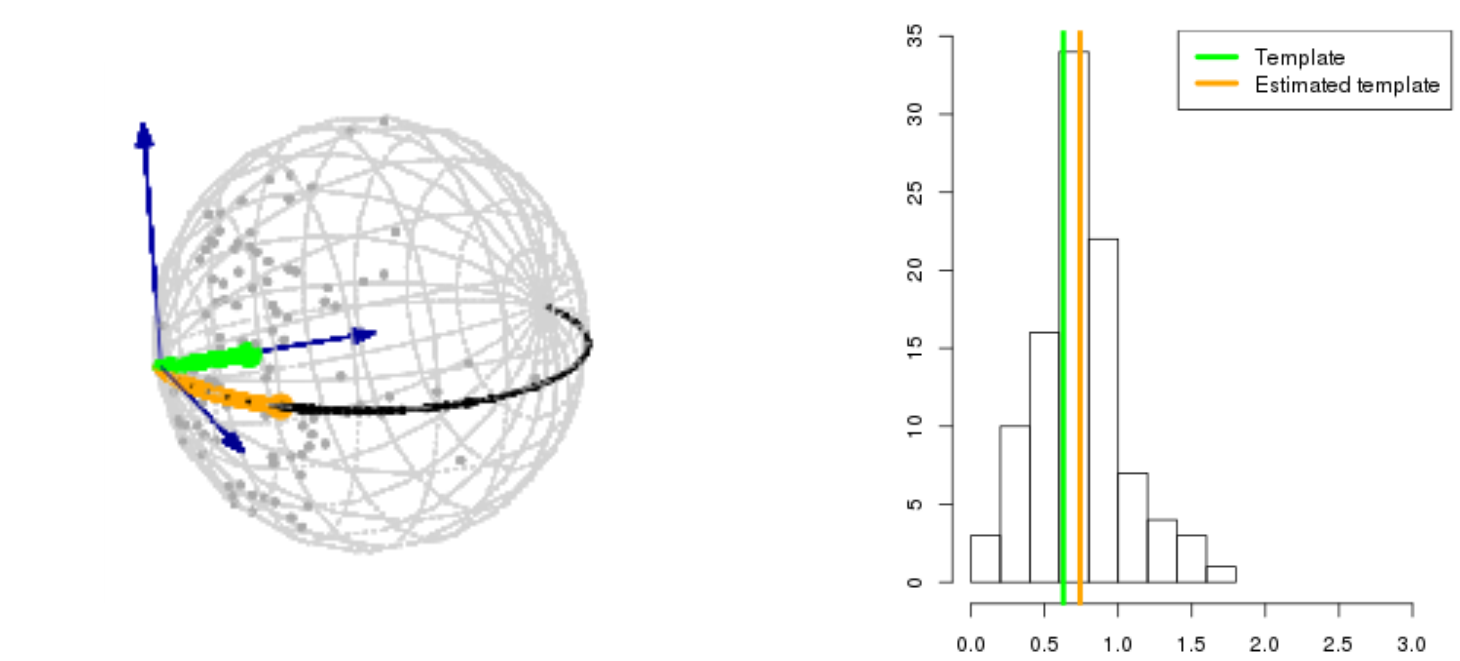
    \caption{Screenshot of \url{https://nmiolane.shinyapps.io/shinyPlane} and \url{https://nmiolane.shinyapps.io/shinySphere}. Simulated data $X_i$'s (grey points), template shape $Y$ (green), registered data $\hat{g_i}\cdot X_i$ (black points), template shape estimate $\hat Y$ (orange). Induced distributions on the shapes, template shape $Y$ (green), template shape estimate $\hat Y$ (orange).\label{fig:shinyPlane1}}
\end{figure}

\paragraph{Probabilistic interpretation of the procedure in Equations~(\ref{eq:procedure}): an approximation of a Maximum-Likelihood} Beside its intuitive interpretation, the procedure of template shape estimation of Equations~(\ref{eq:procedure}) has a probabilistic interpretation. We have the generative model of the data $X_i$'s: it is described in Equation~(\ref{eq:genModel}) and Steps (1)-(3) of the previous subsections. Thus, one may consider the Maximum Likelihood (ML) estimate of $Y$, which is one of its parameters:
\begin{align*}
\hat Y_{ML} = \underset{Y \in Q }{\operatorname{argmax } }  \text{ L}(Y) &= \underset{Y \in Q }{\operatorname{argmax}} \sum_{i=1}^n \log(P(X_i|Y))\\
& = \underset{Y \in Q }{\operatorname{argmax}} \sum_{i=1}^n  \log \left(\int_{g \in G} P(X_i|Y,g).P(g)dg \right).
\end{align*}
In the above, $P(X_i|Y)$ is the probability distribution of the data in $M$ as a function of the parameter $Y$. $P(g)$ is the probability distribution on the poses/parameterizations in $G$ as described in Step 2 of the generative model given in the previous subsection. Then, $P(X_i|Y,g)$ is the probability distribution of the noise as described in Step 3.

The $g$'s are hidden variables in the model. The Expectation-Maximization (EM) algorithm is therefore the natural implementation for computing the ML estimator \cite{Allassonniere2007}. But the EM algorithm is computationally expensive, above all for tridimensional images. Thus, one can usually rely on an approximation of the EM, which is described in  \cite{Allassonniere2007} as the "modal approximation" and used in \cite{Allassonniere2015b,Kurtek2011,Bigot2010}. 

We can check that this approximation is the procedure described in Equations~(\ref{eq:procedure}). Step (i) is an estimation of the hidden observations $g_i$ and an approximation of the E-step of the EM algorithm. Step (ii) is the M-step of the EM algorithm: the maximization of the surrogate in the M-step amounts to the maximization of the variance of the projected data. This is exactly the minimization of the squared distances to the data of (ii). We refer to \cite{Allassonniere2007} for details.

\paragraph{Purpose of this paper reformulated with the geometrization} Our main result is to show that the procedure presented in Equations~(\ref{eq:procedure}) (and illustrated on Figure~\ref{fig:Estimator}) gives an asymptotically biased estimate $\hat Y$ for the template shape $Y$ of the generative model presented in Equation~(\ref{eq:genModel}) (and illustrated in Figures~\ref{fig:genModel1}, ~\ref{fig:genModel2} and \ref{fig:genModel3}). Figures~\ref{fig:shinyPlane1} (a)-(d) present what is meant by \textit{asymptotic bias}: the estimate $\hat Y$, of the procedure, is in orange and the template shape $Y$, of the generative model, is in green. The estimator $\hat Y$ (in orange) does converge when the number of data, i.e. the grey points in Figures~\ref{fig:shinyPlane1}(a)-(c), goes to infinity, \textit{but $\hat Y$ does not converge to the template shape $Y$ it is designed to estimate}. For Figures~\ref{fig:shinyPlane1} (a)-(d), this means that even for an infinite number of grey points, the orange estimate will be different from the green parameter. We say that $\hat Y$ has an asymptotic bias with respect to the parameter $Y$.

Where does this asymptotic bias come from and why doesn't $\hat Y$ converge to $Y$? In a nutshell, the bias comes from the external curvature of the template's orbit and we explain and summarize this in Figure~\ref{fig:curvature} and its caption. The full geometric answer with its technical details is provided in the next section.

\section{Quantification and correction of the asymptotic bias\label{sec:quant}}

This section explains, quantifies and corrects the asymptotic bias of the template shape estimate $\hat Y$ with respect to the parameter $Y$. We start from the definition of the asymptotic bias of an estimator with respect to the parameter it is designed to estimate. More precisely we start from a generalization of this definition to Riemannian manifolds:
\begin{equation}\label{eq:defbias}
\text{Bias}(\hat Y,Y)=\mathbb{E}\left[\text{Log}_Y \hat Y\right].
\end{equation}
This is the asymptotic bias of the estimator $\hat Y$ with respect to the (manifold-valued) parameter $Y$, which generalizes the corresponding definition for linear spaces:
\begin{equation}
\text{Bias}(\hat Y,Y)= \mathbb{E}\left[\hat Y - Y\right]
\end{equation}
In the Riemannian definition of the bias, $\text{Log}_Y \hat Y$ is the Riemannian logarithm of $\hat Y \in Q$ at $Y \in Q$, i.e. a vector of the tangent space of $Q$ at the real parameter $Y$, denoted $T_YQ$. The tangent vector $\text{Log}_Y \hat Y$ is illustrated on Figures~\ref{fig:AsymptoticBias} (a) and (b) for the plane and sphere examples. $\text{Log}_Y \hat Y$ represents how much one would have to shoot from $Y$ to get the estimated parameter $\hat Y$. The norm of $\text{Log}_Y \hat Y$, computed using the metric of $Q$ at $Y$, represents the dissimilarity between $\hat Y$ and $Y$.

 \begin{figure}[!htbp]
  \centering
       \def\svgwidth{0.85\textwidth}
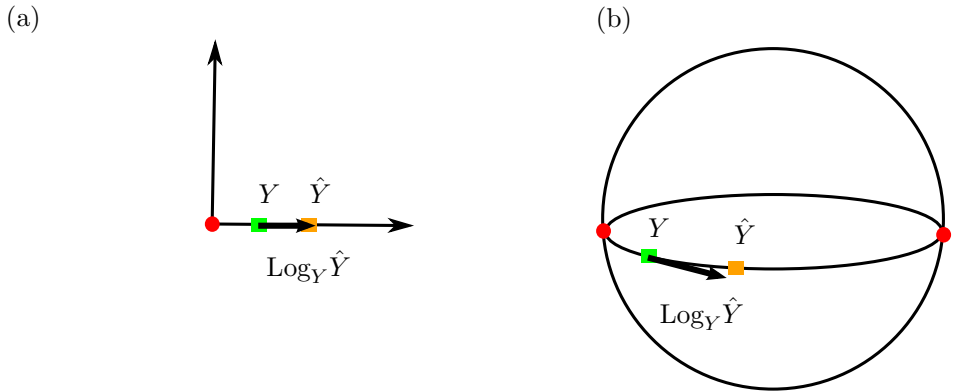
\caption{Illustration of the Riemannian definition of the asymptotic bias $\text{Log}_Y \hat Y$ for the plane example (a) and the sphere example (b). $\text{Log}$ refers to the Riemannian logarithm \cite{Postnikov2001} and $\text{Log}_Y \hat Y$ is thus a tangent vector of the quotient space $Q$ at $Y$. $\text{Log}_Y \hat Y$ represents how much one would have to shoot from $Y$ to get the estimated parameter $\hat Y$. The norm of $\text{Log}_Y \hat Y$, computed using the metric of $Q$ at $Y$, represents the distance or the dissimilarity between $\hat Y$ and $Y$, i.e. how far $\hat Y$ is from estimating $Y$\label{fig:AsymptoticBias}.}
\end{figure}

We could also consider the variance of the estimator $\hat Y$. The variance is defined as $\text{Var}_n (\hat Y) = \operatorname{E}[d_M(Y,E[Y])^2]$. In the limit of an infinite sample, we have: $\text{Var}_\infty (\hat Y) =0$. This is why we focus on the asymptotic bias.

\subsection{Asymptotic bias of the template's estimator on examples}

We first compute the asymptotic bias for the examples of the plane and the sphere to give the intuition.

The probability distribution function of the $X_i$'s comes from the generative model. This is a probability distribution on $\mathbb{R}^2$ for the plane example, parameterized in polar coordinates $(r,\theta)$ like Figure~\ref{fig:simple}. So we can compute the projected distribution function on the shapes, which are the radii $r$ here. This is done simply by integrating out the distribution on $\theta$, the position on the circles. This gives a probability distribution on $\mathbb{R}_+$ for the plane example. We write it $f: r \mapsto f(r)$. We remark that $f$ does not depend on the probability distribution function on the $\theta_i$'s of Step 2 of the generative model. We can also compute $f: \theta \mapsto f(\theta)$ in the sphere example: we integrate over $\phi$ the probability distribution function on $(\theta,\phi)$.

Figure~\ref{fig:pdfs} (a) shows $f$ for the plane example, for a template $r=1$. We plot it for two different noise levels $\sigma=0.3$ and $\sigma=3$. Note that here $f$ is the Rice distribution. Figure~\ref{fig:pdfs} (b) shows $f$ for the sphere example, for a template $\theta=1$. We plot it for different noise levels and $\sigma=0.3$ and $\sigma=3$. In both cases, the x-axis represents the shape space which is $\mathbb{R}_+$ for the plane example and $[0,\pi]$ for the sphere example. The green vertical bar represents the template shape, which is 1 in both cases. The red vertical bar is the expectation of $f$ in each case. It is $\hat Y$, the estimate of $Y$ We see on these plots that $f$ is not centered at the template shape: the green and red bars do not coincide. $f$ is skewed away from 0 in the plane example and away from $0$ and $\pi$ in the sphere example. The skew increases with the noise level $\sigma$. The difference between the green and red bars is precisely the bias of $\hat Y$ with respect to $Y$. 

\begin{figure}[h!]
\centering
     \def\svgwidth{1\textwidth}
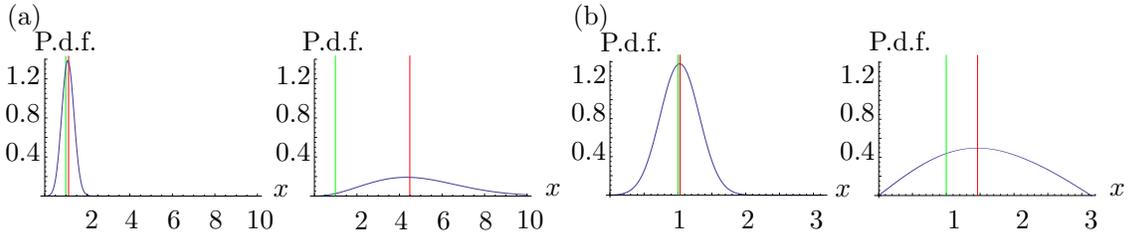
    \caption{(a) Induced distributions on the distance $r$ between two landmarks in $\mathbb{R}^3$ for real distance $y=1$ (in green) and noise level $\sigma=0.3$ and $\sigma=3$. (b) Induced distributions on the angle $x$ between the two landmarks on $S^3$, for real angle $y=1$ and noise levels $\sigma=0.3$ and $\sigma=3$. In both cases the mean shape estimate $\hat y$ is shown in red.}
   \label{fig:pdfs}
\end{figure}

Figure~\ref{fig:bias} shows the bias of $\hat Y$ with respect to $Y$, as a function of $\sigma$, for the plane (left) and the sphere (right). Increasing the noise level $\sigma$ takes the estimate $\hat Y$ away from $Y$. The estimate is repulsed from $0$ in the plane example: it goes to $\infty$ when $\sigma \rightarrow \infty$. It is repulsed from $0$ and $\pi$ in the sphere example: it goes to $\pi/2$ when $\sigma \rightarrow \pi$, as the probability distribution becomes uniform on the sphere in this limit. One can show numerically that the bias varies as $\sigma^2$ around $\sigma =0$ in both cases. This is also observed on the shiny applications \cite{shiny} at \url{https://nmiolane.shinyapps.io/shinyPlane} and \url{https://nmiolane.shinyapps.io/shinySphere}.

\begin{figure}[h!]
\centering
     \def\svgwidth{1\textwidth}
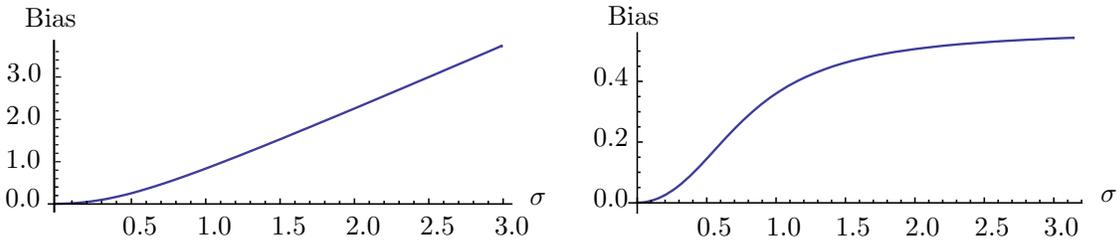
\caption{Asymptotic bias on the mean shape estimate $\hat Y$ with respect to the noise level $\sigma$ for $r=1$ in the plane example (a) and $\theta =1$ in the sphere example (b). The bias is quadratic near $\sigma=0$. Increasing $\sigma$ takes the estimate $\hat Y$ away from $0$ in shape space $Q=\mathbb{R}_+$ (a) and away from $0$ and $\pi$ in shape space $Q=[0,\pi]$ (b).\label{fig:bias}}
\end{figure}

These examples already show the origin of the asymptotic bias of $\hat Y$, for low noise levels $\sigma \rightarrow 0$ or for high noise levels: $\sigma \rightarrow + \infty$ for the plane example and $\sigma \rightarrow \pi$ for the sphere example. As long as there is noise, i.e. $\sigma \neq 0$, there is \textit{a bias that comes from the curvature of the template's orbit.} Figure~\ref{fig:curvature} shows the template's orbit in blue, in (a) for the plane and (b) for the sphere. In both cases the black circle represents the level set $\sigma$ of the Gaussian noise. In the plane example (a), the probability of generating an observation $X_i$ outside of the template's shape orbit is bigger than the probability of generating it inside: the grey area in the black circle is bigger than the white area in the white circle. There will be more registered data that are greater than the template. Their expected value will therefore be greater than the template and thus biased. In the sphere example (b), if the template's shape orbit is defined by a constant $\theta < \pi/2$, the probability of generating an observation $X_i$ "outside" of it, i.e. with $\theta_i > \theta$, is bigger than the probability of generating it "inside". There will be more registered data that are greater than the template $\theta$ and again, their expected value will also be greater than the template. Conversely, if the template is $\theta > \pi/2$, the phenomenon is inversed: there will be more registered data that are smaller than the template. The average of these registered data will also be smaller than the template. Finally, if the template's shape orbit is the great circle defined by $\theta = \pi/2$, then the probability of generating an observation $X_i$ on the left is the same as the probability of generating an observation $X_i$ on the right. In this case, the registered data will be well-balanced around the template $\theta = \pi/2$ and their expected value will be $\pi/2$: there is no asymptotic bias in this particular case. We prove this in the general case in the next section.

 \begin{figure}[!htbp]
  \centering
       \def\svgwidth{0.85\textwidth}
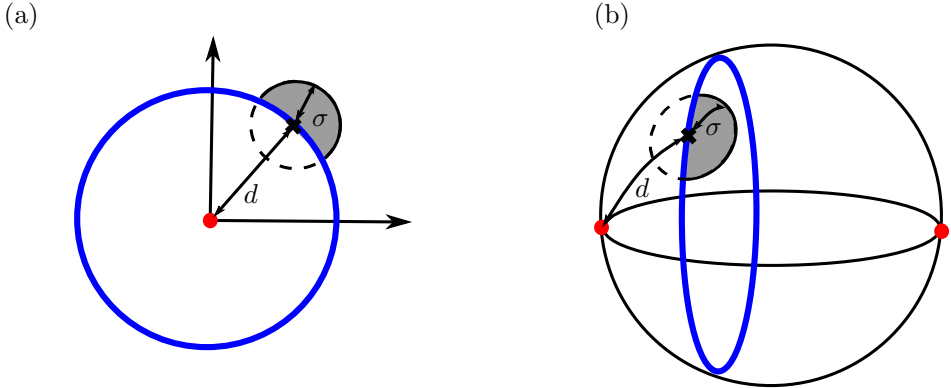
  \caption{The external curvature of the template's orbit creates the asymptotic bias, in the plane example (a) and the sphere example (b). The blue curve represents the template's orbit. The ball of radius $\sigma$ represents a level set of the Gaussian distribution of the noise in $\mathbb{R}^2$ (a) and $S^2$ (b). The grey-colored area represents the distribution of the noise that generates data outside the orbit of $Y$, in Step 3 of the generative model of Equation~(\ref{eq:genModel}) and Figure~\ref{fig:genModel3}. There is a higher probability that the data are generated "outside" the orbit. The template shape estimate is biased towards greater radii (a) or towards angles closer to $\pi/2$ (b).}
   \label{fig:curvature}
\end{figure}

\subsection{Asymptotic bias of the template's estimator for the general case}

We show the asymptotic bias of $\hat Y$ in the general case and prove that it comes from the external curvature of the template's orbit. We show it for $Y$ a principal shape and for a Gaussian noise of variance $\sigma^2$, truncated at $3\sigma$. Our results will need the following definitions of curvature. 

The \textit{second fundamental form $h$} of a submanifold $O$ of $M$ is defined on $T_XO \times T_XO$ by $h(v,w)=(\nabla_v w)^\bot \in N_XO$, where $(\nabla_v w)^\bot$ denotes the orthogonal projection of covariant derivative $\nabla_v w$ onto the normal bundle. The \textit{mean curvature vector $H$ of $O$} is defined as: $H = Tr(h)$. Intuitively, $h$ and $H$ are measures of extrinsic curvature of $O$ in $M$. For example an hypersphere of radius $R$ in $\mathbb{R}^m$ has mean curvature vector $||H||= \frac{m-1}{R}$. 

\begin{theorem}\label{th:pdf}
The data $X_i$'s are generated in the finite-dimensional Riemannian manifold $M$ following the model: $X_i = \text{Exp}(g_i \cdot Y,\epsilon_i), i=1...n$, described in Equation~(\ref{eq:genModel}) and Figures~\ref{fig:genModel1}-\ref{fig:genModel3}. In this model: (i) the action of the finite dimensional Lie group $G$ on $M$, denoted $\cdot$, is isometric, (ii) the parameter $Y$ is the template shape in the shape space $Q$, (iii) $\epsilon_i$ is the noise and follows a (generalization to manifolds of a) Gaussian of variance $\sigma^2$, see Section~\ref{sec:geom}.

Then, the probability distribution function $f$ on the shapes of the $X_i$'s, $i=1...n$, in the asymptotic regime on an infinite number of data $n \rightarrow + \infty$, has the following Taylor expansion around the noise level $\sigma = 0$:
\begin{align*}\label{eq:f}
f(Z) &  = \frac{1}{(\sqrt{2\pi}\sigma)^q}\exp \left(-\frac{d_M^2(Y,Z)}{2\sigma^2} \right)\left( F_0(Z) +\sigma^2 F_2(Z)+\mathcal{O}(\sigma^{4})+  \epsilon(\sigma)\right)
\end{align*}
where (i) $Z$ denotes a point in the shape space $Q$, (ii) $F_0$ and $F_2$ are functions of $Z$ involving the derivatives of the Riemannian tensor at $Z$ and the derivatives of the graph $G$ describing the orbit $O_Z$ at $Z$, and (iii) $\epsilon$ is a function of $\sigma$ that decreases exponentially for $\sigma \rightarrow 0$.
\end{theorem}
\begin{proof}
The sketch of the proof is given in Appendices, with the expressions of $F_0$ and $F_2$. The detailed proof is in the supplementary materials.
\end{proof}
The exponential in the expression of $f$ belongs to a Gaussian distribution centered at $Z$ and of isotropic variance $\sigma^2 \mathbb{I}$. However the whole distribution $f$ differs from the Gaussian because of the $Z$-dependent term in the right parenthesis. This induces a skew of the distribution away from the singular shapes, as observed for the examples in Figure~\ref{fig:pdfs}. This also means that the expectation of this distribution is not $Z$ and that the variance is not the isotropic $\sigma^2 \mathbb{I}$.

\begin{theorem}\label{th:bias}
The data $X_i$'s are generated with the model described in Equation~(\ref{eq:genModel}) and Figures~\ref{fig:genModel1}-\ref{fig:genModel3}, where the template shape $Y$ is a parameter and under the assumptions of Theorem~\ref{th:pdf}. The template shape $Y$ is estimated with $\hat Y$, which is computed by the usual procedure described in Equations~(\ref{eq:procedure}).

In the regime of an infinite number of data $n \rightarrow + \infty$, the asymptotic bias of the template's shape estimator $\hat Y$, with respect to the parameter $Y$, has the following Taylor expansion around the noise level $\sigma = 0$: 
\begin{equation}\label{eq:bias}
\text{Bias}(\hat Y,Y)= - \frac{\sigma^2}{2} H(Y) + \mathcal{O}(\sigma^4) + \epsilon(\sigma)
\end{equation}
where (i) $H$ is the mean curvature vector of the template shape's orbit which represents the external curvature of the orbit in $M$, and (ii) $\epsilon$ is a function of $\sigma$ that decreases exponentially for $\sigma \rightarrow 0$.
\end{theorem}
\begin{proof}
The sketch of the proof is given in Appendices. The detailed proof is in the supplementary materials.
\end{proof}
This generalizes the quadratic behavior observed in the examples on Figure~\ref{fig:bias}. The asymptotic bias has a geometric origin: it comes from the external curvature of the template's orbits, see Figure~\ref{fig:curvature}.

We can vary two parameters in equation~\ref{eq:bias}: $Y$ and $\sigma$. The external curvature of orbits generally increases when $Y$ is closer to a singularity of the shape space (see Section 1) \cite{Lytchak2010}. The singular shape of the two landmarks in $\mathbb{R}^2$ arises when their distance is 0. In this case, the mean curvature vector has magnitude $|H(Y)| = \frac{1}{d}$: it is inversely proportional to $d$, the radius of the orbit. $d$ is also the distance of $Y$ to the singularity $0$.  

\subsection{Limitations and extensions}

\paragraph{Beyond $Y$ being a principal shape} Our results are valid when the template $Y$ is a principal shape. This is a reasonable assumption as the set of principal shapes is dense in the shape space. What happens when $Y$ approaches a singularity, i.e. when $Y$ changes stratum in the stratified space $Q$? Taking the limit $d \rightarrow 0$ in the coefficients of the Taylor expansion is not a legal operation. Therefore, we cannot conclude on the Taylor expansion of the Bias for $d \rightarrow 0$. Indeed, the Taylor expansion may even change order for $d \rightarrow 0$. We take $M=\mathbb{R}^m$ with the action of $SO(m)$ and the template $Y=(0,...,0)$:
\begin{equation}
\text{Bias}(\hat Y,Y)=\sqrt{2}\frac{\Gamma(\frac{m+1}{2})}{\Gamma(\frac{m}{2})}\sigma.
\end{equation}  
The bias is linear in $\sigma$ in this case. 

\paragraph{Beyond $\sigma <<1$} The assumption $\sigma <<1$ represents our hope that the noise on the shape data is not too large with respect to the overall size of the mean shape. Nevertheless it would be very interesting to study the asymptotic bias for any $\sigma$, including large noises ($\sigma \rightarrow +\infty$). The distribution over the $X_i$'s in $M$ will be spread on the whole manifold $M$. We cannot rely on local computations on $M$ (at the scale of $\sigma$) anymore. We have to make global assumptions on the manifold $M$. 

The plane example is the canonical example of a flat manifold. The sphere example is the canonical example of manifold with constant (positive) curvature. The bias as a function of $\sigma$ is plotted in Figure~\ref{fig:bias}. It leads us to the conjecture that the estimate converges towards a barycenter of shape space's singularities when the noise level increases. Singularities have a repulsive action on the estimation of each template's shape. Such repulsive force acts on each estimators. As a result, the estimators of the mean shape finds an equilibrium position: the barycenter.

\paragraph{Beyond one Dirac in $Q$: several templates} We have considered so far that there is a unique template shape $Y$: the generative model has a Dirac distribution at $Y$ in the shape space. What happens for other distributions? We assume that there are $K$ template shapes $Y_1, ..., Y_K$. Observations are generated in $M$ from each template shape $Y_k$ with the generative model of Section 2. Our goal is to unveil the structure of the shape distribution, i.e. the $K$ template shapes here, given the observations in $M$. The distributions on shapes projected on the shape space is a mixture of probability density functions of the form of equation~\ref{eq:f}. Its modes are related to the template shapes. The K-means algorithm is a very popular method for data clustering. We study what happens if one uses K-means algorithms on shapes generated with the generative model above.

The goal is to cluster the shape data in $K$ distinct and significant groups. One performs a coordinate descent algorithm on the following function:
\begin{equation}
J(c,\mu) = \sum_i d_Q(X_i, \mu_{c_{i}})^2.
\end{equation}
In other words, the minimization of $J$ is performed through successive minimizations on the assignment labels $c$'s and the cluster's centers $\mu$'s. Given the $c$, minimizing $J$ with respect to the $\mu$'s is exactly the simultaneous computation of $K$ Fr\'echet means in the shape space. Meaningful well separated clusters (high inter-clusters dissimilarity) are chosen so that members are close to each other (high intra-cluster similarity). In other words, the quality of the clustering is evaluated by the following criterion:
\begin{equation}
D = \underset{\text{clusters } i,j}{\operatorname{min}}\quad\frac{d_Q(c_i,c_j)}{\underset{i}{\operatorname{max}} \text{ diam}(c_i)},
\end{equation} 
which is the dissimilarity between clusters quotiented by the diameter of the clusters. In the absence of singularity in the shape space, the projected distribution looks like Figure~\ref{fig:kmeans} (a) and $D \propto \frac{1}{\sigma}$. The criterion is worse in the presence of singularities. 

 \begin{figure}[!htbp]
  \centering
       \def\svgwidth{1\textwidth}
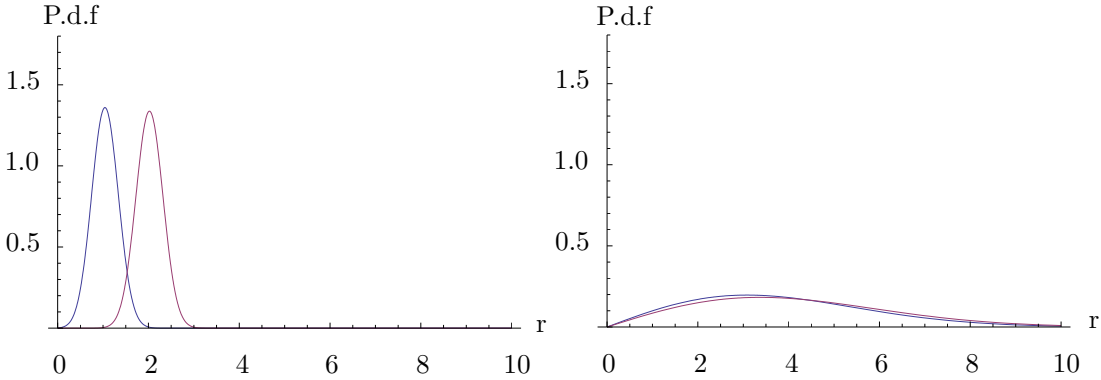
  \caption{Two clusters of template shapes for the plane example: $r_1=1$ (blue) and $r_2 =2$ (dark red). Noise levels: $\sigma =0.3$ (left) and $\sigma =3$ (right). The 2 clusters are hardly distinguishable when the noise increases.}
     \label{fig:kmeans}
\end{figure}

Figure~\ref{fig:kmeans} illustrates this behavior for the plane example. We consider any two clusters $i,j$ and call $\hat Y_i$, $\hat Y_j$ the estimated centroids. The criterion $D$ writes:
\begin{align*}
D \equiv \frac{\hat y_i - \hat y_j}{\sigma} \underset{\sigma \rightarrow + \infty}{\sim}\frac{\Gamma \left(\frac{m+1}{2}\right)}{\sqrt{2}m\Gamma\left(\frac{m}{2}\right)}\frac{y_i^2 - y_j^2}{\sigma^2} = O\left(\frac{1}{\sigma^2}\right).
\end{align*}
Even in the best case with correct assignments to the clusters $i$ and $j$, the K-means algorithm looses an order of validation when computed on shapes.

\paragraph{Beyond the finite dimensional case} Our results are valid when $M$ is a finite dimensional manifold and $G$ a finite dimensional Lie group. Some interesting examples belong to the framework of infinite dimensional manifold with infinite dimensional Lie groups. This is the case for the LDDMM framework on images \cite{Joshi2006}. It would be important to extend these results to the infinite dimensional case. 

We take $M=\mathbb{R}^m$ with the action of $SO(m)$. We have a analytic expression of $f$ in this case \cite{Miolane2015b}.  Figure~\ref{fig:finitedims} shows the influence of the dimension $m$ for the probability distribution functions on the shape space and for the Bias. The bias increases with $m$. This leads us to think that it appears in infinite dimensions as well.

 \begin{figure}[!htbp]
  \centering
       \def\svgwidth{1\textwidth}
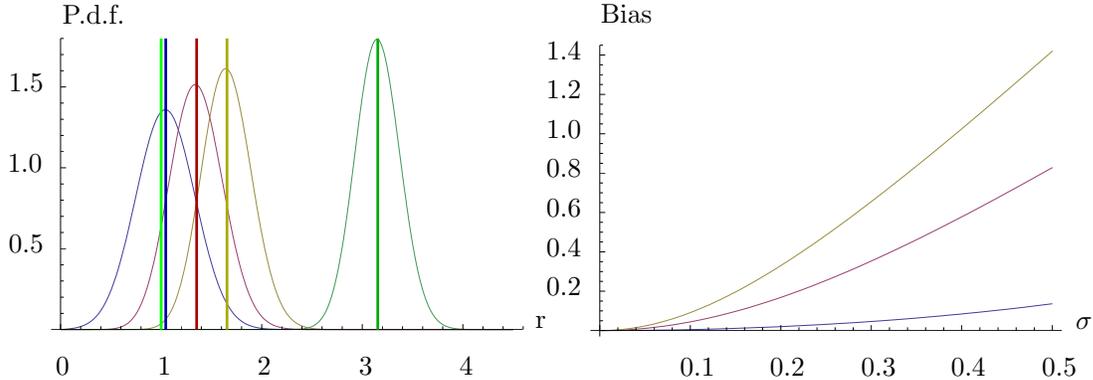
  \caption{Probability distributions functions (noise $\sigma =0.3$) and bias for $\mathbb{R}^m$ for $m=2$, $m=10$, $m=20$ and $m=100$. Template shape is $r=1$.}
     \label{fig:finitedims}
\end{figure}

\section{Correction of the systematic bias}\label{sec:correction}

We propose two procedures to correct the asymptotic bias on the template's estimate. They rely on the bootstrap principle, more precisely a parametric bootstrap, which is a general Monte Carlo based resampling method that enables us to estimate the sampling distributions of estimators \cite{Efron1979}. We assume that we know the variance $\hat \sigma^2$ from the experimental setting.

\subsection{Iterative Bootstrap}

The first procedure is called an Iterative Bootstrap. Algorithm~\ref{alg:iterative} gives the details. Figure~\ref{fig:iterativeBootstrap} illustrates it on the plane example.

\begin{figure}[h!]
\centering
       \def\svgwidth{0.6\textwidth}
       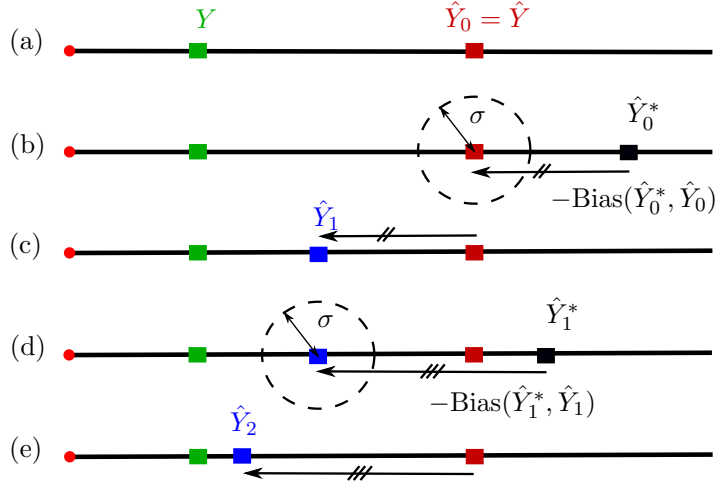
  \caption{Algorithm~\ref{alg:iterative} Iterative bootstrap procedure on the plane example for $n \rightarrow +\infty$. (a) Initialization, (b) Generate bootstrap sample from $\hat Y_{0}$ and compute the corresponding estimate $\hat{Y_{0}}^*$, compute the bias $\hat Y_{0} -\hat Y_{0}^*$, (c) Correct $\hat Y_{0}$ with the bias to get $\hat Y_{1}$, (d) Generate bootstrap sample from $\hat Y_{1}$ and iterate as in (b), (e) Get $\hat Y_2$ etc.}
   \label{fig:iterativeBootstrap}
\end{figure}

Algorithm~\ref{alg:iterative} starts with the usual template's estimate $\hat Y_0=\hat Y$, see Figure~\ref{fig:iterativeBootstrap} (a). At each iteration, we correct $\hat Y$ with a better approximation of the bias. First, we generate bootstrap data by using $\hat Y$ as the template shape of the generative model. We perform the template's estimation procedure with the Fr\'echet mean in the shape space. This gives an estimate $\hat Y_0^*$ of $\hat Y_0$. The bias of $\hat Y_0*$ with respect to $\hat Y_0$ is $\text{Bias}(\hat Y_0^*,\hat Y_0)$. It gives an approximation of the bias $\text{Bias}(\hat{\hat {Y}},\hat Y)$, see Figure~\ref{fig:iterativeBootstrap} (b). We correct $\hat Y$ by this approximation of the bias. This gives a new estimate $\hat Y_1$, see Figure~\ref{fig:iterativeBootstrap} (c). We recall that the bias $\text{Bias}(\hat{\hat {Y}},\hat Y)$ depends on $Y$, see Theorem \ref{th:bias}. $\hat Y_1$ is closer to the template $Y$ than $\hat Y_0$. Thus, the next iteration gives a better approximation $\text{Bias}(\hat Y_1^*,\hat Y_1)$ of $\text{Bias}(\hat{\hat {Y}},\hat Y)$. We correct the initial $\hat Y$ with this better approximation of the bias, etc. The procedure is written formally for a general manifold $M$ in Algorithm \ref{alg:iterative}.

\begin{algorithm}\label{alg:iterative}
\caption{Corrected template shape estimation with \textbf{Iterative Bootstrap}}

\noindent \textbf{Input:} Objects $\{X_i\}_{i=1}^n$, noise variance $\sigma^2$\\
\textbf{Initialization:}\\
\indent $\hat Y_0 = \text{Frechet}(\{[X_i]\}_{i=1}^n)$ \\
\indent $ k \leftarrow 0$\\
\textbf{Repeat:}\\
\indent Generate bootstrap sample $\{X^{(k)^*}_i\}_{i=1}^n$ from $\mathcal{N}_M(Y_k,\sigma^2)$\\
\indent $\widehat{Y_k} = \text{Fr\'echet}(\{[X^{(k)^*}]_i\}_{i=1}^n)$ \\
\indent $\text{Bias}_k = \text{Log}_{Y_k} \widehat{Y_k}$\\
\indent $\hat Y_k = \text{Exp}_{\hat Y_0}\left(-\Pi_{\hat Y_k}^{\hat Y_0} \left( \text{Bias}_k\right) \right)$\\
\indent $k \leftarrow k+1$\\
\noindent \textbf{until convergence:} $||\text{Log}_{\hat Y_{k+1}} \hat Y_k|| < \epsilon $\\
\noindent \textbf{Output: $\hat Y_k$}
\setlength{\parindent}{0ex}
\end{algorithm}

In Algorithm~\ref{alg:iterative}, $\Pi_A^B$ denotes the parallel transport from $T_AM$ to $T_BM$. For linear spaces like $\mathbb{R}^2$ in the plane example, $\text{Log}_{P_1}P_2 = \overrightarrow{P_1P_2}$,  $\text{Exp}_{P_1}(u) = P_1+u$, and the parallel transport is the identity $\Pi_{P_1}^{P_2}( u)=u$. For other manifolds like $S^2$ in our sphere example, the parallel transport $\Pi_A^B(u)$ can theoretically be computed by solving the parallel transport equation at any point on the chosen curve linking $A$ to $B$: $ D_{t_{AB}}v=0$ in $v$, where $D$ is the covariant derivative in the direction $t_{AB}$, the tangent vector of the curve at the chosen point \cite{Postnikov2001}. In practice, the Schild's ladder \cite{Ehlers2012} or the Pole ladder \cite{Lorenzi2013b} can be used to compute an approximation of the parallel transport.

Algorithm~\ref{alg:iterative} is a fixed-point iteration $Y^{(k+1)} = F(Y^{(k)})$ where:
\begin{equation}
F(X) = \text{Exp}_{\hat Y}( -\Pi_{X}^{\hat{Y}} \left( \text{Bias}\right))\qquad\text{where:  }\quad \text{Bias} = \text{Log}_{X} \hat X.
\end{equation}
In a linear setting we have simply $F(X) = \hat Y- \overrightarrow{X\hat X}$. One can show that $F$ is a contraction and that $Y$, the template shape, is the unique fixed point of $F$ (using the local bijectivity of the Riemannian exponential and the injectivity of the estimation procedure). Thus the procedure converges to $Y$ in the case of an infinite number of observations $n \rightarrow +\infty$. Figure~\ref{fig:iterativeBootstrap_FixedPoint} illustrates the convergence for the plane example, with a Gaussian noise of standard deviation $\sigma =1$. The template shape $Y=1.2$ was initially estimated at $\hat Y = 4.91$. Algorithm~\ref{alg:iterative} corrects the bias.

\begin{figure}[h!]
\centering
       \def\svgwidth{0.5\textwidth}
       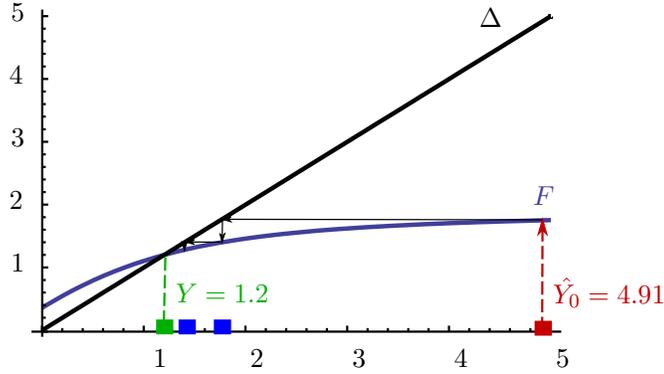
  \caption{$F$ of the fixed-point procedure and first 2 iterations for $\sigma=1$, $m=3$. $\Delta$ is the first diagonal. The initial estimate is biased $\hat{Y_0}=4.91$. The Iterative Bootstrap converges towards the template shape $Y=1.2$.}
   \label{fig:iterativeBootstrap_FixedPoint}
\end{figure}

Figures~\ref{fig:iterationsplane} and~\ref{fig:iterationssphere} show the iterations of Iterative Bootstrap for the plane and the sphere example.

\begin{figure}[h!]
\centering
       \def\svgwidth{0.8\textwidth}
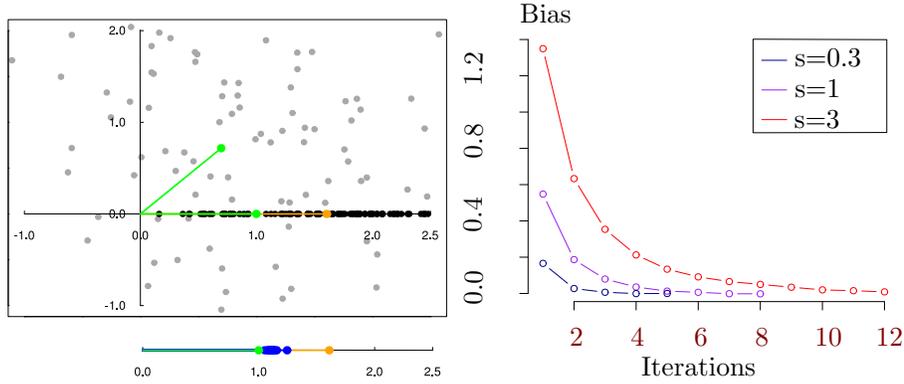
\caption{Left: Implementation of the plane example: the green point is the template shape $Y$, the grey points are the data $X_i$'s generated with the model~(\ref{eq:genModel}), the black points are the registered data $\hat g_i \cdot X_i$'s, the orange point is the template shape estimate $\hat Y$. The quotient space $\mathbb{R}_+$ is copied below, and the blue points show the iterations of the iterative bootstrap of Algorithm~~\ref{alg:iterative} that corrects the bias of $\hat Y$ as an estimate of $Y$: the blue points go from the orange point $\hat Y$ to the green point $Y$. Right: Convergence of the iterative bootstrap of Algorithm~\ref{alg:iterative}, for the plane example. The colors red, purple, blue represent different noises $\sigma$. The bias of $\hat Y$ as an estimator of $Y$ is shown on the ordinate axis: it converges to $0$ in a few iterations.\label{fig:iterationsplane}}
\end{figure}

\begin{figure}[h!]
\centering
       \def\svgwidth{0.8\textwidth}
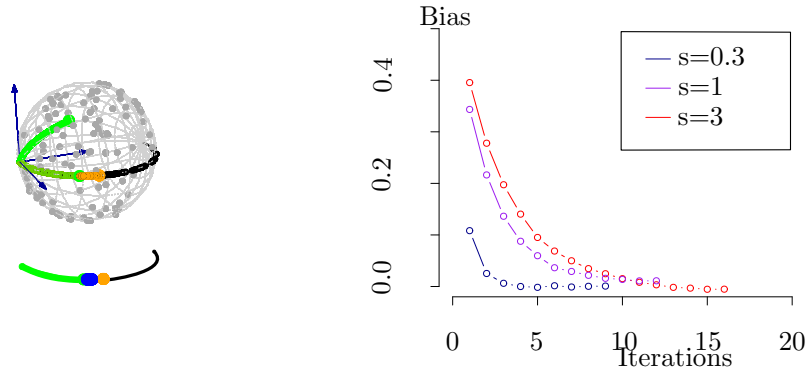
\caption{Left: Implementation of the sphere example: the green point is the template shape $Y$, the grey points are the data $X_i$'s generated with the model~(\ref{eq:genModel}), the black points are the registered data $\hat g_i \cdot X_i$'s, the orange point is the template shape estimate $\hat Y$. The quotient space $[0, \pi]$ is copied below, and the blue points show the iterations of the iterative bootstrap of Algorithm~~\ref{alg:iterative} that corrects the bias of $\hat Y$ as an estimate of $Y$: the blue points go from the orange point $\hat Y$ to the green point $Y$. Right: Convergence of the iterative bootstrap of Algorithm~\ref{alg:iterative}, for the sphere example. The colors red, purple, blue represent different noises $\sigma$. The bias of $\hat Y$ as an estimator of $Y$ is shown on the ordinate axis: it converges to $0$ in a few iterations.\label{fig:iterationssphere}}
\end{figure}

\subsection{Nested Bootstrap}

The second procedure is called the Nested Bootstrap. Algorithm~\ref{alg:nested} details it. Figure~\ref{fig:nestedBootstrap} illustrates it on the plane example.

\begin{figure}[h!]
\centering
       \def\svgwidth{0.6\textwidth}
       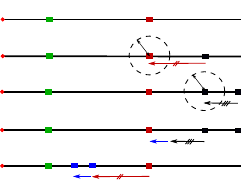
  \caption{Algorithm~\ref{alg:nested} Nested Bootstrap on the plane example for $n \rightarrow +\infty$. (a) Initialization, (b) Generate bootstrap sample from $\hat Y_{0}$; compute the estimate $\hat{Y_{0}}^*$, compute the bias $\hat Y_{(0)} -\hat Y_{0}^*$, (c) Generate bootstrap sample from $\hat Y_{0^*}$; compute the estimate $\hat{Y_{0}}^{**}$, compute the bias $\hat Y_{0}^* -\hat Y_{0}^{**}$, (d) compute the blue arrow, i.e. the bias of $\text{Bias}(\hat Y_0^{**},\hat{Y}_0^*)$ as an estimate of $\text{Bias}(\hat Y_0^*,\hat{Y}_0)$, (e) Correct $\hat Y$ with the bias-corrected bias.}
   \label{fig:nestedBootstrap}
\end{figure}

Algorithm~\ref{alg:nested} starts like Algorithm~\ref{alg:iterative} with $\hat{Y}_0 = \hat Y$, see Figure~\ref{fig:nestedBootstrap} (a). It also performs a parametric bootstrap with $\hat{Y}^{(0)}$ as the template, computes the bootstrap replication $\hat{Y}_0^*$ and the approximation $\text{Bias}(\hat Y_0^*,\hat{Y}_0)$ of $\text{Bias}(\hat Y,Y)$, see Figure~\ref{fig:iterativeBootstrap} (b). Now Algorithm~\ref{alg:nested} differs from Algorithm~\ref{alg:iterative}. We want to know how biased is $\text{Bias}(\hat Y_0^*,\hat{Y}_0)$ as an estimate of $\text{Bias}(\hat Y,Y)$? This is a valid question as the bias depends on the template $Y$, see Theorem~\ref{th:bias}. We want to estimate this dependence. We perform a bootstrap, nested in the first one, with $\hat{Y}^{(0)*}$ as the template. We compute the estimate $\hat{Y}_0^{**}$ and the approximation $\text{Bias}(\hat Y_0^{**},\hat{Y}_0^*)$ of $\text{Bias}(\hat Y_0^*,\hat{Y}_0)$, see Figure~\ref{fig:iterativeBootstrap} (c). We observe how far $\text{Bias}(\hat Y_0^{**},\hat{Y}_0^*)$  is from $\text{Bias}(\hat Y_0^*,\hat{Y}_0)$. This gives the blue arrow, which is the bias of $\text{Bias}(\hat Y_0^{**},\hat{Y}_0^*)$ as an estimate of $\text{Bias}(\hat Y_0^*,\hat{Y}_0)$, see Figure~\ref{fig:iterativeBootstrap} (d). The blue arrow is an approximation of how far $\text{Bias}(\hat Y_0^*,\hat{Y}_0)$ is from $\text{Bias}(\hat Y,Y)$. We correct our estimation of the bias (in red) by the blue arrow. We correct $\hat{Y}$ by the bias-corrected estimate of its bias, see Figure~\ref{fig:iterativeBootstrap} (e).

\begin{algorithm}\label{alg:nested}
\caption{Corrected template shape estimation with \textbf{Nested Bootstrap}}

\noindent \textbf{Input:} Objects $\{X_i\}_{i=1}^n$, noise variance $\sigma^2$\\
\textbf{Initialization:}\\
\indent $\hat Y_0 = \text{Frechet}(\{[X_i]\}_{i=1}^n)$ \\
\textbf{Bootstrap:}\\
\indent Generate bootstrap sample $\{X^{*}_i\}_{i=1}^n$ from $\mathcal{N}_M(\hat Y_0,\sigma^2)$\\
\indent $\hat Y_0^* = \text{Fr\'echet}(\{[X^{*}]_i\}_{i=1}^n)$ \\
\indent $\text{Bias} = \text{Log}_{\hat Y_0} \widehat{Y}_0^*$\\
\noindent \textbf{Nested Bootstrap:} \\
\indent For each $i$:
\begin{itemize}
\item Generate bootstrap sample $\{X^{**}_i\}_{k=1}^n$ from $\mathcal{N}_M(\hat Y_0^*,\sigma^2)$
\item $\hat Y_{0,i}^{**} = \text{Fr\'echet}(\{[X^{**}]_i\}_{k=1}^n)$
\end{itemize}
\indent $ \text{Bias}(\text{Bias}) = \text{Log}_{\hat Y_0} \widehat{Y}_0^* - \Pi_{\hat Y_0^*}^{\hat Y_0} \text{Log}_{Y_0^*} \widehat{Y}_0^{**}$\\
\indent $\hat Y_1 = \text{Exp}_{\hat Y_0}\left(-\text{Bias}-\text{Bias}(\text{Bias})\right)$\\
\noindent \textbf{Output: $\hat Y_1$}
\setlength{\parindent}{0ex}
\end{algorithm}

\subsection{Comparison}

One may use the Iterative Bootstrap or the Nested Bootstrap depending on the experimental setting. We illustrate them both on the plane example in Figure~\ref{fig:comparison}. Figure~\ref{fig:comparison} (a) shows the performance of both algorithms for a signal-over-noise ratio (SNR) of $1$: the template shape in green is a $r=1$ and the standard deviation of the noise is $\sigma = 1$, so that $\text{SNR}=\frac{r}{\sigma}=1$. Figure~\ref{fig:comparison} (b) shows both algorithms for $\text{SNR}=\frac{r}{\sigma} = \frac{1}{3} = 0.33$. In all four experiments: the template shape is the green dot at $r=1$, the template shape estimate is in orange, and the successive steps of the bootstrap algorithms are the blue dots: we have several blue dots for the Iterative Bootstrap, and two blue dots for the Nested Bootstrap.

The advantages of the Iterative Bootstrap are the following. It corrects the bias of $\hat Y$ perfectly in the case of a very large number of observations $n$, as we can see in Figures~\ref{fig:comparison} (a) on top and (b) on top: the blue dots converge to the green dot for the two different SNRs. Thus, the Iterative Bootstrap can be used to experimentally compute the mean curvature vector $H$ of each orbit of a group action. One probes the orbit's curvature by "feeling it" with a Riemannian Gaussian on $M$ and projecting on the shape space. The drawbacks of the Iterative Bootstrap are the following. It works only with very large $n$. It is not robust as it uses the generative model several times. If the generative model is far from being true, then the iterative bootstrap fails. 

The advantages of the Nested Bootstrap are the following. It is a standard statistical procedure that is more robust with respect to variations of the generative model. Even if generative model is different from the one that we assume, the Nested Bootstrap performs well. Moreover, it does not need as much data as the Iterative Bootstrap. Its drawback is that it does not correct perfectly the bias, especially when the noise is large. This can be seen in Figures~\ref{fig:comparison} (a) on bottom and (b) on bottom. While the Nested Bootstrap gets close to the green dot on Figure~\ref{fig:comparison} (a) on bottom for the $\text{SNR} = 1$, it stays significantly far from the green dot on Figure~\ref{fig:comparison} (b) bottom for the $\text{SNR} = 0.33$.

\begin{figure}[h!]
\centering
       \def\svgwidth{1\textwidth}
       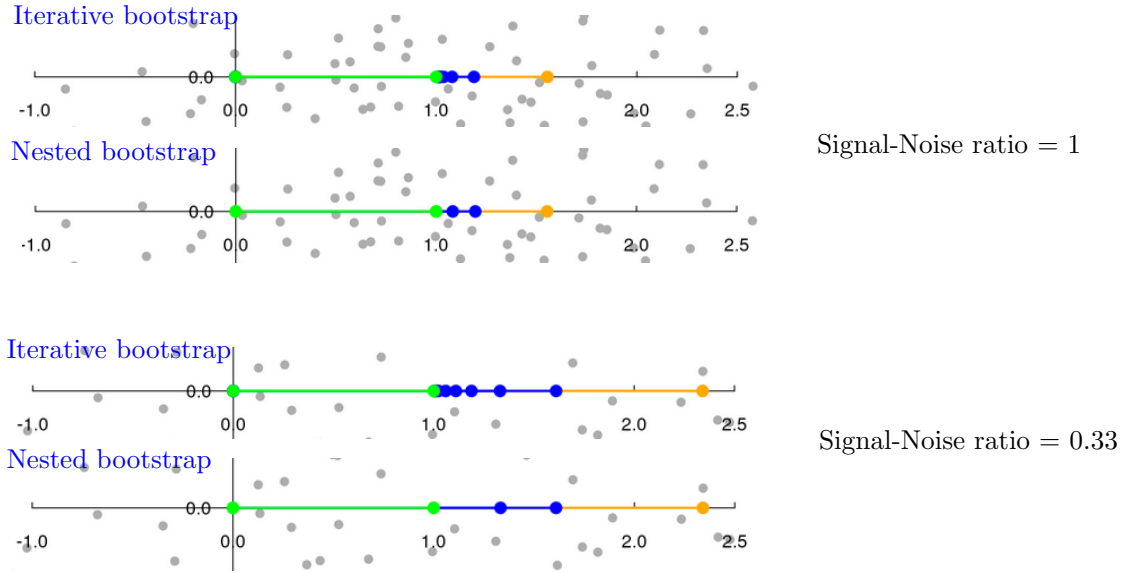
  \caption{Comparison of the Iterative bootstrap and the Nested bootstrap on simulation with two different Signal-Noise ratio, which is $\text{SNR} = \frac{Y}{\sigma} = \frac{r}{\sigma}$, the ratio of the template $Y$, which is the radius $r$ in the plane example, on the noise level $\sigma$. (a) shows $\text{SNR} =1$ and (b) shows $\text{SNR} = 0.33$. In all four experiments: the template shape is the green dot at $r=1$, the template shape estimate is in orange, and the successive steps of the bootstrap algorithms are the blue dots: we have several blue dots for the Iterative Bootstrap, and two blue dots for the Nested Bootstrap.\label{fig:comparison}}
\end{figure}

These simulations give a rule of thumb, i.e. some intuition, for when the bias needs to be corrected. They confirm what could already be observed in Figure~\ref{fig:bias}. In Figure~\ref{fig:bias}, the template is fixed at $r=1$ or $\theta = 1$. A variation in the noise level $\sigma$ corresponds to a variation in the SNR. In particular, we read the threshold $SNR = 1$ when $\sigma = 1$, i.e. when the noise $\sigma$ is comparable to the distance of the template $Y$ to the singularity. In both cases for $SNR > 1$, the template estimate is significantly different from the template as the bias is of the order of magnitude of the template itself. 

\section{Applications to simulated and real data}\label{sec:apps}

\subsection{Simulated triangles}

We perform a simulation using the iterative bootstrap on triangles. We randomly generate $n=10^5$ triangles in $\mathbb{R}^2$ through the generative model described in Equation~(\ref{eq:genModel}) of Section~\ref{sec:geom}. This is illustrated on Figure~\ref{fig:iterationstriangles}. We consider the isometric action of the Lie group $SO(2)$ of 2D rotations on $(\mathbb{R}^2)^3$, the space of 3 landmarks in 2D. For Step 1 of the generative model of Section~\ref{sec:geom}, the template triangle is chosen arbitrarily and then fixed during the simulations. The template triangle is represented in green in Figure~\ref{fig:iterationstriangles}. For Step 2, we consider a Dirac distribution at the identity in the Lie group $SO(2)$. In other words, we do not rotate the triangles. At the end of this step, each of the $10^5$ triangles is exactly the green triangle of Figure~\ref{fig:iterationstriangles}. This simpler model does not decrease the impact of the simulation: the noise of Step 3 is independent of the position of the triangle on their orbit, and Step (i) of the procedure is to quotient out the position of the orbit. For Step 3, we add bivariate Gaussian noise on each landmark, i.e. on each of the three points defining the green triangle. This gives a data set of $10^5$ triangles. Some of them are represented in grey on Figure~\ref{fig:iterationstriangles}.

We then apply the procedure described in Equations~(\ref{eq:procedure}) (i)-(ii) to estimate the (green) template triangle. In Step (i), we register the (grey) triangle data. This gives the registered the data, illustrated in black in Figure~\ref{fig:iterationstriangles}. We then compute the Fr\'echet mean of the black triangles by computing the Euclidean mean of each of their 3 landmarks. This gives the estimate of the template triangle, in orange on Figure~\ref{fig:iterationstriangles}.

\begin{figure}[h!]
\centering
       \def\svgwidth{0.8\textwidth}
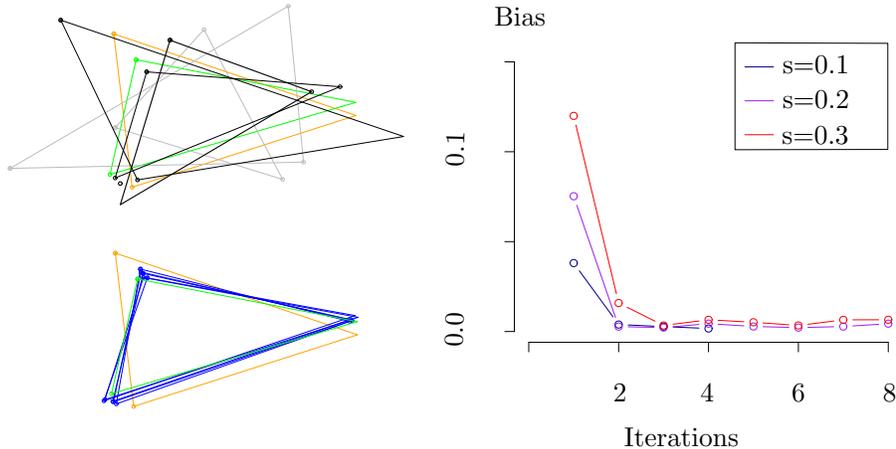
\caption{Left: Implementation for simulated triangles: the green triangle is the template shape $Y$, the grey triangles are (some of) the data $X_i$'s generated with the model~(\ref{eq:genModel}), the black triangles are (some of) the registered data $\hat g_i \cdot X_i$'s, the orange triangle is the template shape estimate $\hat Y$. The blue triangles show the iterations of the iterative bootstrap of Algorithm~~\ref{alg:iterative} that corrects the bias of $\hat Y$ as an estimate of $Y$: the blue triangles go from the orange triangle $\hat Y$ to the green triangle $Y$. Right: Convergence of the iterative bootstrap of Algorithm~\ref{alg:iterative}. The colors red, purple, blue represent different noises $\sigma$. The bias of $\hat Y$ as an estimator of $Y$ is shown on the ordinate axis: it converges to $0$ in a few iterations.}\label{fig:iterationstriangles}
\end{figure}

The template estimate in orange is different from the template in green, even with a very high number of observations: $n=10^5$. We apply the iterative bootstrap to correct this bias. The number of iterations required for the convergence of Algorithm 1 with respect to the noise level are shown in Figure~\ref{fig:iterationstriangles}. We observe the convergence in the three experiments for less than 10 iterations.  
 
\subsection{Real triangles: shape of the Optic Nerve Head}

Now we go to real triangle data. We have 24 images of Rhesus monkeys' eyes, acquired with a Heidelberg Retina Tomograph \cite{Patrangenaru2015}. For each monkey, an experimental glaucoma was introduced in one eye, while the second eye was kept as control. One seeks a significant difference between the glaucoma and the control eyes. On each image, three anatomical landmarks were recorded: $S$ for the superior aspect of the retina, $N$ for the nose side of the retina, and $T$ for the side of the retina closest to the temporal bone of the skull. The data are matrices $\{X_i\}_{i=1}^n$ where the landmark coordinates form the rows. For the ONH example, $M$ is the space of $3$ landmarks in 3D, $M=(\mathbb{R}^3)^3$ and the rotations act isometrically on each object $X_i$.

\textbf{Analysis} This simple example illustrates the estimation of the template shape. We use the following procedure to compute the mean shape for each group. We initialize $\hat Y$ with $X_1$ and repeat the following two steps until convergence:
\begin{align*}
(1) \quad & \forall
 i \in \{1,...,n\},\quad \hat R_i = \underset{R \in SO(3)}{\text{argmin }} ||\hat Y - X_i.R||^2,  \quad \text{(register to the current mean shape)},\\
(2)\quad& \hat{Y} = \frac{1}{n}\sum_{i=1}^n X_i.\hat R_i \quad \text{(update the mean shape estimate)}.
\end{align*}
Figure~\ref{fig:ONH} shows the mean shapes $\hat Y^{\text{control}}$ of the control group (left) and $\hat Y^{\text{glaucoma}}$ of the glaucoma group (right) in orange, while the initial data are in grey. The difference between the two groups is quantified by the distance between their means: $||\hat Y^{\text{control}}-\hat Y^{\text{glaucoma}}||= 21.84\upmu$m. We want to determine if this analysis presents a bias that significantly changes the estimated shape difference between the groups.

\begin{figure}[htbp!]
\centering
     \def\svgwidth{1\textwidth}
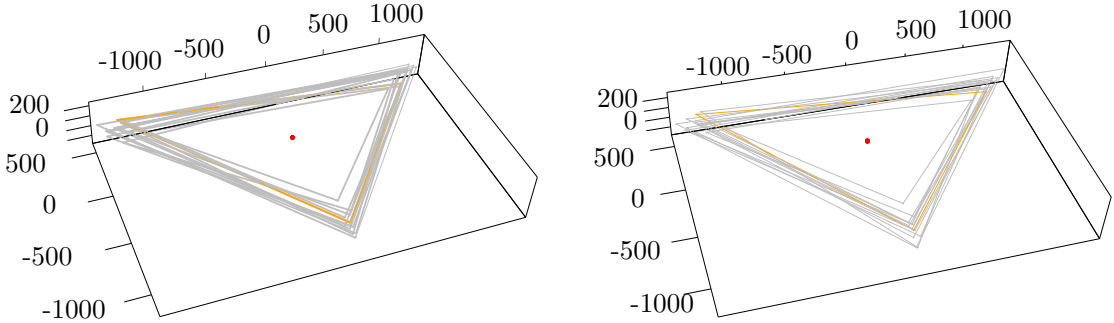
\caption{Triangles data in grey for the control group (left) and the glaucoma group (right). In orange, the estimated template shapes. Distances are measured in $\upmu$m.}\label{fig:ONH}
\end{figure}

We use the nested bootstrap to compute an approximation of the asymptotic bias on each mean shape, for a range of noise's standard deviation in $\{100\upmu\text{m}, 200\upmu\text{m}, 300\upmu\text{m}, 400\upmu\text{m}\}$. The asymptotic bias on the template shape of the glaucoma group is $\{0.1\upmu\text{m},0.11\upmu\text{m},0.12\upmu\text{m},$ $0.13\upmu\text{m}\}$ and of the control group is $\{0.27\upmu\text{m},0.42\upmu\text{m},0.55\upmu\text{m},0.67\upmu\text{m}\}$. The corrected template shape differences are $\{22.01\upmu\text{m}, 22.08\upmu\text{m}, 22.14\upmu\text{m}, 22.18\upmu\text{m}\}$. In particular, for $\sigma = 400\upmu\text{m}$, we observe that the bias in the template shape are respectively $0.67\upmu\text{m}$ for the healthy group and $0.13\upmu\text{m}$ for the glaucoma group. This follows the rule-of-thumb: the bias is more important for the healthy group, for which the overall size is smaller than the glaucoma group, for a same noise level. The bias of the template shape estimate accounts for less than $1\upmu\text{m}$ in this case, which is less than $0.1$\% of the shapes' sizes. This computation guarantees that this study has not been significantly affected by the bias. 
  
\subsection{Protein shapes in Molecular Biology}

We estimate the impact of the bias on statistics on protein shapes. This subsection aims to suggest the potential importance of the results of this paper for Molecular Biology.

A standard hypothesis in Biology is that structure (i.e. shape) and function of proteins are related. Fundamental research questions about protein shapes include structure prediction - given the protein amino-acid sequence, one tries to predict its structure - and design - given the shape, one tries to predict the sequence needed.

One relies on experimentally determined 3D structures gathered in the Protein Data Base (PDB) \cite{Berman00}. They contain errors on the protein's atoms coordinates. Average errors range from 0.01 $\mathring{\text{A}}$ to 1.76 $\mathring{\text{A}}$, which is of the magnitude of the length of some covalent bonds. These values are averaged over the whole protein and in general, the main-chain atoms are better defined than the side-chain atoms or the atoms at the periphery. This is illustrated on Figure~\ref{fig:bfactor} where we have plot the B-factor (related to coordinates errors \cite{Tickle1998}) as a colored map on the atoms for proteins of PDB-codes 1H7W and 4HBB. 

\begin{figure}[htbp!]
\centering
\includegraphics[scale=0.15]{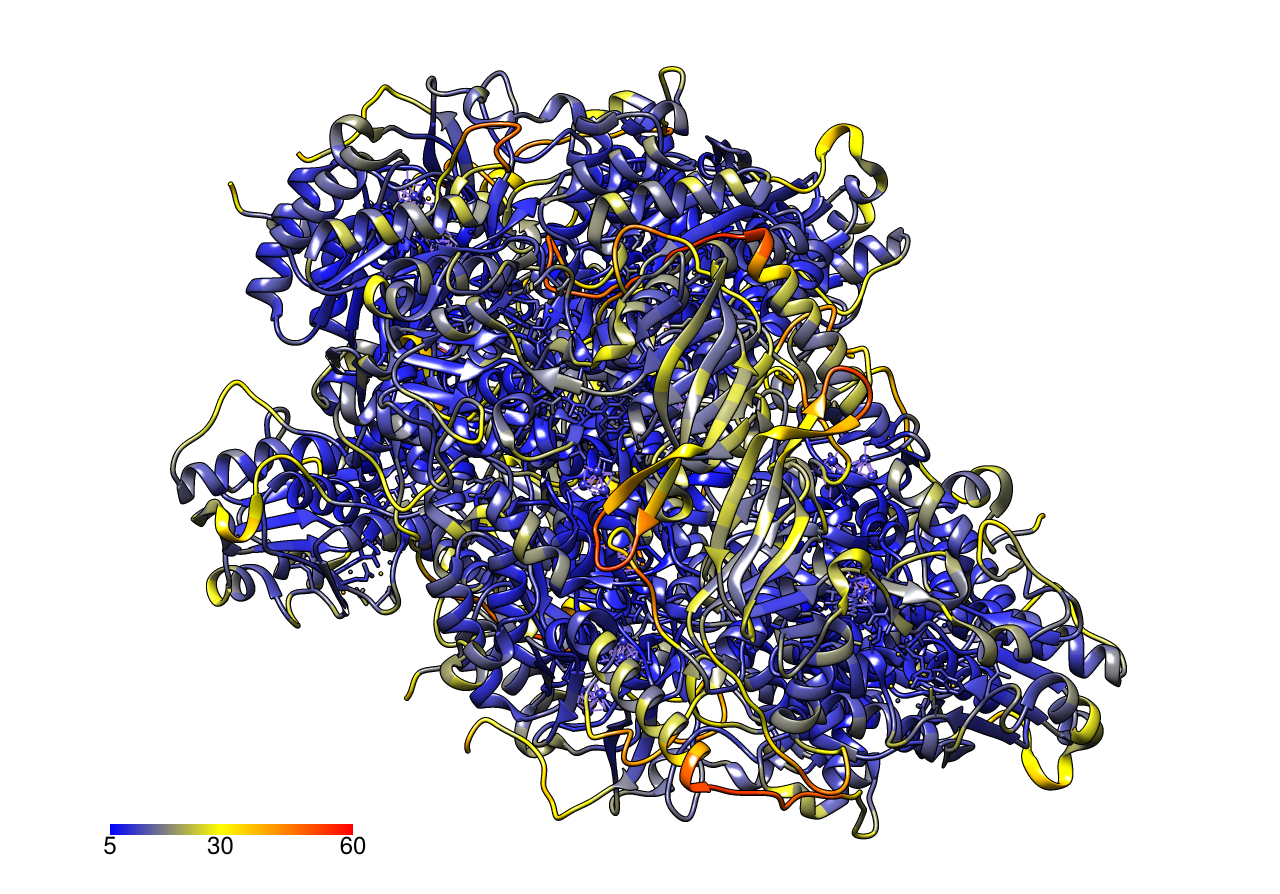}
\includegraphics[scale=0.15]{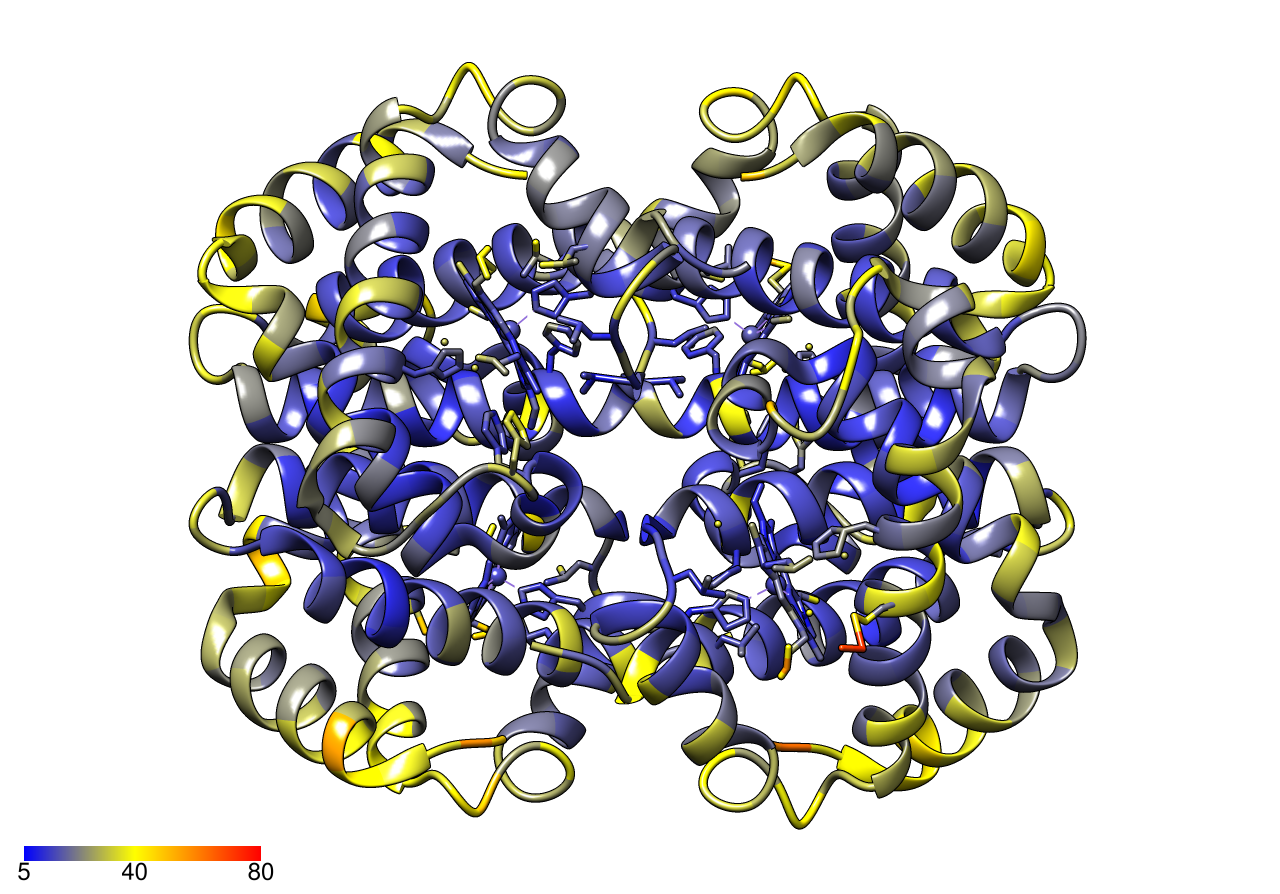}
\caption{Errors on atoms coordinates represented by the B-factor, for proteins 1H7W (left) and 4HBB (right). Atoms at the periphery of the proteins tend to have more errors, which appear in yellow-red colors.\label{fig:bfactor}}
\end{figure}

\paragraph{Protein's radius of gyration} A biased estimate of a protein shape has consequences for studies on proteins folding. Stability and folding speed of a protein depend on both the estimated shape of the denatured state (unfolded state) and of the native state (folded state). One may study if compact initial states yield to faster folding. The protein compactness is represented by the protein's Radius of Gyration, defined as: $R_g^2 = \frac{1}{N}\sum_{\text{non H atoms i}}^N (r_i - R_C)^2$, where $N$ is the number of non-hydrogen atoms, $r_i$, $R_C$ are resp. the coordinates of atoms and centers. Error on atoms coordinates give a bias on the estimate of the Radius of Gyration:
\begin{equation}
\mathbb{B}(R_g^2)= \sigma^2\frac{3(N-1)}{N} = \bar{R_g}^2\frac{(N-1)}{N}\frac{3}{\text{SNR}^2},
\end{equation}
where we also express this bias with respect to an adaptation of the signal-noise-ratio introduced in Section~\ref{sec:correction}: $\text{SNR}^2 = \frac{R_g^2}{\sigma^2}$.

The radius of protein HJSJ (85 residues) is known to be around 10 $\mathring{\text{A}}$. The error on  $R_g^2$ is of $0.3$\% with an average error of positions on the atoms of $0.3 \mathring{\text{A}}$. It is  8.6\% for an error of $1.7 \mathring{\text{A}}$. The error will be greater if one considers binding sites at the periphery of the proteins rather than the whole protein. Indeed sites' size is smaller and they have less atoms.

One could think about doing clustering on radii of Gyration using the K-means algorithm on shapes. The index $D$ of Section~\ref{sec:correction} is:
\begin{equation}
D= \frac{{R_1^\sigma}^2- {R_2^\sigma}^2}{\sigma} = \frac{R_1^2 -R_2^2}{\sigma} + 3 \sigma \left(\frac{N_1 -1}{N_1} - \frac{N_2-1}{N_2}\right).
\end{equation}
Clustering on radii of gyration may lead to a misleading indicator. $D$ indicates that the clustering performs better that it actually does.

\paragraph{False positive probability in protein's motif detection}

The relation between a protein's shape and function is linked to its motifs, which define the supersecondary structure. Motifs have biological properties: for example the helix–turn–helix motif \cite{Brennan1989}is responsible for the binding of DNA within several prokaryotic proteins. Automatic motif detection is another challenge in the study of protein shapes. We investigate the impact of bias on the false positive probability estimation in motif detection.

Let us consider a set $\{P_i\}_{i=1}^n$ of proteins each with $N_i$ atoms. One is interested in the motifs of $k$ atoms that can be detected in the protein's set, where $ k < N$. We define $\sigma$ that represents an allowed error zone. The number of detected motifs increases if: (i) one decreases $k$, or (ii) one increases $\sigma$, or (iii) increases $n$. Thus how many detected motifs actually come from chance, with respect to the parameters $k$, $\sigma$, $n$? The false positives probability indicates when one detects truth and when one detects noise. The usual estimate of the false positive probability is $P = \frac{\mathcal{V}_0}{\mathcal{V}_l}$. Here $\mathcal{V}_0$ is the volume of the error zone allowed. $\mathcal{V}_l$ is the total volume of the protein \cite{Pennec1998}, thus the a ball of radius the Radius of Gyration. Thus $\mathcal{V}_l$ may be biased and overestimated. The probability of false positive is underestimated.

We consider the example of \cite{Pennec1998b}. One tries to find motifs between the tryptophan repressor of Escherichia coli (PDB code 2WRP) and the CRO protein of phage 434 (PDB code 2CRO). These two proteins are known to share the helix-turn-helix motif. The radius of Gyration of 2WRP is $R_g = 20 \mathring{\text{A}}$, the total volume is: $\mathcal{V}_l = \frac{4}{3}\pi R_g^3 \simeq 33510 \mathring{\text{A}}^3$. We assume an error zone that takes the form of a diagonal covariance matrix with standard deviations $\sigma = 0.35 \mathring{\text{A}}$. We get the error zone volume: $\mathcal{V}_0 = \chi^3 \frac{4}{3}\pi \sigma^3 = 4.06 \mathring{\text{A}}^3$, where $\xi^2 = 8$  comes from a convention about how much error is allowed: the covariance of the error within the error volume shall be less than $\xi^2$, see \cite{Pennec1998b} for details. The estimation of the false positive probability is: $P = 1.2 \times 10^{-4}$. We find that $P$ is underestimated by $0.27 \%$ using the expression of the Radius of Gyration's bias.

\subsection{Brain template in Neuroimaging}

We apply the rule of thumb of Section~\ref{sec:correction} to determine when the bias needs a correction in the computation of a brain template from medical images. There are numerous technical difficulties for this application. First, $M$ and $G$ are now infinite dimensional. Then, the Lie group action is not necessarily isometric. Thus, it is clear that the results of the theorems do not apply directly. Nevertheless, this subsection still allows us to gain intuition about how this paper may impact the field of neuroimaging.

In neuroimaging, a template is an image representing a reference anatomy. Computing the template is often the first step in medical image processing. Then, the subjects' anatomical shapes may be characterized by their spatial deformations \textit{from the template}. These deformations may serve for (i) a statistical analysis of the subject shapes, or (ii) for automated segmentation by mapping the template's segmented regions into the subject spaces. In both cases, if the template is not centered among the population, i.e. if it is biased, then the analyzes and conclusions could be biased. We are interested in highlighting the variables that control the template's bias.

The framework of Large Deformation Diffeomorphic Metric Mapping (LDDMM) \cite{Younes2012} embeds the template estimation in our geometric setting. The Lie group of diffeomorphisms acts on the space of images as follows:
\begin{equation}
\rho : \text{Diff}(\Omega) \times L_2(\Omega) \rightarrow L_2(\Omega), \qquad (\phi,I) \mapsto \phi \cdot I = I \circ \phi^{-1}.
\end{equation}
The isotropy group of $I$ writes: $G_I = \{\phi \in \text{Diff}(\Omega)|  I\circ \phi^{-1} = I\}$. Its Lie algebra $\mathfrak{g}_I$ consists of the infinitesimal transformations whose vector fields are parallel to the level sets of $I$: $\mathfrak{g}_I = \{ v | \forall x \in \Omega, \nabla I(x)^T .v(x) = 0\}$. The orbit of $I$ is : $O_I = \{I' \in L_2(\Omega)| \exists \phi \in \text{Diff}(\Omega)\text{ s.t. } I'\circ \phi^{-1} = I\}$. 

The "shape space" is by definition the space of orbits. Two images that are diffeomorphic deformations of one another are in the same orbit. They correspond to the same point in the shape space. Topology of an image is defined as the image's properties that are invariant by diffeomorphisms. Consequently, the shape space is the space of the images topology, represented by the topology of their level sets. We get a stratification of the shape space when we gather the orbits by orbit type. A stratum is more singular than another, if it has higher orbit type, i.e. larger isotropy group. 

The manifold $M$ has an infinite stratification. One changes stratum every time there is a change in the topology of an image's level sets. Singular strata are connected to simpler topology. "Principal" strata are connected to more complicated topology. Indeed, the simpler the topology of the level sets is, the higher is the "symmetry" of the image. Thus the larger is its isotropy group. Note that strata with smaller isotropy group (more detailed topology) do not represent "singularities" from the point of view of a given image and do not influence the bias. In fact, such strata are at distance 0: an infinitesimal local change in intensity can create a maximum or minimum, thus complexifying the topology. 

Using the rule-of-thumb of Section~\ref{sec:correction}, the template's bias depends on its distance $d$ to the next singularity, at the scale of $\sigma$ the intersubjects variability. The template is biased in the regions where the difference in intensity between maxima and minima is of the same amplitude as the variability. The template may converge to pure noise in these regions.

\section*{Conclusion}

We introduced tools of statistics on manifolds to study the properties of template's shape estimation in Medical imaging and Computer vision. We have shown asymptotic bias by considering the shape space's geometry. The bias comes from the external curvature of the template's orbit at the scale of the noise on the data. This provides a geometric interpretation for the bias observed in \cite{Allassonniere2007,Allassonniere2015b}. We investigated the case of several templates and the performance K-mean algorithms on shapes: clusters are less well separated because of each centroid's bias. The variables controlling the bias are: (i) the distance in shape space from the template to a singular shape and (ii) the noise's scale. This gives a rule-of-thumb for determining when the bias is important and needs correction. We proposed two procedures for correcting the bias: an iterative bootstrap and a nested bootstrap. These procedures can be applied to any type of shape data: landmarks, curves, images, etc. They also provide a way to compute the external curvature of an orbit. 

Our results are exemplified on simulated and real data. Many studies use the template's shape estimation algorithm in Molecular Biology, Medical Imaging or Computer vision. Their estimations are necessarily biased. But these studies often belong to a regime where the bias is not important (less than $0.1\%$). For example, the bias is important in landmark shapes analyses when the landmarks' noise is comparable to the template shape's size. Studies are rarely in this regime. We have considered shapes belonging to infinite dimensional shape spaces. Our results do not apply to the infinite dimensional case. We have used them to gain intuition about it. The bias might be more important in infinite dimensions and needs a correction as we have suggested.

\bibliographystyle{siamplain}
\makeatletter
\renewcommand\@biblabel[1]{#1. }
\makeatother


\appendix
\addcontentsline{toc}{subsection}{Appendix}

\section{Notation} 

We summarize in this appendix the main elements of the proofs of theorems \ref{th:pdf} and \ref{th:bias}.

\begin{figure}[!htbp]
  \centering
     \def\svgwidth{0.6\textwidth}
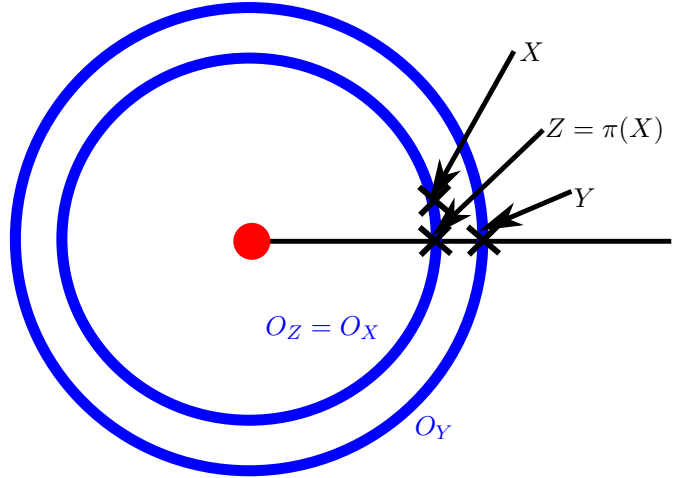
  \caption{Summary of the notations used in the proofs.}
   \label{fig:notations}
\end{figure} 

\paragraph{\textbf{Points (independent of coordinates systems)}} Points are denoted with upper cases. $Y$ is the template shape $Y$ and $\hat Y$ is its estimate. Here $X$ is a point in $M$. We consider that $X$ belongs to a principal orbit. This will have no impact on the integration because the set of principal orbits is dense in $M$. We write $Z=\pi(X)$ the projection of $X$ in the shape space. Figure~\ref{fig:notations} shows the elements $Y$, $X$, $O_Z$, $\pi(X)$. We denote $m$ the dimension of $M$, $p$ the dimension of the principal orbits and $q$ the dimension of the quotient space.

\paragraph{\textbf{Normal coordinate systems}} We often use Normal Coordinate Systems (NCS) to express the coordinates of tangent vectors. We refer to \cite{Postnikov2001} for theoretical developments about the NCS and to \cite{Brewin2009} for Taylor expansions of differential geometric tensors in a NCS. 

For example, we may consider a NCS centered at the point $Y$, with respect to the Riemannian metric of $M$. This NCS is valid on an open neighborhood of $Y$, that is start-shaped domain around $Y$. Moreover, we assume that $M$ is geodesically complete : thus, this domain is equal to the whole manifold $M$, with the exception of the cut locus which is of null measure \cite{Postnikov2001}. The Riemannian logarithm of a point $X$ in the NCS at $Y$ is denoted: $\overrightarrow{YX}$.

\section{Preliminaries}

\subsection{Truncated Gaussian moments in a Euclidean space\label{sec:eucl}}

We give preliminary computations of Gaussian moments in a $m$-dimensional vector space $\mathbb{R}^m$, using the curved notation $\mathcal{M}$. We refer to the \textit{(unnormalized) moment of order $k$ of the $m$-dimensional Gaussian of covariance $\sigma^2 A$} as:
\begin{align*}
\mathcal{M}^{i_1...i_k}(\sigma^2 A) & = \int_{\mathbb{R}^m}X^{i_1}...X^{i_k}\exp \left(-\frac{X^T A^{-1} X}{2\sigma^2}\right)dX
\end{align*}

The order 0, 2 and 4 are:
\begin{equation}
\boxed{\begin{aligned}
\mathcal{M}^0(\sigma^2.A) & = \sigma^{m}\sqrt{(2\pi)^m} \sqrt{\det\left(A\right)}\\
\mathcal{M}^{ab}(\sigma^2.A) & = \sigma^{m+2}\sqrt{(2\pi)^m} \sqrt{\det\left(A\right)}.A^{ab}\\
\mathcal{M}^{abcd}(\sigma^2.A) &= \sigma^{m+4}\sqrt{(2\pi)^m} \sqrt{\det\left(A\right)} .\left( {A^{ab}}{A^{cd}} + {A^{ac}}{A^{bd}}+{A^{ad}}{A^{bc}} \right)\\
\mathcal{M}^{i_1...i_k}(\sigma^2.A) &= \Theta\left(\sigma^{m+k}\right)^{i_1...i_k} \quad \text{if $k$ even}\\
& =0 \quad \text{if $k$ odd}
\end{aligned}} 
\end{equation}
where $\Theta$ denotes the proportionality to $\sigma^{m+k}$.

The \textit{truncated (unnormalized) moment at radius $r$ of the $m$-dimensional Gaussian of covariance $\sigma^2 A$} are defined as:
\begin{align*}
\mathcal{M}_r^{i_1...i_k}(\sigma^2 A) & = \int_{B_r}X^{i_1}...X^{i_k}\exp \left(-\frac{X^T A^{-1} X}{2\sigma^2}\right)dX
\end{align*}
where the integration domain is now the $m$-dimensional ball $B_r$.

They write,  with respect to the total moments:
\begin{equation}
\boxed{\mathcal{M}_r^{i_1...i_k}(\sigma^2 A) = \mathcal{M}^{i_1...i_k}(\sigma^2 A) + \epsilon(\sigma)}
\end{equation}
where $\sigma \rightarrow \epsilon(\sigma)$ is a function that decreases exponentially for $\sigma \rightarrow 0$.

\subsection{Isotropic Gaussian Moments on a Riemannian manifolds \label{sec:riem}}

We turn to computations of Gaussian moments in a $m$-dimensional Riemannian manifold $M$, using the notation $\mathfrak{M}$. We refer to the \textit{(unnormalized) moment of order $k$ of the $m$-dimensional isotropic Gaussian of covariance $\sigma^2 \mathbb{I}$} as:
\begin{align*}
{\mathfrak{M}}^{i_1...i_k}\left(\sigma^2 \mathbb{I}\right) = \int_{M} \overrightarrow{YX}^{i_1}...\overrightarrow{YX}^{i_k} \exp \left(-\frac{d_M^2(X,Y)}{2\sigma^2}\right)dM(X)
\end{align*}
and to the \textit{truncated (unnormalized) moment at radius $r$ of the $m$-dimensional isotropic Gaussian of covariance $\sigma^2 \mathbb{I}$} as:
\begin{align*}
{\mathfrak{M}}^{i_1...i_k}_r\left(\sigma^2 \mathbb{I}\right)& = \int_{B_r}\overrightarrow{YX}^{i_1}...\overrightarrow{YX}^{i_k}\exp \left(-\frac{d_M^2(X,Y)}{2\sigma^2}\right)d\overrightarrow{YX}
\end{align*}
where the integration domain is now the $m$-dimensional geodesic ball $B_r$ of radius $r$ and centered at $Y$.

First, we consider the truncated moments:
\begin{equation}
\boxed{{\mathfrak{M}}^{i_1...i_k}_r\left(\sigma^2 \mathbb{I}\right) = \sigma^{k+m}\mathcal{M}^{i_1...i_k}(\mathbb{I})+ \frac{\sigma^{k+m+2}}{6} \sum_{ab} R_{ab}(Y)\mathcal{M}^{abi_1...i_k}(\mathbb{I}) + \mathcal{O}(\sigma^{m+k+4}) + \epsilon(\sigma)}
\end{equation}

They write, with respect to the non-truncated moments:
\begin{align*}
{\mathfrak{M}}^{i_1...i_k}\left(\sigma^2 \mathbb{I}\right) = {\mathfrak{M}}^{i_1...i_k}_r\left(\sigma^2 \mathbb{I}\right) + \int_{\mathcal{C}_{B_r}} X^{i_1}...X^{i_k} \exp \left(-\frac{d_M^2(Y,X)}{2\sigma^2} \right)  dM(X)
\end{align*}
The second term of the sum is negligible when $\sigma \rightarrow 0$.
\begin{equation}
\boxed{{\mathfrak{M}}^{i_1...i_k}\left(\sigma^2 \mathbb{I}\right)= \sigma^{k+m}\mathcal{M}^{i_1...i_k}(\mathbb{I})+ \frac{\sigma^{k+m+2}}{6} \sum_{ab} R_{ab}(Y)\mathcal{M}^{abi_1...i_k}(\mathbb{I}) + \mathcal{O}(\sigma^{m+k+4})+ \epsilon(\sigma).}
\end{equation}
where $\epsilon(\sigma)$ decreases exponentially when $\sigma \rightarrow 0$.

\section{Proof of Theorem~\ref{th:pdf}: Induced probability density on shapes\label{sec:th1}}

\begin{figure}[!htbp]
  \centering
     \def\svgwidth{0.6\textwidth}
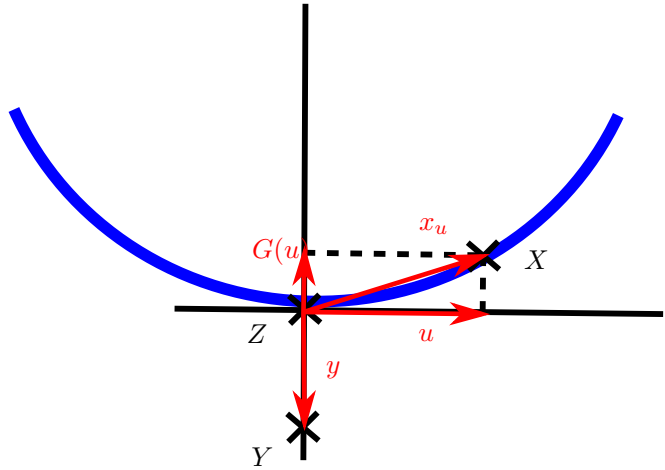
  \caption{Summary of the notations used in the roof of Theorem~\ref{th:pdf}.}
   \label{fig:notations2}
\end{figure}

The generative model implies the following Riemannian Gaussian distribution on the objects:
\begin{equation}
f(Z)=\frac{1}{C_M(\sigma)} \exp\left(-\frac{d_M^2(X,Y)}{2\sigma^2}\right),
\end{equation}
where $C_M(\sigma)$ is the integration constant:
\begin{equation}
C_M(\sigma) =\int_{M} \exp \left(-\frac{d_M^2(X, Y)}{2\sigma^2} \right)dM(X)
\end{equation}

We compute the induced probability distribution $f$ on shapes by integrating the distribution on the orbit of $X$ out of $f(Z)$:
\begin{align*}
f(Z) & = \frac{1}{C_M(\sigma)}\int_{X \in O_Z} \exp \left(-\frac{d_M^2(Y,X)}{2\sigma^2}\right)dO_Z(X).
\end{align*}
where $Z = \pi(X)$ is fixed in the following computations. We give a Taylor expansion of $f(Z)$ for $\sigma \rightarrow 0$ to the order $\sigma^2$.

We start by dividing the integral:
\begin{align*}
f(Z) & = \frac{1}{C_M(\sigma)}. \left(\int_{O_Z \cap B_r} \exp \left(-\frac{d_M^2(Y,X)}{2\sigma^2}\right) dO_Z(X) + \int_{O_Z \cap \mathcal{C}_{B_r}} \exp \left(-\frac{d_M^2(Y,X)}{2\sigma^2}\right)dO_Z(X)\right) 
\end{align*}
In the parenthesis, we call the first integral $I_Z(\sigma)$ and the second term is an $\epsilon_Z(\sigma)$:
\begin{align*}
f(Z) & = {C_M(\sigma)}^{-1}.\left( I_Z(\sigma)+ \epsilon_Z(\sigma) \right)
\end{align*}

We compute the Taylor expansions of $C_M(\sigma)^{-1}$ and $I_Z(\sigma)$ for $\sigma \rightarrow 0$ and show that $\epsilon_Z(\sigma) = o\left(\exp \left(-\frac{r^2}{2\sigma^2}\right)\right)$

\subsection{Taylor expansion of $C_M(\sigma)^{-1}$}

The integration constant $C_M(\sigma)$ is the moment of order $0$ of the $m$-dimensional isotropic Gaussian of variance $\sigma^2 \mathbb{I}$ in the $m$-dimensional manifold $M$.
\begin{align*}
{C_M(\sigma)}^{-1} &= (\sqrt{2\pi}\sigma)^{-m}.\left( 1 - \frac{\sigma^{2} }{6}R(Y) + O(\sigma^{3}) + \epsilon(\sigma)\right)
\end{align*}

\subsection{Taylor expansion of $I_Z(\sigma)$}

Now we compute the integral $I_Z(\sigma)$ involved in the formula of $f(Z)$:
\begin{align*}
I_Z(\sigma) & = \int_{O_Z \cap B_r} \exp \left(-\frac{d_M^2(Y,X)}{2\sigma^2}\right)dO(X)
\end{align*}
We denote: $B_{r_Z}= B_{r_Z}$. We first perform the computations for any $\sigma$. We recall that $r$ is fixed, and small enough in a sense made precise later.

\subsubsection{Computing $d_M^2(Y,X)$} The first component of $I_Z(\sigma)$ is $d_M^2(Y,X)$. We express the Taylor expansion of $d_M^2$ for $X \in O_Z$ close to $Z$. 

First, we parameterize the point $X$. The points $X$, $Z=\pi(X)$ and the orbit $O_Z \subset M$ are illustrated on Figure~\ref{fig:notations}: the orbit $O=O_Z$ is the blue circle in $M = \mathbb{R}^2$ going through $X$ and $Z$. The orbit $O_Z$ can be seen in the tangent space $T_{Z}M$ through the Logarithm map at $Z$. For $r$ small enough, i.e. for $X$ close enough to $Z$, we can locally represent $O_Z$ in $T_{Z} M$ as the graph of a smooth function $G$ from $T_{Z} O$ to $N_{Z}O$, using the vector $u \in T_{Z}O_Z$ around $u=0$:
\begin{align*}
I: T_{Z}O_Z & \mapsto T_ZM=T_{Z}O_Z \oplus N_{Z} O_Z\\
u & \mapsto x_u = (u, G(u))
\end{align*}
The local graph $u \rightarrow G(u)$ is illustrated on Figure~\ref{fig:notations} and has the following Taylor expansion around $u = 0$ in the NCS at $Z$: 
\begin{align*}
G(u)^a = \frac{1}{2}h^a_{bc}(Z)u^bu^c + G_3(Z)^a_{bcd}u^bu^cu^d + G_4(Z)^a_{bcde}u^bu^cu^d u^e+ O(||u||^5)
\end{align*}
The 0-th and 1-th order derivatives of $u \rightarrow G(u)$ are zero because the graph goes through $Z$ and is tangent at $T_{Z}O_Z$. The second order derivative is by definition the second fundamental form $h(Z)$ introduced in Section~\ref{sec:quant}: $h(Z)$ represents the best quadratic approximation of the graph $G$. The third and fourth orders $G_3(Z)$ and $G_4(Z)$ are further refinements on the shape of the graph $G$ around $Z$.

Second, we compute the Taylor expansion of $d_M^2(Y,X)$ with respect to $u$:
\begin{equation*}
\boxed{d_M^2(Y,X) = d_M^2(Y,Z) + u^aL_a(Z)+ u^au^bM_{ab}(Z)+u^au^bu^cP_{abc}(Z)+u^au^bu^cu^dS_{abcd}(Z)+ \mathcal{O}(||u||^5)}
\end{equation*}
where:
\begin{equation}
\boxed{
\begin{aligned}
L_a(Z) & =0 \\
M_{ab}(Z) & = -y^ch_{ab}(Z)^c-\frac{1}{3}R_{acbd}(Z)y^cy^d
 +\frac{1}{12}\nabla_dR_{aebc}(Z)y^dy^ey^c \\
& - \frac{1}{180}(44R_{geaf}(Z)R_{gcbd}(Z)+3\nabla_{ef}R_{acbd}(Z))y^ey^fy^cy^d \\
& - \frac{1}{54}R_{hfag}(Z)\nabla_cR_{hdbe}(Z)y^fy^gy^cy^dy^e \\
\end{aligned}
}
\end{equation}
and $P_{abc}(Z)$ and $S_{abcd}(Z)$ are tensors mixing  the derivatives of the graph $G$ and the derivatives of the Riemannian curvature $R$.

\subsubsection{Computing $dO_Z(X)$} The second component of $I_Z(\sigma)$ is the measure of the orbit $dO_Z(X)$. We seek the Taylor expansion of the measure:
\begin{equation}
\boxed{dO_Z(X) = \left(1+ T_c(Z) u^c - N_{cb}(Z) u^bu^c+\mathcal{O}(||u||^3)\right)du}
\end{equation}
where:
\begin{equation}
\boxed{
\begin{aligned}
T_{c}(Z) & = h^a_{ca}(Z)\\
N_{cb}(Z) & = \frac{1}{6}\text{Ric}_{cb}(Z) +\frac{1}{2} h^a_{ca}(Z) h^a_{ca}(Z)-\frac{1}{2} h^a_{cd}(Z)h^d_{ba}(Z)
\end{aligned}}
\end{equation}

\subsubsection{Gathering to compute $I_Z(\sigma)$}

We plug the expressions of $d_M^2(Y,X)$ and $dO_Z(X)$, computed in the previous subsections, in the expression of $I_Z(\sigma)$ to get:
\begin{flalign*}
I_Z(\sigma) & =\exp \left(-\frac{d_M^2(Y,Z)}{2\sigma^2} \right) \left(\sigma^pm_0(Z)+\sigma^{p+2}m_2(Z) +\Theta(\sigma^4) + \epsilon(\sigma) \right)&
\end{flalign*}
where:
\begin{equation}
\boxed{
\begin{aligned}
m_0(Z) & = \mathcal{M}(M)^0 = \sqrt{(2\pi)^p} \sqrt{\det\left((M_{ab}(Z)^{-1}\right)}\\
m_2(Z) &= N_{cb}(Z)\mathcal{M}(M^{-1}(Z))^{bc}-\frac{1}{2}S_{abcd}(Z)\mathcal{M}(M^{-1}(Z))^{abcd}
\end{aligned}
}
\end{equation}
where $M_{ab}$ and $S_{abcd}$ are given in the previous subsections.

\subsection{Upper bound on $\epsilon_Z(\sigma)$}

Ultimately, we show that:
\begin{equation}
\epsilon_Z(\sigma) = o \left(\exp\left(-\frac{r^2}{2\sigma^2}\right)\right)
\end{equation}
and $\epsilon_Z$ is exponentially decreasing when $\sigma \rightarrow 0$.

\subsection{Final result: Taylor expansion of $f(Z)$}

Replacing the terms in the expression of $f(Z)$ and get:
\begin{align*}
f(Z) &  = \frac{\exp \left(-\frac{d_M^2(Y,Z)}{2\sigma^2} \right)}{(\sqrt{2\pi}\sigma)^q}\left( F_0(Z) +\sigma^2 F_2(Z)+\Theta_Z(\sigma^{4})+  \epsilon(\sigma)\right)
\end{align*}
where:
\begin{equation}
\boxed{
\begin{aligned}
F_0(Z) & =\sqrt{\det\left((M_{ab}(Z)^{-1}\right)}\\
F_2(Z) & = - \frac{1}{6} \sqrt{\det\left((M_{ab}(Z)^{-1}\right)}R(Z)  + \frac{m_2(Z)}{(\sqrt{2\pi})^p}
\end{aligned}
}
\end{equation}
and we refer to the previous subsections for the formula of $M_{ab}(Z)$ and $m_2(Z)$.

\section{Proof of Theorem~\ref{th:bias}: Bias on the template shape}

We compute the bias $\text{Bias}(Y,\hat{Y})$ of $\hat Y$ as an estimator of $Y$. In the following, we take a NCS at $Y$. In particular, the vector $\overrightarrow{YZ} = \text{Log}_YZ$ has coordinates written $z$.

The expectation of the distribution $f$ of shapes in $Q$ is $\hat Y$ by definition. The point $\hat Y$, expressed in a NCS at the template $Y$, gives $\text{Bias}(Y,\hat{Y})$, a tangent vector at $T_YM$ that indicates how much one has to shoot to reach the estimator $\hat Y$:
\begin{equation}
\text{Bias}(Y,\hat{Y}) = \text{Log}_Y\hat Y = \int_{Q} \overrightarrow{YZ} f(Z) dQ(Z).
\end{equation}

\noindent First, we take a ball of small radius $r$ in $Q$ and fix $r$. We split the integral:
\begin{flalign*}
\text{Bias}(Y,\hat{Y}) &= \int_{B_r^Q} \overrightarrow{YZ} f(Z) dQ(Z) + \int_{ C_{B_r^Q}} \overrightarrow{YZ} f(Z) dQ(Z)&
\end{flalign*}

\noindent By the result of the preliminaries, adapted to $Q$, the right part is a function $\sigma \rightarrow \epsilon(\sigma)$ that is exponentially decreasing for $\sigma \rightarrow 0$.
\begin{flalign*}
\text{Bias}(Y,\hat{Y}) &= \int_{B_r^Q} \overrightarrow{YZ} f(Z) dQ(Z) + \epsilon(\sigma)&
\end{flalign*}

\subsection{Using the result of Theorem~\ref{th:pdf}}

\noindent We plug the expression of the density $f$ using Theorem~\ref{th:pdf}, writing the Taylor expansions of the $F$'s terms around $z=0$:
\begin{align*}
F_0(Z) &= F_{00}(Y)+F_{01d}(Y)z^d+\mathcal{O}(||z||^2)\\
F_2(Z) &= F_{20}(Y)+F_{21d}(Y)z^d+\mathcal{O}(||z||^2)
\end{align*}
to get:
\begin{flalign*}
\text{Bias}(Y,\hat{Y})^a & = F_{01d}(Y)\sigma^2 \delta^{ad} +\epsilon(\sigma)&\\
& = F_{01}^a(Y)\sigma^2  +\mathcal{O}(\sigma^4)+ \epsilon(\sigma)
\end{flalign*}

\subsection{Computation of $F_{01}^a(Y)$: Taylor expansion of $F_0(Z)$ in the coordinate $z$}

The term $F_{01}^a(Y)$ is the first order coefficient in the Taylor expansion of $F_0(Z)$ around $Y$ in the coordinate $z$. We compute its Taylor expansion first in $y$ and then convert it into a Taylor expansion in $z$. We have $y=-z$ and we get:
\begin{align*}
F_{01}^a(Y) & = -\frac{1}{2}H^a(Y)
\end{align*}

\subsection{Final result: Taylor expansion of $\text{Bias}(\hat Y, Y)$}

Replacing $F_{01}^a(Y)$ by its value computed above:
\begin{align*}
\text{Bias}(Y,\hat{Y})^a & = -\frac{1}{2}.H^a(Y)\sigma^2  +\mathcal{O}(\sigma^4)+ \epsilon(\sigma)
\end{align*}
which is the final result.

\end{document}


\maketitle

These are the supplementary materials for the paper "Template shape estimation: correcting an asymptotic bias". We present the detailed proofs of the theorems. The precise statements of the theorems are given again in each section of this supplementary materials.

\section{Notations} 

We denote $Y$ the template shape and $\hat Y$ its estimate. $X$ is a point in the manifold $M$. We consider that $X$ belongs to a principal orbit and we recall that the set of principal orbits is dense in $M$. We write $Z=\pi(X)$ the projection of $X$ in the shape space $Q$. We denote $m$ the dimension of $M$, $p$ the dimension of the principal orbits and $q$ the dimension of the quotient space. Figure~\ref{fig:notations} shows the elements $Y$, $X$, $O_Z=O_X$, $Z=\pi(X)$. 

\begin{figure}[!htbp]
  \centering
     \def\svgwidth{0.6\textwidth}
\input{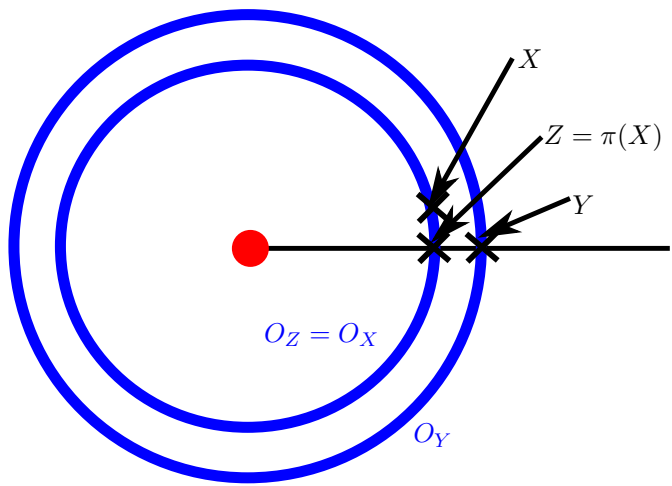}
  \caption{Summary of the notations used in the proofs. $Y$ is the template shape, $X$ is a point in $M$, belonging to a principal orbit $O_X$. $Z=\pi(X)$ is the projection of $X$ in the shape space.}
   \label{fig:notations}
\end{figure} 

\paragraph{\textbf{Normal Coordinate Systems}} We often use Normal Coordinate Systems (NCS) to express the coordinates of tangent vectors. We refer to \cite{Postnikov2001} for theoretical developments about the NCS and to \cite{Brewin2009} for Taylor expansions of differential geometric tensors in a NCS. 

For example, we may consider a NCS centered at the point $Y$, with respect to the Riemannian metric of $M$. This NCS is valid on an open neighborhood of $Y$, that is start-shaped domain around $Y$. Moreover, we assume that $M$ is geodesically complete : thus, this domain is equal to the whole manifold $M$, with the exception of the cut locus which is of null measure \cite{Postnikov2001}. The Riemannian logarithm of a point $X$ in the NCS at $Y$ is denoted: $\overrightarrow{YX}$.

\paragraph{\textbf{Asymptotic behavior for $\sigma \rightarrow 0$}} We denote: (i) $\Theta(\sigma^k)$ a function that is proportional to $\sigma^k$, (ii) $\mathcal{O}(\sigma^k)$ a function equivalent to $\sigma^k$ for $\sigma \rightarrow 0$ and (iii) $\epsilon(\sigma)$ a function that is exponentially decreasing for $\sigma \rightarrow 0$.

\section{Preliminaries}

\subsection{A first computation\label{sec:firstcomput}}

First, we show a technical result that will be used throughout the proofs. We show that the following integral on $T_YM$, the tangent space of $M$ at $Y$:
\begin{align*}
\int_{\mathcal{C}_{B_r}} d_M(Y,X)^k \exp \left(-\frac{d_M^2(Y,X)}{2\sigma^2} \right)  d\overrightarrow{YX}
\end{align*}
is a function $\sigma \rightarrow\epsilon(\sigma)$, i.e. decreases exponentially when $\sigma \rightarrow 0$. In the integral above, the notation $\mathcal{C}_{B_r}$ denotes the complement in $M$ of the geodesic ball $B_r$ of center $Y$ and of radius $r$.

We split the coordinates in $T_YM$ into $(\rho, u)$ (polar coordinates):
\begin{align*}
|| \int_{\mathcal{C}_{B_r}} d_M(Y,X)^k \exp \left(-\frac{d_M^2(Y,X)}{2\sigma^2} \right)  dX ||  & \leq K \int_{u\in S^m}\int_{r}^{\rho(u)} \rho^k \exp \left(-\frac{\rho^2}{2\sigma^2} \right) \rho^{m-1} d\rho dS^m\\
& \leq K \int_{u\in S^m}\int_{r}^{\rho(u)} \rho^{m+k-1} \exp \left(-\frac{\rho^2}{2\sigma^2} \right)  d\rho dS^m
\end{align*}
where $\rho(u)$ is the distance to the cutlocus in the direction $u$.

\noindent The positive integral is upper bounded by the same integral defined on the larger domain:
\begin{align*}
||\int_{\mathcal{C}_{B_r}} d_M(Y,X)^k \exp \left(-\frac{d_M^2(Y,X)}{2\sigma^2} \right)  dX ||  & \leq K \int_{u\in S^m}\int_{r}^{+\infty} \rho^{m+k-1} \exp \left(-\frac{\rho^2}{2\sigma^2} \right) d\rho dS^m
\end{align*}

\noindent Integrating the volume of the unit hypersphere:
\begin{align*}
||\int_{\mathcal{C}_{B_r}} d_M(Y,X)^k \exp \left(-\frac{d_M^2(Y,X)}{2\sigma^2} \right)  dX||  & \leq K \int_{r}^{+\infty} \rho^{m+k-1} \exp \left(-\frac{\rho^2}{2\sigma^2} \right) d\rho \\
\end{align*}
where the right-hand-side is dominated by $\exp(-\frac{r^2}{2\sigma^2})$ by the dominated convergence theorem. Therefore:
\begin{equation}
\int_{\mathcal{C}_{B_r}} d_M(Y,X)^k \exp \left(-\frac{d_M^2(Y,X)}{2\sigma^2} \right)  dX= \epsilon(\sigma)
\end{equation}
i.e. decreases exponentially when $\sigma \rightarrow 0$.

\subsection{Truncated Gaussian moments in a Euclidean space\label{sec:eucl}}

We give preliminary computations of Gaussian moments in a $m$-dimensional vector space $\mathbb{R}^m$, using the curved notation $\mathcal{M}$. We refer to the \textit{(unnormalized) moment of order $k$ of the $m$-dimensional Gaussian of covariance $\sigma^2 A$} as:
\begin{align*}
\mathcal{M}^{i_1...i_k}(\sigma^2 A) & = \int_{\mathbb{R}^m}X^{i_1}...X^{i_k}\exp \left(-\frac{X^T A^{-1} X}{2\sigma^2}\right)dX
\end{align*}
and to the \textit{truncated (unnormalized) moment at radius $r$ of the $m$-dimensional Gaussian of covariance $\sigma^2 A$} as:
\begin{align*}
\mathcal{M}_r^{i_1...i_k}(\sigma^2 A) & = \int_{B_r}X^{i_1}...X^{i_k}\exp \left(-\frac{X^T A^{-1} X}{2\sigma^2}\right)dX
\end{align*}
where the integration domain is now the $m$-dimensional ball $B_r$.

\subsubsection{First, we recall the expressions of some unnormalized Gaussian moments.} The order 0, 2 and 4 are:
\begin{equation}
\boxed{\begin{aligned}
\mathcal{M}^0(\sigma^2.A) & = \sigma^{m}\sqrt{(2\pi)^m} \sqrt{\det\left(A\right)}\\
\mathcal{M}^{ab}(\sigma^2.A) & = \sigma^{m+2}\sqrt{(2\pi)^m} \sqrt{\det\left(A\right)}.A^{ab}\\
\mathcal{M}^{abcd}(\sigma^2.A) &= \sigma^{m+4}\sqrt{(2\pi)^m} \sqrt{\det\left(A\right)} .\left( {A^{ab}}{A^{cd}} + {A^{ac}}{A^{bd}}+{A^{ad}}{A^{bc}} \right)\\
\mathcal{M}^{i_1...i_k}(\sigma^2.A) &= \Theta\left(\sigma^{m+k}\right)^{i_1...i_k} \quad \text{if $k$ even}\\
& =0 \quad \text{if $k$ odd}
\end{aligned}} 
\end{equation}
where we recall that $\Theta(\sigma^{m+k})$ denotes the proportionality to $\sigma^{m+k}$.

\subsubsection{Second, we turn to the (unnormalized) truncated moments.} They write,  with respect to the total moments:
\begin{align*}
\mathcal{M}_r^{i_1...i_k}(\sigma^2 A) & = \mathcal{M}^{i_1...i_k}(\sigma^2 A) - \int_{\mathcal{C}_{B_r}}X^{i_1}...X^{i_k}\exp \left(-\frac{X^TA^{-1} X}{2\sigma^2}\right)dX
\end{align*}
The second term of the sum is negligible when $\sigma \rightarrow 0$. To see this, we put an upper bound on its norm through triangular inequality:
\begin{align*}
|| \int_{\mathcal{C}_{B_r}}X^{i_1}...X^{i_k}\exp \left(-\frac{X^TA^{-1} X}{2\sigma^2}\right)dX|| \leq \int_{\mathcal{C}_{B_r}}||X^{i_1}||...||X^{i_k}||\exp \left(-\frac{X^TA^{-1} X}{2\sigma^2}\right)dX,
\end{align*}
using triangular inequality on their coordinates:
\begin{align*}
|| \int_{\mathcal{C}_{B_r}}X^{i_1}...X^{i_k}\exp \left(-\frac{X^TA^{-1} X}{2\sigma^2}\right)dX|| \leq \int_{\mathcal{C}_{B_r}}||X||^k\exp \left(-\frac{X^TA^{-1} X}{2\sigma^2}\right)dX,
\end{align*}
and performing the change of variables $X' = A^{-1/2}X$, i.e. taking the matrix square root of the positive definite matrix $A^{-1}$:
\begin{align*}
|| \int_{\mathcal{C}_{B_r}}X^{i_1}...X^{i_k}\exp \left(-\frac{X^TA^{-1} X}{2\sigma^2}\right)dX|| &\leq \int_{C_{B_{|||A^{1/2}|||r}}}|||A|||^{k/2}||X'||^k\exp \left(-\frac{X'^T X'}{2\sigma^2}\right)\det\left(A^{1/2}\right)dX,\\
& =  |||A|||^{k/2}\det\left(A^{1/2}\right)\int_{\mathcal{C}_{B_{|||A^{1/2}|||r}}}||X'||^k\exp \left(-\frac{||X'||^2}{2\sigma^2}\right)dX.
\end{align*}
By the computations in Subsection~\ref{sec:firstcomput}, we have:
\begin{align*}
|| \int_{\mathcal{C}_{B_r}}X^{i_1}...X^{i_k}\exp \left(-\frac{X^TA^{-1} X}{2\sigma^2}\right)dX|| = \epsilon(\sigma)
\end{align*}
Therefore, the truncated moments are equivalent to the (full) moments for $\sigma \rightarrow 0$:
\begin{equation}
\boxed{\mathcal{M}_r^{i_1...i_k}(\sigma^2 A) = \mathcal{M}^{i_1...i_k}(\sigma^2 A) + \epsilon(\sigma)}
\end{equation}

\subsection{Isotropic Gaussian Moments on a Riemannian manifolds \label{sec:riem}}

We turn to computations of Gaussian moments in a $m$-dimensinoal Riemannian manifold $M$, using the notation $\mathfrak{M}$. We refer to the \textit{(unnormalized) moment of order $k$ of the $m$-dimensional isotropic Gaussian of covariance $\sigma^2 \mathbb{I}$} as:
\begin{align*}
{\mathfrak{M}}^{i_1...i_k}\left(\sigma^2 \mathbb{I}\right) = \int_{M} \overrightarrow{YX}^{i_1}...\overrightarrow{YX}^{i_k} \exp \left(-\frac{d_M^2(X,Y)}{2\sigma^2}\right)dM(X)
\end{align*}
and to the \textit{truncated (unnormalized) moment at radius $r$ of the $m$-dimensional isotropic Gaussian of covariance $\sigma^2 \mathbb{I}$} as:
\begin{align*}
{\mathfrak{M}}^{i_1...i_k}_r\left(\sigma^2 \mathbb{I}\right)& = \int_{B_r}\overrightarrow{YX}^{i_1}...\overrightarrow{YX}^{i_k}\exp \left(-\frac{d_M^2(X,Y)}{2\sigma^2}\right)d\overrightarrow{YX}
\end{align*}
where the integration domain is now the $m$-dimensional geodesic ball $B_r$ of radius $r$ and centered at $Y$.

\subsubsection{First, we consider the truncated moments:}
\begin{align*}
{\mathfrak{M}}^{i_1...i_k}_r\left(\sigma^2 \mathbb{I}\right) = \int_{B_r} \overrightarrow{YX}^{i_1}...\overrightarrow{YX}^{i_k} \exp \left(-\frac{d_M^2(X,Y)}{2\sigma^2}\right)dM(X)
\end{align*}
In a NCS at $Y$ : $d_M^2(X,Y)=\overrightarrow{YX}^T \overrightarrow{YX}$:
\begin{align*}
{\mathfrak{M}}^{i_1...i_k}_r\left(\sigma^2 \mathbb{I}\right) = \int_{B_r} \overrightarrow{YX}^{i_1}...\overrightarrow{YX}^{i_k} \exp \left(-\frac{\overrightarrow{YX}^T \overrightarrow{YX}}{2\sigma^2}\right)dM(X)
\end{align*}

On the small ball $B_r$, $dM(X) = d\overrightarrow{YX} + \frac{1}{6}R_{ab}(Y)\overrightarrow{YX}^a\overrightarrow{YX}^bd\overrightarrow{YX}+\mathcal{O}(||\overrightarrow{YX}||^3)d\overrightarrow{YX}$:
\begin{align*}
{\mathfrak{M}}^{i_1...i_k}_r\left(\sigma^2 \mathbb{I}\right) = \int_{B_r} \overrightarrow{YX}^{i_1}...\overrightarrow{YX}^{i_k} &\exp \left(-\frac{\overrightarrow{YX}^T \overrightarrow{YX}}{2\sigma^2}\right)d\overrightarrow{YX} \\
&+ \frac{1}{6} \sum_{ab} R_{ab}(Y) \int_{B_r} \overrightarrow{YX}^a\overrightarrow{YX}^b.\overrightarrow{YX}^{i_1}...\overrightarrow{YX}^{i_k} \exp \left(-\frac{\overrightarrow{YX}^T \overrightarrow{YX}}{2\sigma^2}\right)d\overrightarrow{YX}\\
& + \int_{B_r} O(||\overrightarrow{YX}||^3) \overrightarrow{YX}^{i_1}...\overrightarrow{YX}^{i_k}\exp \left(-\frac{\overrightarrow{YX}^T \overrightarrow{YX}}{2\sigma^2}\right)d\overrightarrow{YX}
\end{align*}

Now we recognize the (un-normalized) moments of a truncated isotropic Gaussian a $B_r$ in the vector space $T_Y M \simeq \mathbb{R}^m$. We replace them by the expressions given in the previous subsection:
\begin{align*}
{\mathfrak{M}}^{i_1...i_k}_r\left(\sigma^2 \mathbb{I}\right) & =\left(\mathcal{M}^{i_1...i_k}(\sigma^2\mathbb{I})+\epsilon(\sigma)\right) + \frac{1}{6} \sum_{ab} R_{ab}(Y) \left(\mathcal{M}^{abi_1...i_k}(\sigma^2\mathbb{I})+\epsilon(\sigma)\right)  + \mathcal{O}(\sigma^{m+k+4})
\end{align*}
The sum of two functions that decrease exponentially for $\sigma \rightarrow 0$ is a function that decreases exponentially for $\sigma \rightarrow 0$. So that:
\begin{align*}
{\mathfrak{M}}^{i_1...i_k}_r\left(\sigma^2 \mathbb{I}\right) & = \mathcal{M}^{i_1...i_k}(\sigma^2\mathbb{I})+ \frac{1}{6} \sum_{ab} R_{ab}(Y)\mathcal{M}^{abi_1...i_k}(\sigma^2\mathbb{I}) + \mathcal{O}(\sigma^{m+k+4}) + \epsilon(\sigma)
\end{align*}
We take the $\sigma$'s out:
\begin{equation}
\boxed{{\mathfrak{M}}^{i_1...i_k}_r\left(\sigma^2 \mathbb{I}\right) = \sigma^{k+m}\mathcal{M}^{i_1...i_k}(\mathbb{I})+ \frac{\sigma^{k+m+2}}{6} \sum_{ab} R_{ab}(Y)\mathcal{M}^{abi_1...i_k}(\mathbb{I}) + \mathcal{O}(\sigma^{m+k+4}) + \epsilon(\sigma)}
\end{equation}

\subsubsection{Second, we consider the (full) moments.} They write, with respect to the non-truncated moments:
\begin{align*}
{\mathfrak{M}}^{i_1...i_k}\left(\sigma^2 \mathbb{I}\right) = {\mathfrak{M}}^{i_1...i_k}_r\left(\sigma^2 \mathbb{I}\right) + \int_{\mathcal{C}_{B_r}} X^{i_1}...X^{i_k} \exp \left(-\frac{d_M^2(Y,X)}{2\sigma^2} \right)  dM(X)
\end{align*}
The second term of the sum is negligible when $\sigma \rightarrow 0$. To see this, we put an upper bound on its norm through triangular inequality:
\begin{align*}
||\int_{\mathcal{C}_{B_r}} X^{i_1}...X^{i_k} \exp \left(-\frac{d_M^2(Y,X)}{2\sigma^2} \right)  dM(X)||
& \leq \int_{\mathcal{C}_{B_r}} ||X^{i_1}||...||X^{i_k}|| \exp \left(-\frac{d_M^2(Y,X)}{2\sigma^2} \right)dM(X)
\end{align*}
then using triangular inequality on the coordinates:
\begin{align*}
||\int_{\mathcal{C}_{B_r}} X^{i_1}...X^{i_k} \exp \left(-\frac{d_M^2(Y,X)}{2\sigma^2} \right)  dM(X)|| & \leq \int_{\mathcal{C}_{B_r}} d_M(Y,X)^k \exp \left(-\frac{d_M^2(Y,X)}{2\sigma^2} \right)  dM(X).
\end{align*}
We assume that the Ricci curvature of $M$ is bounded from below. Therefore, the measure $dM$ has an upper bound with respect to the Lebesgue measure on the tangent space, which we write $K$: $dM(X) \leq K dX$:
\begin{align*}
||\int_{\mathcal{C}_{B_r}} X^{i_1}...X^{i_k} \exp \left(-\frac{d_M^2(Y,X)}{2\sigma^2} \right)  dM(X) || & \leq K \int_{\mathcal{C}_{B_r}} d_M(Y,X)^k \exp \left(-\frac{d_M^2(Y,X)}{2\sigma^2} \right)  dX
\end{align*}
By the computations in Subsection~\ref{sec:firstcomput}, this inequality together with $\sigma \rightarrow 0$ shows that:
\begin{equation}
\int_{\mathcal{C}_{B_r}} X^{i_1}...X^{i_k} \exp \left(-\frac{d_M^2(Y,X)}{2\sigma^2} \right)  dM(X) = \epsilon(\sigma)
\end{equation}
i.e. is a function that decreases exponentially when $\sigma \rightarrow 0$.

Therefore, the (unnormalized) moment of order $k$ of the $m$-dimensional isotropic Gaussian of covariance $\sigma^2 \mathbb{I}$ writes:
\begin{equation}
\boxed{{\mathfrak{M}}^{i_1...i_k}\left(\sigma^2 \mathbb{I}\right)= \sigma^{k+m}\mathcal{M}^{i_1...i_k}(\mathbb{I})+ \frac{\sigma^{k+m+2}}{6} \sum_{ab} R_{ab}(Y)\mathcal{M}^{abi_1...i_k}(\mathbb{I}) + \mathcal{O}(\sigma^{m+k+4})+ \epsilon(\sigma).}
\end{equation}

\section{Proof of Theorem~\ref{th:pdf}: Induced probability density on shapes}

In this section we prove the Theorem 1 of our paper. We recall Theorem 1 below.

\begin{theorem}\label{th:pdf}
\textcolor{red}{The data $X_i$'s are generated in the finite-dimensional Riemannian manifold $M$ following the model: $X_i = \text{Exp}(g_i \cdot Y,\epsilon_i), i=1...n$, described in the paper. In this model: (i) the action of the finite dimensional Lie group $G$ on $M$, denoted $\cdot$, is isometric, (ii) the parameter $Y$ is the template shape in the shape space $Q$, (iii) $\epsilon_i$ is the noise and follows a (generalization to manifolds of a) Gaussian of variance $\sigma^2$, see Section 1 of the paper.}

\textcolor{red}{Then, the probability distribution function $f$ on the shapes of the $X_i$'s, $i=1...n$, in the asymptotic regime on an infinite number of data $n \rightarrow + \infty$, has the following Taylor expansion around the noise level $\sigma = 0$}:
\begin{align*}\label{eq:f}
f(Z) &  = \frac{1}{(\sqrt{2\pi}\sigma)^q}\exp \left(-\frac{d_M^2(Y,Z)}{2\sigma^2} \right)\left( F_0(Z) +\sigma^2 F_2(Z)+\mathcal{O}(\sigma^{4})+  \epsilon(\sigma)\right)
\end{align*}
where (i) $Z$ denotes a point in the shape space $Q$, (ii) $F_0$ and $F_2$ are functions of $Z$ involving the derivatives of the Riemannian tensor at $Z$ and the derivatives of the graph $G$ describing the orbit $O_Z$ at $Z$, and (iii) $\epsilon$ is a function of $\sigma$ that decreases exponentially for $\sigma \rightarrow 0$.
\end{theorem}

We consider $Z \in B_r$ i.e. in the geodesic ball of center $Y$ and radius $Y$. We use a NCS at $Z$. The notations are summarized on Figure~\ref{fig:notations2}. The reader can refer to this Figure along the proof.

\begin{figure}[!htbp]
  \centering
     \def\svgwidth{0.6\textwidth}
\input{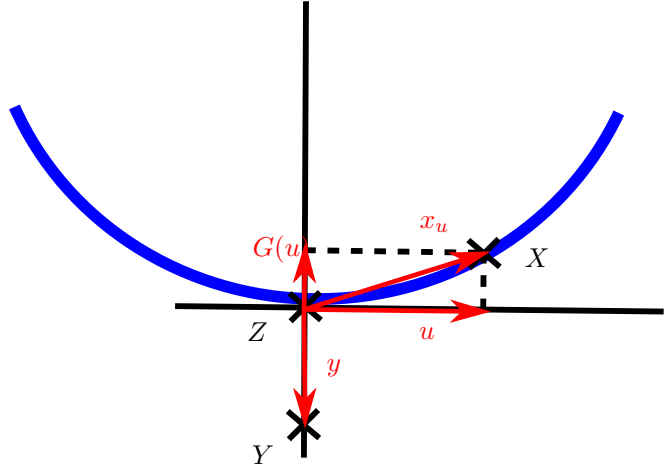}
  \caption{Summary of the notations used in the proof of Theorem 1. The vectors of the tangent space $T_ZM$ are in red.}
   \label{fig:notations2}
\end{figure} 

The generative model implies the following Riemannian Gaussian distribution on the objects:
\begin{equation}
f(Z)=\frac{1}{C_M(\sigma)} \exp\left(-\frac{d_M^2(X,Y)}{2\sigma^2}\right),
\end{equation}
where $C_M(\sigma)$ is the integration constant:
\begin{equation}
C_M(\sigma) =\int_{M} \exp \left(-\frac{d_M^2(X, Y)}{2\sigma^2} \right)dM(X)
\end{equation}

We compute the induced probability distribution $f$ on shapes by integrating the distribution on the orbit of $X$ out of $f(Z)$:
\begin{align*}
f(Z) & = \frac{1}{C_M(\sigma)}\int_{O_Z} \exp \left(-\frac{d_M^2(Y,X)}{2\sigma^2}\right)dO_Z(X).
\end{align*}

We start by dividing the integral:
\begin{align*}
f(Z) & = \frac{1}{C_M(\sigma)}. \left(\int_{O_Z \cap B_r} \exp \left(-\frac{d_M^2(Y,X)}{2\sigma^2}\right) dO_Z(X) + \int_{O_Z \cap \mathcal{C}_{B_r}} \exp \left(-\frac{d_M^2(Y,X)}{2\sigma^2}\right)dO_Z(X)\right) 
\end{align*}
In the parenthesis, we call the first integral $I_Z(\sigma)$ and the second term is an $\eta(\sigma)$:
\begin{align*}
f(Z) & = {C_M(\sigma)}^{-1}.\left( I_Z(\sigma)+ \eta(\sigma) \right)
\end{align*}

We compute the Taylor expansions of $C_M(\sigma)^{-1}$ and $I_Z(\sigma)$ for $\sigma \rightarrow 0$ and show that $\eta(\sigma) = \epsilon(\sigma)$.

\subsection{Taylor expansion of $C_M(\sigma)^{-1}$}

The integration constant $C_M(\sigma)$ is the moment of order $0$ of the $m$-dimensional isotropic Gaussian of variance $\sigma^2 \mathbb{I}$ in the $m$-dimensional manifold $M$. From Subsection~\ref{sec:riem}:
\begin{align*}
C_M(\sigma) & = (\sqrt{2\pi}\sigma)^m +  \frac{\sigma^{m+2}}{6} \sum_{ab} R_{ab}(Y)(\sqrt{2\pi})^m\delta_{ab}  + \mathcal{O}(\sigma^{m+3}) + \epsilon(\sigma)\\
& = (\sqrt{2\pi}\sigma)^m +  \frac{\sigma^{m+2}}{6} R(Y)(\sqrt{2\pi})^m + \mathcal{O}(\sigma^{m+3}) + \epsilon(\sigma)
\end{align*}
where $\epsilon(\sigma)$ is exponentially decreasing wrt $\sigma$ and $R(Y)$ is the scalar curvature of $M$ at $Y$. Its inverse:
\begin{align*}
{C_M(\sigma)}^{-1} &= \left((\sqrt{2\pi}\sigma)^m +  \frac{(\sqrt{2\pi})^m}{6}  R(Y) \sigma^{m+2} + \mathcal{O}(\sigma^{m+3}) + \epsilon(\sigma)\right)^{-1}\\
& = (\sqrt{2\pi}\sigma)^{-m}.\left(1+ \frac{\sigma^{2} }{6}R(Y)+ O(\sigma^{3}) + \epsilon(\sigma)\right)^{-1}\\
& = (\sqrt{2\pi}\sigma)^{-m}.\left( 1 - \frac{\sigma^{2} }{6}R(Y) + O(\sigma^{3}) + \epsilon(\sigma)\right)
\end{align*}

\subsection{Taylor expansion of $I_Z(\sigma)$}

Now we compute the integral $I_Z(\sigma)$ involved in the formula of $f(Z)$:
\begin{align*}
I_Z(\sigma) & = \int_{O_Z \cap B_r} \exp \left(-\frac{d_M^2(Y,X)}{2\sigma^2}\right)dO(X)
\end{align*}
We denote: $B_{r_Z}= B_{r_Z}$. We first perform the computations for any $\sigma$. We recall that $r$ is fixed, and small enough in a sense made precise later.

\subsubsection{Computing $d_M^2(Y,X)$} The first component of $I_Z(\sigma)$ is $d_M^2(Y,X)$. We express the Taylor expansion of $d_M^2$ for $X \in O_Z$ close to $Z$. 

First, we parameterize the point $X$. The points $X$, $Z=\pi(X)$ and the orbit $O_Z \subset M$ are illustrated on Figure~\ref{fig:notations}: the orbit $O_Z = O_X$ is the blue circle in $M = \mathbb{R}^2$ going through $X$ and $Z$. The orbit $O_Z$ can be seen in the tangent space $T_{Z}M$ through the Logarithm map at $Z$. For $r$ small enough, i.e. for $X$ close enough to $Z$, we can locally represent $O_Z$ in $T_{Z} M$ as the graph of a smooth function $G$ from $T_{Z} O$ to $N_{Z}O$, using the vector $u \in T_{Z}O_Z$ around $u=0$:
\begin{align*}
I: T_{Z}O_Z & \mapsto T_ZM=T_{Z}O_Z \oplus N_{Z} O_Z\\
u & \mapsto x_u = (u, G(u))
\end{align*}
The local graph $u \rightarrow G(u)$ is illustrated on Figure~\ref{fig:notations} and has the following Taylor expansion around $u = 0$ in the NCS at $Z$: 
\begin{align*}
G(u)^a = \frac{1}{2}h^a_{bc}(Z)u^bu^c + G_3(Z)^a_{bcd}u^bu^cu^d + G_4(Z)^a_{bcde}u^bu^cu^d u^e+ O(||u||^5)
\end{align*}
The 0-th and 1-th order derivatives of $u \rightarrow G(u)$ are zero because the graph goes through $Z$ and is tangent at $T_{Z}O_Z$. The second order derivative is by definition the second fundamental form $h(Z)$ introduced in Section 2 of the paper: $h(Z)$ represents the best quadratic approximation of the graph $G$. The third and fourth orders $G_3(Z)$ and $G_4(Z)$ are further refinements on the shape of the graph $G$ around $Z$.

Second, we compute the Taylor expansion of $d_M^2(Y,X)$ with respect to $u$:
\begin{equation*}
\boxed{d_M^2(Y,X) = d_M^2(Y,Z) + u^aL_a(Z)+ u^au^bM_{ab}(Z)+u^au^bu^cP_{abc}(Z)+u^au^bu^cu^dS_{abcd}(Z)+ \mathcal{O}(||u||^5)}
\end{equation*}
and our goal is to compute the different tensors.

\noindent $Y$ and $X$ are represented by their Riemannian Logarithms at $Z$: $y=\text{Log}_ZY$, and $x_u=\text{Log}_ZX$. We also recall: $u^Ty=0$. Using these, we express the squared distance $d_M^2(Y,X)$ in the NCS at $Z$. We use the formula p.23 in \cite{Brewin2009} with the notations $\Delta x \leftarrow y-x_u$ and $x \leftarrow x_u$:
\begin{align*}
d_M^2 & = d_M^2(\text{Exp}_Z(x(u)),\text{Exp}_Zy)\\
& = \delta_{ab}(y-x_u)^a(y-x_u)^b\\
& -\frac{60}{180}R_{cadb}x_u^cx_u^d(y-x_u)^a(y-x_u)^b \\
&-\frac{15}{180}x_u^dx_u^e \nabla_aR_{dbec}(y-x_u)^a(y-x_u)^b(y-x_u)^c \\
&- \frac{3}{540}x_u^ex_u^f(44R_{gaeb}R_{gcfd}+3\nabla_{ab}R_{ecfd})(y-x_u)^a(y-x_u)^b(y-x_u)^c(y-x_u)^d\\
&+\frac{1}{54}x_u^fx_u^gR_{hafb}\nabla_cR_{hdge}(y-x_u)^a(y-x_u)^b(y-x_u)^c(y-x_u)^d(y-x_u)^e\\
& -\frac{30}{180}x_u^cx_u^dx_u^e\nabla_cR_{daeb}\\
&+\frac{1}{180}x_u^dx_u^ex_u^f(8R_{gdea}R_{gbfc}-9\nabla_{da}R_{ebfc})(y-x_u)^a(y-x_u)^b(y-x_u)^c\\
&-\frac{5}{540}x_u^ex_u^fx_u^g(8R_{haeb}\nabla_{c}R_{hfgd}+9R_{haeb}\nabla_{h}R_{fcgd}\\
&\qquad\qquad\qquad+20R_{haeb}\nabla_fR_{hcgd}-6R_{hefa}\nabla_{b}R_{hcgd})(y-x_u)^a(y-x_u)^b(y-x_u)^c(y-x_u)^d\\
& +\frac{1}{180}x_u^cx_u^dx_u^ex_u^f(8R_{gcda}R_{gefb}-9\nabla_{cd}R_{eafb})(y-x_u)^a(y-x_u)^b\\
&+\frac{1}{180}x_u^dx_u^ex_u^fx_u^g(4R_{hadb}\nabla_eR_{hfgc}+4R_{hdea}\nabla_bR_{hfgc}+4R_{hdea}\nabla_fR_{hbgc}\\
& \qquad\qquad\qquad+3\nabla_{dea}R_{fbgc})(y-x_u)^a(y-x_u)^b(y-x_u)^c\\
& +\mathcal{O}(||x_u||^5)\\
&\\
&\\
&\\
&\\
&\\
\end{align*}

\noindent We express this in orders of $x_u$:
\begin{align*}
d_M^2 & = \delta_{ab}y^ay^b\\
& -2 \delta_{ab}y^ax_u^b\\
& -\frac{1}{3}R_{cadb}x_u^cx_u^dy^ay^d\\
& +\frac{1}{12}x_u^dx_u^e \nabla_aR_{dbec}y^ay^by^c\\
&- \frac{1}{180}x_u^ex_u^f(44R_{gaeb}R_{gcfd}+3\nabla_{ab}R_{ecfd})y^ay^by^cy^d \\
& - \frac{1}{54}x_u^fx_u^g R_{hafb}\nabla_cR_{hdge}y^ay^by^cy^dy^e\\
& +\frac{2}{3}x_u^cx_u^dR_{cadb}x_u^ay^b\\
& - \frac{1}{12}x_u^dx_u^e\nabla_aR_{dbec}(x_u^ay^by^c +2y^ay^bx_u^c)\\
& - \frac{1}{180}x_u^ex_u^f(44R_{gaeb}R_{gcfd}+3\nabla_{ab}R_{ecfd})(2x_u^ay^by^cy^d+2y^ay^bx_u^cy^d)\\
& +\frac{1}{54}x_u^fx_u^gR_{hafb}\nabla_cR_{hdge}(4x_u^ay^by^cy^dy^e+y^ay^bx_u^cy^dy^e)\\
&-\frac{1}{3}x_u^cx_u^dx_u^e\nabla_c R_{daeb}y^ay^b\\
&  \frac{1}{180}x_u^dx_u^ex_u^f(8R_{gdea}R_{gbfc}-9\nabla_{da}R_{ebfc})y^ay^by^c\\
& -\frac{1}{180}x_u^ex_u^fx_u^g(8R_{haeb}\nabla_{c}R_{hfgd}+9R_{haeb}\nabla_{h}R_{fcgd}+20R_{haeb}\nabla_fR_{hcgd}-6R_{hefa}\nabla_{b}R_{hcgd})y^ay^by^cy^d\\
& -\frac{1}{3}R_{cadb}x_u^cx_u^dx_u^ax_u^b\\
& - \frac{1}{12}x_u^dx_u^e\nabla_aR_{dbec}(2x_u^ax_u^by^c+y^ax_u^bx_u^c)\\
& -\frac{1}{180}x_u^ex_u^f(44R_{gaeb}R_{gcfd}+3\nabla_{ab}R_{ecfd})(2x_u^ax_u^by^cy^d+2y^ax_u^bx_u^cy^d+2y^ay^bx_u^cx_u^d+4x_u^ay^bx_u^cy^d)\\
& + \frac{1}{54}x_u^fx_u^gR_{hafb}\nabla_cR_{hdge}(4x_u^ax_u^by^cy^dy^e+4y^ax_u^bx_u^cy^dy^e)\\
& + \frac{1}{180}x_u^dx_u^fx_u^g(8R_{gdea}R_{gbfc}-9\nabla_{da}R_{ebfc})(x_u^ay^by^c+2y^ax_u^bx_u^c)\\
& - \frac{1}{180}x_u^ex_u^fx_u^g(8R_{haeb}\nabla_{c}R_{hfgd}+9R_{haeb}\nabla_{h}R_{fcgd}+20R_{haeb}\nabla_fR_{hcgd}-6R_{hefa}\nabla_{b}R_{hcgd})(x_u^ay^by^cy^d+y^ay^bx_u^cy^d)\\
& + \frac{1}{180}x_u^cx_u^dx_u^ex_u^f(8R_{gdea}R_{gbfc}-9\nabla_{da}R_{ebfc})y^ay^b\\
& + \frac{1}{180}(4R_{hadb}\nabla_eR_{hfgc}+4R_{hdea}\nabla_bR_{hfgc}+4R_{hdea}\nabla_fR_{hbgc}-3\nabla_{dea}R_{fbgc})y^ay^by^c\\
& + \mathcal{O}(||x_u||^5)\\
\end{align*}

\noindent We replace $x_u^a$ by $x_u^a = u^a + G(u)^a $ where the expression of $G(u)$ is a $\mathcal{O}(||u||^2)$, so that:
\begin{align*}
d_M^2 & = d_M^2(Y, Z) -2 \delta_{ab}y^au^b\\
& -2\delta_{ab}y^aG(u)^b -\frac{1}{3}R_{cadb}u^cu^dy^ay^d
 +\frac{1}{12}u^du^e \nabla_aR_{dbec}y^ay^by^c
- \frac{1}{180}u^eu^f(44R_{gaeb}R_{gcfd}+3\nabla_{ab}R_{ecfd})y^ay^by^cy^d \\
& - \frac{1}{54}u^fu^g R_{hafb}\nabla_cR_{hdge}y^ay^by^cy^dy^e \\
& -\frac{2}{3}R_{cadb}u^cG(u)^dy^ay^d
 +\frac{2}{12}u^dG(u)^e \nabla_aR_{dbec}y^ay^by^c
- \frac{2}{180}u^eG(u)^f(44R_{gaeb}R_{gcfd}+3\nabla_{ab}R_{ecfd})y^ay^by^cy^d \\
& - \frac{2}{54}u^fG(u)^g R_{hafb}\nabla_cR_{hdge}y^ay^by^cy^dy^e 
 +\frac{2}{3}u^cu^dR_{cadb}u^ay^b
 - \frac{1}{12}u^du^e\nabla_aR_{dbec}(u^ay^by^c +2y^ay^bu^c)\\
& - \frac{1}{180}u^eu^f(44R_{gaeb}R_{gcfd}+3\nabla_{ab}R_{ecfd})(2u^ay^by^cy^d+2y^ay^bu^cy^d)\\
& +\frac{1}{54}u^fu^gR_{hafb}\nabla_cR_{hdge}(4u^ay^by^cy^dy^e+y^ay^bu^cy^dy^e)\\
& -\frac{1}{3}u^cu^du^e\nabla_c R_{daeb}y^ay^b
+  \frac{1}{180}u^du^eu^f(8R_{gdea}R_{gbfc}-9\nabla_{da}R_{ebfc})y^ay^by^c\\
& -\frac{1}{180}u^eu^fu^g(8R_{haeb}\nabla_{c}R_{hfgd}+9R_{haeb}\nabla_{h}R_{fcgd}+20R_{haeb}\nabla_fR_{hcgd}-6R_{hefa}\nabla_{b}R_{hcgd})y^ay^by^cy^d\\
& -\frac{1}{3}R_{cadb}G(u)^cG(u)^dy^ay^d
 +\frac{1}{12}G(u)^dG(u)^e \nabla_aR_{dbec}y^ay^by^c\\
&- \frac{1}{180}G(u)^eG(u)^f(44R_{gaeb}R_{gcfd}+3\nabla_{ab}R_{ecfd})y^ay^by^cy^d \\
& - \frac{1}{54}G(u)^fG(u)^g R_{hafb}\nabla_cR_{hdge}y^ay^by^cy^dy^e
 +\frac{2}{3}u^cu^dR_{cadb}G(u)^ay^b\\
& - \frac{1}{12}u^du^e\nabla_aR_{dbec}(G(u)^ay^by^c +2y^ay^bG(u)^c)\\
& - \frac{1}{180}u^eu^f(44R_{gaeb}R_{gcfd}+3\nabla_{ab}R_{ecfd})(2G(u)^ay^by^cy^d+2y^ay^bG(u)^cy^d)\\
& +\frac{1}{54}u^fu^gR_{hafb}\nabla_cR_{hdge}(4G(u)^ay^by^cy^dy^e+y^ay^bG(u)^cy^dy^e)\\
& -\frac{1}{3}u^cu^dG(u)^e\nabla_c R_{daeb}y^ay^b
+  \frac{1}{180}u^du^eG(u)^f(8R_{gdea}R_{gbfc}-9\nabla_{da}R_{ebfc})y^ay^by^c\\
& -\frac{1}{180}u^eu^fG(u)^g(8R_{haeb}\nabla_{c}R_{hfgd}+9R_{haeb}\nabla_{h}R_{fcgd}+20R_{haeb}\nabla_fR_{hcgd}-6R_{hefa}\nabla_{b}R_{hcgd})y^ay^by^cy^d\\
& -\frac{1}{3}R_{cadb}u^cu^du^au^b\\
& - \frac{1}{12}u^du^e\nabla_aR_{dbec}(2u^au^by^c+y^au^bu^c)\\
& -\frac{1}{180}u^eu^f(44R_{gaeb}R_{gcfd}+3\nabla_{ab}R_{ecfd})(2u^au^by^cy^d+2y^au^bu^cy^d+2y^ay^bu^cu^d+4u^ay^bu^cy^d)\\
& + \frac{1}{54}u^fu^gR_{hafb}\nabla_cR_{hdge}(4u^au^by^cy^dy^e+4y^au^bu^cy^dy^e)\\
& + \frac{1}{180}u^du^fu^g(8R_{gdea}R_{gbfc}-9\nabla_{da}R_{ebfc})(u^ay^by^c+2y^au^bu^c)\\
& - \frac{1}{180}u^eu^fu^g(8R_{haeb}\nabla_{c}R_{hfgd}+9R_{haeb}\nabla_{h}R_{fcgd}+20R_{haeb}\nabla_fR_{hcgd}-6R_{hefa}\nabla_{b}R_{hcgd})(u^ay^by^cy^d+y^ay^bu^cy^d)\\
& + \frac{1}{180}u^cu^du^eu^f(8R_{gdea}R_{gbfc}-9\nabla_{da}R_{ebfc})y^ay^b\\
& + \frac{1}{180}(4R_{hadb}\nabla_eR_{hfgc}+4R_{hdea}\nabla_bR_{hfgc}+4R_{hdea}\nabla_fR_{hbgc}-3\nabla_{dea}R_{fbgc})y^ay^by^c\\
& + \mathcal{O}(||u||^5)
\end{align*}

\noindent We replace $G(u)$ by its Taylor expansion. Identifying the tensors gives:
\begin{equation}
\boxed{
\begin{aligned}
L_a(Z) & =0 \\
M_{ab}(Z) & = -y^ch_{ab}(Z)^c-\frac{1}{3}R_{acbd}(Z)y^cy^d
 +\frac{1}{12}\nabla_dR_{aebc}(Z)y^dy^ey^c \\
& - \frac{1}{180}(44R_{geaf}(Z)R_{gcbd}(Z)+3\nabla_{ef}R_{acbd}(Z))y^ey^fy^cy^d \\
& - \frac{1}{54}R_{hfag}(Z)\nabla_cR_{hdbe}(Z)y^fy^gy^cy^dy^e \\
\end{aligned}
}
\end{equation}
and $P_{abc}(Z)$ and $S_{abcd}(Z)$ are tensors mixing  the derivatives of the graph $G$ and the derivatives of the Riemannian curvature $R$ of $M$ at $Z$.

\subsubsection{Computing $dO_Z(X)$} The second component of $I_Z(\sigma)$ is the measure of the orbit $dO_Z(X)$. We seek the Taylor expansion of the measure:
\begin{equation}
\boxed{dO_Z(X) = \left(1+ T_c(Z) u^c - N_{cb}(Z) u^bu^c+\mathcal{O}(||u||^3)\right)du}
\end{equation}
and our goal is, again, to express the tensors $T_c(Z)$ and $N_{cb}(Z)$.

The measure $dO_Z(X)$ is the restriction of the measure $dM(X)$:
\begin{align*}
dO_Z(X) = dM(X)|_{T_XO_Z}
\end{align*}
and we know that for $X$ close enough to $Z$: $dM(X) = dx_u -\frac{1}{6}\text{Ric}_{ab}(Z)x_u^ax_u^bdx_u$. Thus:
\begin{align*}
dO_Z(X) & = \left( 1-\frac{1}{6}\text{Ric}_{ab}(Z)x_u^ax_u^b+\mathcal{O}(||x_u||^3)\right)dx_u
\end{align*}

\noindent We replace $x_u^a$ by $x_u^a = u^a + \frac{1}{2}h^a_{bc}(Z)u^bu^c + O(||u||^4)$
\begin{align*}
dO_Z(X) & = \left( 1-\frac{1}{6}\text{Ric}_{ab}(Z)u^au^b+\mathcal{O}(||u||^3)\right)dx_u
\end{align*}
\noindent We express $dx_u$ with respect to $du$:
\begin{align*}
dx_u &= \det\left(\frac{dx_u^a}{du^b}\right)du\\
& = \det\left(\frac{d\left(u^a+\frac{1}{2}h^a_{bc}(Z)u^bu^c + O(||u||^4)\right)}{du^b}\right)du\\
&= \det\left(\delta_b^a +\frac{1}{2}h^a_{cb}(Z)u^c + \frac{1}{2}h^a_{bc}(Z)u^c+\mathcal{O}(||u||^3)\right)du\\
& = \det\left(\delta_b^a +h^a_{cb}(Z)u^c +\mathcal{O}(||u||^3)\right)du
\end{align*}
\noindent Developing the determinant:
\begin{align*}
dx_u &= \left( 1+h^a_{ca}(Z)u^c+\frac{1}{2}\left( (h^a_{ca}(Z)u^c)^2- h^a_{cb}(Z)u^ch^b_{da}(Z)u^d\right)+\mathcal{O}(||u||^3)\right)du
\end{align*}

\noindent Plugging in $dO_Z(X)$:
\begin{align*}
dO_Z(X) & = \left(  1-\frac{1}{6}\text{Ric}_{ab}(Z)u^au^b+\mathcal{O}(||u||^3)\right)\left( 1+h^a_{ca}(Z)u^c \right)\\
& +\left(  1-\frac{1}{6}\text{Ric}_{ab}(Z)u^au^b+\mathcal{O}(||u||^3)\right)\left(\frac{1}{2}\left( (h^a_{ca}(Z)u^c)^2- h^a_{cb}(Z)u^ch^b_{da}(Z)u^d\right)\right)\\
 &+\left(  1-\frac{1}{6}\text{Ric}_{ab}(Z)u^au^b+\mathcal{O}(||u||^3)\right)\left(\mathcal{O}(||u||^3)\right)du
\end{align*}

\noindent We develop by keeping up to the quadratic terms only:
\begin{align*}
dO_Z(X) & = \left(1-\frac{1}{6}\text{Ric}_{ab}(Z)u^au^b +h^a_{ca}(Z)u^c+\frac{1}{2}\left( (h^a_{ca}(Z)u^c)^2- h^a_{cb}(Z)u^ch^b_{da}(Z)u^d\right)+\mathcal{O}(||u||^3)\right)du
\end{align*}

\noindent We reorganize the terms, relabeling the mute labels:
\begin{align*}
dO_Z(X) & = \left(1+h^a_{ca}(Z)u^c -\frac{1}{6}\text{Ric}_{cb}(Z)u^cu^b +\frac{1}{2}\left( h^a_{ca}(Z)u^c. h^a_{ca}(Z)u^b- h^a_{cd}(Z)u^ch^d_{ba}(Z)u^b\right)+\mathcal{O}(||u||^3)\right)du
\end{align*}

\noindent Now we can factorize the quadratic terms:
\begin{align*}
dO_Z(X) & = \left(1+h^a_{ca}(Z)u^c - u^bu^c\left(\frac{1}{6}\text{Ric}_{cb}(Z) +\frac{1}{2} h^a_{ca}(Z) h^a_{ca}(Z)-\frac{1}{2} h^a_{cd}(Z)h^d_{ba}(Z)\right)+\mathcal{O}(||u||^3)\right)du
\end{align*}

\noindent So that we find the expressions of the tensors:
\begin{equation}
\boxed{
\begin{aligned}
T_{c}(Z) & = h^a_{ca}(Z)\\
N_{cb}(Z) & = \frac{1}{6}\text{Ric}_{cb}(Z) +\frac{1}{2} h^a_{ca}(Z) h^a_{ca}(Z)-\frac{1}{2} h^a_{cd}(Z)h^d_{ba}(Z)
\end{aligned}}
\end{equation}

\subsubsection{Gathering to compute $I_Z(\sigma)$}

We plug the expressions of $d_M^2(Y,X)$ and $dO_Z(X)$, computed in the preivous subsections, in the expression of $I_Z(\sigma)$:
\begin{align*}
I_Z(\sigma) & = \int_{B_{r_Z}} \exp \left(-\frac{d_M^2(Y,X)}{2\sigma^2}\right)dO_Z(X)
\end{align*}

\noindent Plugging the squared distance $d_M^2(Y,X)$ first:
\begin{flalign*}
I_Z(\sigma) & =\int_{B_{r_Z}} \exp \left(-\frac{d_M^2(Y,Z) + u^au^bM_{ab}(Z)+u^au^bu^cP_{abc}(Z)+u^au^bu^cu^dS_{abcd}(Z)+\mathcal{O}(||u||^5)}{2\sigma^2}\right) dO_Z(X)&
\end{flalign*}

\noindent We split the exponential and extract the part of the exponential that does not depend on $u$:
\begin{flalign*}
I_Z(\sigma) & =\exp \left(-\frac{d_M^2(Y,Z)}{2\sigma^2} \right)\int_{B_{r_Z}} \exp \left(-\frac{u^au^bM_{ab}(Z)}{2\sigma^2}\right)\exp\left(-\frac{u^au^bu^cP_{abc}(Z)+ u^au^bu^cu^dS_{abcd}(Z)+\mathcal{O}(||u||^5)}{2\sigma^2}\right) dO_Z(X)&
\end{flalign*}

\noindent We perform the Taylor expansion of the term with $\mathcal{O}(||u||^4)$, recalling that at this point, $\sigma$ can still be anything.
\begin{flalign*}
I_Z(\sigma) & =\exp \left(-\frac{d_M^2(Y,Z)}{2\sigma^2} \right)\int_{B_{r_Z}}  \exp \left(-\frac{u^au^bM_{ab}(Z)}{2\sigma^2}\right)\left(1-\frac{u^au^bu^cP_{abc}(Z)+u^au^bu^cu^dS_{abcd}(Z)+\mathcal{O}(||u||^5)}{2\sigma^2} \right) dO_Z(X)&
\end{flalign*}

\noindent Now we plug the $dO_Z(X)$:
\begin{flalign*}
I_Z(\sigma) & =\exp \left(-\frac{d_M^2(Y,Z)}{2\sigma^2} \right)\int_{B_{r_Z}}  \exp \left(-\frac{u^au^bM_{ab}(Z)}{2\sigma^2}\right)\left(1-\frac{u^au^bu^cP_{abc}(Z)+ u^au^bu^cu^dS_{abcd}(Z)}{2\sigma^2}+\frac{\mathcal{O}(||u||^5)}{2\sigma^2}\right).&\\
&\qquad\qquad\qquad\qquad\qquad\qquad\qquad\qquad.\left(1+T_cu^c(Z) - N_{cb}(Z) u^bu^c+\mathcal{O}(||u||^3)\right)du
\end{flalign*}

\noindent We develop the product of the parenthesis on the right:
\begin{flalign*}
I_Z(\sigma) & =\exp \left(-\frac{d_M^2(Y,Z)}{2\sigma^2} \right)\int_{B_{r_Z}}  \exp \left(-\frac{u^au^bM_{ab}(Z)}{2\sigma^2}\right)\left(1 +T_cu^c(Z) - N_{cb}(Z) u^bu^c + \frac{u^au^bu^cP_{abc}(Z)}{2\sigma^2}\right)du\\
&+ \exp \left(-\frac{d_M^2(Y,Z)}{2\sigma^2} \right)\int_{B_{r_Z}}  \exp \left(-\frac{u^au^bM_{ab}(Z)}{2\sigma^2}\right)\left(\mathcal{O}(||u||^4)-\frac{u^au^bu^cu^dS_{abcd}(Z)}{2\sigma^2}+\frac{\mathcal{O}(||u||^5)}{2\sigma^2}\right)du &
\end{flalign*}

\noindent By skew symmetry, the terms in $T_cu^c(Z)$ and $u^au^bu^cP_{abc}(Z)$ integrate to $0$. Moreover, $\mathcal{O}(||u||^3)$, $\mathcal{O}(||u||^5)$ become $\mathcal{O}(||u||^4)$, $\mathcal{O}(||u||^6)$:
\begin{flalign*}
I_Z(\sigma) & =\exp \left(-\frac{d_M^2(Y,Z)}{2\sigma^2} \right)\int_{B_{r_Z}} \exp \left(-\frac{u^au^bM_{ab}(Z)}{2\sigma^2}\right)&\\
& \qquad\qquad\qquad\qquad\left(1- N_{cb}(Z) u^bu^c+\mathcal{O}(||u||^4)-\frac{u^au^bu^cu^dS_{abcd}(Z)}{2\sigma^2}+\frac{\mathcal{O}(||u||^6)}{2\sigma^2}\right)du
\end{flalign*}

\noindent We recognize the unnormalized truncated Gaussian moments of in $\mathbb{R}^p$, where $p$ is the dimension of the orbit $O_Z$, see Subsection~\ref{sec:eucl}:
\begin{flalign*}
I_Z(\sigma) & =\exp \left(-\frac{d_M^2(Y,Z)}{2\sigma^2} \right)\\
&.\left(\mathcal{M}_{r_Z}(\sigma^2 M^{-1})- N_{cb}(Z) \mathcal{M}_{r_Z}^{bc}(\sigma^2 M^{-1}) +\mathcal{O}(\sigma^{p+4})- \frac{S_{abcd}(Z)}{2\sigma^2}\mathcal{M}_{r_Z}^{abcd}(\sigma^2 M^{-1})+\frac{\mathcal{O}(\sigma^{p+6})}{2\sigma^2}\right)&
\end{flalign*}

\noindent We express them in terms of the unnormalized Gaussian moments in $\mathbb{R}^p$, see Subsection~\ref{sec:eucl}:
\begin{flalign*}
I_Z(\sigma) & =\exp \left(-\frac{d_M^2(Y,Z)}{2\sigma^2} \right). &\\
& \quad\quad\quad\quad.\left( \sigma^p \mathcal{M}(M^{-1}) + N_{cb}(Z)\sigma^{p+2} \mathcal{M}(M^{-1})^{bc} - \frac{S_{abcd}}{2\sigma^2}\sigma^{p+4} \mathcal{M}(M^{-1})^{abcd}+\mathcal{O}(\sigma^{p+4})+ \frac{\mathcal{O}(\sigma^{p+6})}{\sigma^2}+\epsilon(\sigma)\right) 
\end{flalign*}

\noindent We simplify the $\sigma$'s:
\begin{flalign*}
I_Z(\sigma) & = \exp \left(-\frac{d_M^2(Y,Z)}{2\sigma^2} \right).&\\
& \quad. \left( \sigma^p \mathcal{M}(M^{-1})^0 + \sigma^{p+2}\left(N_{cb}\mathcal{M}(M^{-1})^{bc}-\frac{S_{abcd}(Z)}{2}\mathcal{M}(M^{-1})^{abcd}\right)+\mathcal{O}(\sigma^{p+4}) +\epsilon(\sigma)\right)
\end{flalign*}

For convenience in the later computation, we define the notations:
\begin{flalign*}
I_Z(\sigma) & =\exp \left(-\frac{d_M^2(Y,Z)}{2\sigma^2} \right) \left(\sigma^pm_0(Z)+\sigma^{p+2}m_2(Z) +\mathcal{O}(\sigma^4) + \epsilon(\sigma) \right)&
\end{flalign*}
where:
\begin{equation}
\boxed{
\begin{aligned}
m_0(Z) & = \mathcal{M}(M)^0 = \sqrt{(2\pi)^p} \sqrt{\det\left((M_{ab}(Z)^{-1}\right)}\\
m_2(Z) &= N_{cb}(Z)\mathcal{M}(M^{-1}(Z))^{bc}-\frac{1}{2}S_{abcd}(Z)\mathcal{M}(M^{-1}(Z))^{abcd}
\end{aligned}
}
\end{equation}
where $M_{ab}$ and $S_{abcd}$ are given in the previous subsections.

\subsection{Upper bound on $\eta(\sigma)$}

We proceed with an adaptation of the method in Subsection~\ref{sec:firstcomput}:
\begin{align*}
\eta(\sigma) & = \int_{O_Z \cap \mathcal{C}_{B_r}}  \exp \left(-\frac{d_M^2(Y,X)}{2\sigma^2}\right) dO(X)
\end{align*}
Assuming that the Ricci curvature of the orbit is bounded by below, the measure of the orbit is bounded by above, by a constant that we write $K_O$:
\begin{align*}
||\eta(\sigma)||& \leq K_O \int_{O_Z \cap \mathcal{C}_{B_r}}  \exp \left(-\frac{d_M^2(Y,X)}{2\sigma^2}\right) du
\end{align*}

We integrate on the orbit by filiating it with hyperspheres of radii $\rho$:
\begin{align*}
||\eta(\sigma)||& \leq K_O \int_{r}^{+\infty} \int_{S_\rho \cap O_Z} \exp \left(-\frac{\rho^2}{2\sigma^2}\right) d(S_\rho \cap O_Z)d\rho\\
& = K_O \int_{r}^{+\infty} \exp \left(-\frac{\rho^2}{2\sigma^2}\right) \text{Vol}(S_\rho \cap O_Z) d\rho
\end{align*}

Now the volume of $S_\rho \cap O_Z$ is polynomial in $\rho$. We denote $P$ this polynom:
\begin{align*}
||\eta(\sigma)||& \leq  K_O \int_{r}^{+\infty} \exp \left(-\frac{\rho^2}{2\sigma^2}\right) P(\rho) d\rho
\end{align*}

By dominated convergence theorem, the right-hand-side is dominated by $\exp(-\frac{r^2}{2\sigma^2})$. Thus:
\begin{equation}
\eta(\sigma) = \epsilon(\sigma)
\end{equation}
i.e. $\sigma \rightarrow \eta(\sigma)$ is exponentially decreasing when $\sigma \rightarrow 0$.

\subsection{Final result: Taylor expansion of $f(Z)$}

Replacing the terms:
\begin{flalign*}
f(Z) &  = \frac{\left( 1 - \frac{\sigma^{2} }{6}R(Y) + O(\sigma^{3}) + \epsilon(\sigma)\right)}{(\sqrt{2\pi}\sigma)^m}\left(\exp \left(-\frac{d_M^2(Y,Z)}{2\sigma^2} \right) \left(\sigma^pm_0(Z)+\sigma^{p+2}m_2(Z) +\mathcal{O}(\sigma^{p+4}) + \epsilon(\sigma) \right)+  \epsilon(\sigma)\right)&
\end{flalign*}

\noindent We put the $\epsilon$ inside the main parenthesis:
\begin{flalign*}
f(Z) &  = \frac{\left( 1 - \frac{\sigma^{2} }{6}R(Y) + O(\sigma^{3}) + \epsilon(\sigma)\right)}{(\sqrt{2\pi}\sigma)^m}.&\\
& \quad. \exp \left(-\frac{d_M^2(Y,Z)}{2\sigma^2} \right) \left( \sigma^pm_0(Z)+\sigma^{p+2}m_2(Z) +\mathcal{O}(\sigma^{p+4}) +\epsilon(\sigma)+ \exp \left(+\frac{d_M^2(Y,Z)}{2\sigma^2}\right) \epsilon(\sigma) \right)
\end{flalign*}

\noindent We recall that $Z\in B_r$ so that:
\begin{flalign*}
f(Z) &  = \frac{\left( 1 - \frac{\sigma^{2} }{6}R(Y) + O(\sigma^{3})\right)}{(\sqrt{2\pi}\sigma)^m}.\exp \left(-\frac{d_M^2(Y,Z)}{2\sigma^2} \right)\left( \sigma^pm_0(Z)+\sigma^{p+2}m_2(Z) +\mathcal{O}(\sigma^{p+4}) +\epsilon(\sigma)\right)&
\end{flalign*}

We define:
\begin{align*}
f_Q(Z)= \frac{\exp \left(-\frac{d_M^2(Y,Z)}{2\sigma^2} \right)}{(\sqrt{2\pi}\sigma)^q}
\end{align*}

\noindent and put it in the front, remembering that $m= p+q$:
\begin{flalign*}
f(Z) &  = f_Q(Z) \frac{\left( 1 - \frac{\sigma^{2} }{6}R(Y) + O(\sigma^{3})\right)}{(\sqrt{2\pi}\sigma)^p}.\left( \sigma^pm_0(Z)+\sigma^{p+2}m_2(Z) +\mathcal{O}(\sigma^{p+4}) +\epsilon(\sigma)\right) &
\end{flalign*}

\noindent We divide by $\sigma^p$:
\begin{flalign*}
f(Z) &  =  f_Q(Z) \left(1 - \frac{\sigma^{2} }{6}R(Y) + O(\sigma^{3})\right).\left(\frac{m_0(Z)}{(\sqrt{2\pi})^p}+\sigma^{2} \frac{m_2(Z)}{(\sqrt{2\pi})^p}+\mathcal{O}(\sigma^4) + \epsilon(\sigma)\right)&
\end{flalign*}

\noindent We develop everything except the Gaussian in front:
\begin{flalign*}
f(Z) &  = f_Q(Z)\left(\frac{m_0(Z)}{(\sqrt{2\pi})^p} - \frac{\sigma^{2} }{6}\frac{m_0(Z)}{(\sqrt{2\pi})^p}R(Z) +\sigma^{2} \frac{m_2(Z)}{(\sqrt{2\pi})^p}+\mathcal{O}(\sigma^4)+  \epsilon(\sigma)\right)&
\end{flalign*}

\noindent We write this:
\begin{flalign*}
f(Z) &  = f_Q(Z)\left( F_0(Z) +\sigma^2 F_2(Z)+\mathcal{O}(\sigma^{4})+  \epsilon(\sigma)\right)&
\end{flalign*}
where:
\begin{align*}
F_0(Z)& = \frac{m_0(Z)}{(\sqrt{2\pi})^p} \\
& = \frac{\sqrt{(2\pi)^p} \sqrt{\det\left((M_{ab}(Z)^{-1}\right)}}{(\sqrt{2\pi})^p} \\
& =\sqrt{\det\left((M_{ab}(Z)^{-1}\right)}
\end{align*}
And:
\begin{align*}
F_2(Z)& = - \frac{1}{6}\frac{m_0(Z)}{(\sqrt{2\pi})^p}R(Z) + \frac{m_2(Z)}{(\sqrt{2\pi})^p}\\
& = - \frac{1}{6} \sqrt{\det\left((M_{ab}(Z)^{-1}\right)}R(Z)  + \frac{m_2(Z)}{(\sqrt{2\pi})^p}
\end{align*}
So that:
\begin{equation}
\boxed{
\begin{aligned}
F_0(Z) & =\sqrt{\det\left((M_{ab}(Z)^{-1}\right)}\\
F_2(Z) & = - \frac{1}{6} \sqrt{\det\left((M_{ab}(Z)^{-1}\right)}R(Z)  + \frac{m_2(Z)}{(\sqrt{2\pi})^p}
\end{aligned}
}
\end{equation}
and we refer to the previous subsections for the formula of $M_{ab}(Z)$ and $m_2(Z)$.

\section{Proof of Theorem~\ref{th:bias}: Bias on the template shape}

Now we prove the second theorem given in the paper "Template shape estimation".

\begin{theorem}\label{th:bias}
\textcolor{red}{The data $X_i$'s are generated with the model described in the paper "Template shape estimation", where the template shape $Y$ is a parameter and under the assumptions of Theorem 1. The template shape $Y$ is estimated with $\hat Y$, which is computed by the usual procedure described the paper.}

\textcolor{red}{In the regime of an infinite number of data $n \rightarrow + \infty$,} the asymptotic bias of the template's shape estimator $\hat Y$, with respect to the parameter $Y$, has the following Taylor expansion around the noise level $\sigma = 0$: 
\begin{equation}\label{eq:bias}
\text{Bias}(\hat Y,Y)= - \frac{\sigma^2}{2} H(Y) + \mathcal{O}(\sigma^4) + \epsilon(\sigma)
\end{equation}
where (i) $H$ is the mean curvature vector of the template shape's orbit \textcolor{red}{which represents the external curvature of the orbit in $M$, and (ii) $\epsilon$ is a function of $\sigma$ that decreases exponentially for $\sigma \rightarrow 0$.}
\end{theorem}

We compute the bias $\text{Bias}(Y,\hat{Y})$ of $\hat Y$ as an estimator of $Y$. In the following, we take a NCS at $Y$. In particular, the vector $\overrightarrow{YZ} = \text{Log}_YZ$ has coordinates written $z$.

The expectation of the distribution $f$ of shapes in $Q$ is $\hat Y$ by definition. The point $\hat Y$, expressed in a NCS at the template $Y$, gives $\text{Bias}(Y,\hat{Y})$, a tangent vector at $T_YM$ that indicates how much one has to shoot to reach the estimator $\hat Y$:
\begin{equation}
\text{Bias}(Y,\hat{Y}) = \text{Log}_Y\hat Y = \int_{Q} \overrightarrow{YZ} f(Z) dQ(Z).
\end{equation}

\noindent First, we take a ball of small radius $r$ in $Q$ and fix $r$. We split the integral:
\begin{flalign*}
\text{Bias}(Y,\hat{Y}) &= \int_{B_r^Q} \overrightarrow{YZ} f(Z) dQ(Z) + \int_{ \mathcal{C}_{B_r^Q}} \overrightarrow{YZ} f(Z) dQ(Z)&
\end{flalign*}

\noindent By the result of the preliminaries, adapted to $Q$, the right part is a function $\sigma \rightarrow \epsilon(\sigma)$ that is exponentially decreasing for $\sigma \rightarrow 0$.
\begin{flalign*}
\text{Bias}(Y,\hat{Y}) &= \int_{B_r^Q} \overrightarrow{YZ} f(Z) dQ(Z) + \epsilon(\sigma)&
\end{flalign*}

\subsection{Using the result of Theorem~\ref{th:pdf}}

\noindent We plug the expression of the density $f$ using Theorem~\ref{th:pdf}:
\begin{flalign*}
\text{Bias}(Y,\hat{Y}) &= \int_{B_r^Q}\overrightarrow{YZ} f_Q(Z)\left( F_0(Z) +\sigma^2 F_2(Z)+\mathcal{O}(\sigma^4)+  \epsilon(\sigma)\right) dQ(Z) + \epsilon(\sigma)\\
& = \int_{B_r^Q} \overrightarrow{YZ}f_Q(Z) \left( F_0(Z) +\sigma^2 F_2(Z)+\mathcal{O}(\sigma^4)+  \epsilon(\sigma)\right) dQ(Z) + \epsilon(\sigma) &
\end{flalign*}

\noindent Computing the $a$-coordinate of the bias:
\begin{flalign*}
\text{Bias}(Y,\hat{Y})^a = \int_{B_r^Q} z^af_Q(Z) \left( F_0(Z) +\sigma^2 F_2(Z)+\mathcal{O}(\sigma^4)+  \epsilon(\sigma)\right) dQ(Z) + \epsilon(\sigma) &&
\end{flalign*}

\noindent We are on a ball of small radius $r$ around $Y$. We write the Taylor expansions of the $F$'s terms around $z=0$:
\begin{align*}
F_0(Z) &= F_{00}(Y)+F_{01d}(Y)z^d+\mathcal{O}(||z||^2)\\
F_2(Z) &= F_{20}(Y)+F_{21d}(Y)z^d+\mathcal{O}(||z||^2)
\end{align*}

\noindent We replace these Taylor expansions in the expression of the bias:
\begin{flalign*}
\text{Bias}(Y,\hat{Y})^a & = \int_{B^Q_r} z^a f_Q(Z) \left( F_{00}(Y)+F_{01d}(Y)z^d +\mathcal{O}(||z||^2)\right)dQ(Z)\\
& \quad\quad+\int_{B^Q_r} z^a f_Q(Z)\left(\sigma^2 F_{20}(Y)+\sigma^2 F_{21d}(Y)z^d+\sigma^2\mathcal{O}(||z||^2)+\mathcal{O}(\sigma^4)+  \epsilon(\sigma)\right) dQ(Z) + \epsilon(\sigma)&
\end{flalign*}

\noindent Reorganizing:
\begin{flalign*}
\text{Bias}(Y,\hat{Y}) & = \int_{B_r^Q} z^a f_Q(Z) \left( F_{00}(Y)+F_{01d}(Y)z^d +\sigma^2 F_{20}(Y)+\sigma^2 F_{21d}(Y)z^d\right) dQ(Z) &\\
& \qquad\qquad\qquad +\int_{B_r^Q}\overrightarrow{YZ} f_Q(Z) \left(\mathcal{O}(||z||^2)+\sigma^2\mathcal{O}(||z||^2)+\mathcal{O}(\sigma^4)+  \epsilon(\sigma) \right)dQ(Z)\\
& \qquad\qquad\qquad + \epsilon(\sigma)
\end{flalign*}

\noindent By the dominated convergence theorem, we can put the inside $\epsilon(\sigma)$ outside the parenthesis.
\begin{flalign*}
\text{Bias}(Y,\hat{Y}) & = \int_{B_r^Q} z^a f_Q(Z) \left( F_{00}(Y)+F_{01d}(Y)z^d +\sigma^2 F_{20}(Y)+\sigma^2 F_{21d}(Y)z^d\right) dQ(Z) &\\
& \qquad\qquad\qquad + \int_{B_r^Q}\overrightarrow{YZ} f_Q(Z) \left(\mathcal{O}(||z||^2)+\sigma^2\mathcal{O}(||z||^2)+\mathcal{O}(\sigma^4)\right)dQ(Z)\\
& \qquad\qquad\qquad + \epsilon(\sigma)
\end{flalign*}

\noindent We develop the measure on $dQ$: $dQ(Z) = (1-\frac{1}{6}\text{Ric}(Y)_{bc}z^bz^c+\mathcal{O}(||z||^3))dz$:
\begin{flalign*}
\text{Bias}(Y,\hat{Y})^a & = \int_{B_r^Q} f_Q(Z) \left( F_{00}(Y)z^a+F_{01d}(Y)z^az^d \right)(1-\frac{1}{6}\text{Ric}(Y)_{bc}z^bz^c+\mathcal{O}(||z||^3))dz \\
& \qquad\qquad\qquad+\int_{B_r^Q} f_Q(Z) \left(\sigma^2 F_{20}(Y)z^a+\sigma^2 F_{21d}(Y).z^az^d\right)(1-\frac{1}{6}\text{Ric}(Y)_{bc}z^bz^c+\mathcal{O}(||z||^3))dz &\\
& \qquad\qquad\qquad + \int_{B_r^Q}f_Q(Z) z^a\left(\mathcal{O}(||z||^2)+\sigma^2\mathcal{O}(||z||^2)+\mathcal{O}(\sigma^4)\right)(1-\frac{1}{6}\text{Ric}(Y)_{bc}z^bz^c+\mathcal{O}(||z||^3))dz\\
& \qquad\qquad\qquad + \epsilon(\sigma)
\end{flalign*}

\noindent We develop:
\begin{flalign*}
\text{Bias}(Y,\hat{Y})^a & = \int_{B_r^Q} f_Q(Z) \left( F_{00}(Y)z^a+F_{01d}(Y)z^az^d +\sigma^2 F_{20}(Y)z^a+\sigma^2 F_{21d}(Y).z^az^d\right)dz &\\
&  \qquad\qquad\qquad + \int_{B_r^Q} f_Q(Z) \left( F_{00}(Y)z^a+F_{01d}(Y)z^az^d \right)(-\frac{1}{6}\text{Ric}(Y)_{bc}z^bz^c)dz \\
& \qquad\qquad\qquad + \int_{B_r^Q} f_Q(Z) \left(\sigma^2 F_{20}(Y)z^a+\sigma^2 F_{21d}(Y).z^az^d\right)(-\frac{1}{6}\text{Ric}(Y)_{bc}z^bz^c)dz \\
& \qquad\qquad\qquad +\int_{B_r^Q} f_Q(Z) \left( F_{00}(Y)z^a+F_{01d}(Y)z^az^d +\sigma^2 F_{20}(Y)z^a+\sigma^2 F_{21d}(Y).z^az^d\right)\mathcal{O}(||z||^3)dz \\
& \qquad\qquad\qquad + \int_{B_r^Q}f_Q(Z) z^a \left(\mathcal{O}(||z||^2)+\sigma^2\mathcal{O}(||z||^2)+\mathcal{O}(\sigma^4)\right)dz\\
& \qquad\qquad\qquad + \int_{B_r^Q}f_Q(Z) z^a\left(\mathcal{O}(||z||^2)+\sigma^2\mathcal{O}(||z||^2)+\mathcal{O}(\sigma^4)\right)(-\frac{1}{6}\text{Ric}(Y)_{bc}z^bz^c)dz\\
& \qquad\qquad\qquad + \int_{B_r^Q}f_Q(Z) z^a\left(\mathcal{O}(||z||^2)+\sigma^2\mathcal{O}(||z||^2)+\mathcal{O}(\sigma^4)\right)\mathcal{O}(||z||^3)dz\\
& \qquad\qquad\qquad + \epsilon(\sigma)
\end{flalign*}

\noindent We eliminate the odd terms that give 0 by skewsymmetry:
\begin{flalign*}
\text{Bias}(Y,\hat{Y})^a & = \int_{B_r^Q} f_Q(Z) \left( +F_{01d}(Y)z^az^d +\sigma^2 F_{21d}(Y).z^az^d\right)dz &\\
&  \qquad\qquad\qquad + \int_{B_r^Q} f_Q(Z) \left(F_{01d}(Y)z^az^d +\sigma^2 F_{21d}(Y).z^az^d\right)(-\frac{1}{6}\text{Ric}(Y)_{bc}z^bz^c)dz \\
& \qquad\qquad\qquad +\int_{B_r^Q} f_Q(Z) \left( F_{00}(Y)z^a +\sigma^2 F_{20}(Y)z^a\right)\mathcal{O}(||z||^3)dz \\
& \qquad\qquad\qquad + \int_{B_r^Q}f_Q(Z) z^a\left(\mathcal{O}(\sigma^4)\right)dz\\
& \qquad\qquad\qquad + \int_{B_r^Q}f_Q(Z) z^a\left(\mathcal{O}(\sigma^4)\right)(-\frac{1}{6}\text{Ric}(Y)_{bc}z^bz^c)dz\\
& \qquad\qquad\qquad + \int_{B_r^Q}f_Q(Z) z^a\left(\mathcal{O}(||z||^2)+\sigma^2\mathcal{O}(||z||^2)+\mathcal{O}(\sigma^4)\right)\mathcal{O}(||z||^3)dz\\
& \qquad\qquad\qquad + \epsilon(\sigma)
\end{flalign*}

\noindent We delete the terms that will give more than $\mathcal{O}(\sigma^{2})$ by integration (these are normalized moments) and put them in $\mathcal{O}(\sigma^4)$ which is the next order since there is no$\mathcal{O}(\sigma^3)$:
\begin{flalign*}
\text{Bias}(Y,\hat{Y})^a & = \int_{B_r^Q} f_Q(Z) F_{01d}(Y)z^az^ddz + \int_{B_r^Q}f_Q(Z) z^a\left(\mathcal{O}(\sigma^4)\right)dz+\mathcal{O}(\sigma^4)+ \epsilon(\sigma)&
\end{flalign*}

\noindent We gather the terms in $\mathcal{O}(\sigma^4)$:
\begin{flalign*}
\text{Bias}(Y,\hat{Y})^a & = \int_{B_r^Q} f_Q(Z) F_{01d}(Y)z^az^ddz +\mathcal{O}(\sigma^4) + \epsilon(\sigma) &
\end{flalign*}

\noindent We recognize the \textit{normalized} truncated moment of order $2$ in the $q$-dimensional Riemannian manifold $Q$, see Subsection~\ref{sec:riem}:
\begin{flalign*}
\text{Bias}(Y,\hat{Y})^a & = F_{01d}(Y) {\mathfrak{M}^Q_r}^{ad}(\sigma^2\mathbb{I}) +\mathcal{O}(\sigma^4) + \epsilon(\sigma)&
\end{flalign*}

\noindent We express it with respect to the non-truncated \textit{normalized} moment in $\mathbb{R}^q$, see Subsection~\ref{sec:riem}:
\begin{flalign*}
\text{Bias}(Y,\hat{Y})^a & = F_{01d}(Y) \left( \sigma^2{\mathcal{M}^Q}^{ad}(\mathbb{I})+\mathcal{O}(\sigma^4)-\epsilon(\sigma)\right)+\mathcal{O}(\sigma^4) + \epsilon(\sigma)&
\end{flalign*}

\noindent We gather the $\epsilon$'s and the $\mathcal{O}(\sigma^4)$:
\begin{flalign*}
\text{Bias}(Y,\hat{Y})^a & =F_{01d}(Y) {\sigma^2\mathcal{M}^Q_r}^{ad}(\mathbb{I}) +\mathcal{O}(\sigma^4) + \epsilon(\sigma)&
\end{flalign*}

\noindent We replace the \textit{normalized} 2nd order moment by its expression which is simply $\delta^{ad}$, see Subsection~\ref{sec:eucl}:
\begin{flalign*}
\text{Bias}(Y,\hat{Y})^a & = F_{01d}(Y)\sigma^2 \delta^{ad} +\mathcal{O}(\sigma^4) + \epsilon'(\sigma)&\\
& = F_{01}^a(Y)\sigma^2  +\mathcal{O}(\sigma^4)+ \epsilon(\sigma)
\end{flalign*}

\subsection{Computation of $F_{01}^a(Y)$: Taylor expansion of $F_0(Z)$ in the coordinate $z$}

The term $F_{01}^a(Y)$ is the first order coefficient in the Taylor expansion of $F_0(Z)$ around $Y$ in the coordinate $z$. Thus, we compute this Taylor expansion. 

We first compute the Taylor expression of $F_0(Z)$ using the coordinate $y = \overrightarrow{ZY}=\text{Log}_ZY$ in the NCS at $Z$. The expression of $F_0(Z)$ in Section~\ref{sec:th1} gives:
\begin{align*}
F_0(Z)& = \Theta_{Z,0}(1) = \sqrt{(2\pi)^p} \sqrt{\det\left((M_{ab}(Z)^{-1}\right)} 
\end{align*}

\noindent The previous subsections give:
\begin{align*}
M_{ab}(Z)=  -y_c h^c_{ab}(Z) +d_{ab}(Z) 
\end{align*}
and we replace $d_{ab}(Z)$ by the formula p.23 in \cite{Brewin2009} but keeping only the first order:
\begin{align*}
M_{ab}(Z)= \delta_{ab} -y_c h^c_{ab}(Y) +\mathcal{O}(||y||^2)
\end{align*}
So that:
\begin{align*}
M_{ab}(Z)^{-1}= \delta_{ab} +y_c h^c_{ab}(Y) +\mathcal{O}(||y||^2)
\end{align*}

\noindent We plug this in $F_0(Z)$:
\begin{align*}
F_0(Z)& = \sqrt{\det\left( \delta_{ab} +y_c h^c_{ab}(Y) +\mathcal{O}(||y||^2) \right)} \\
& = \sqrt{ 1 +y_c \text{Trace}(h^c_{ab}(Y)) +\mathcal{O}(||y||^2) }\\
& = \sqrt{ 1 +y_c H^c(Y) +\mathcal{O}(||y||^2) }\\
& = \left( 1 +\frac{1}{2}y_c H^c(Y) +\mathcal{O}(||y||^2) \right)
\end{align*}
where the trace of the second fundamental form is the external curvature vector $H(Y)$ by definition.

We convert this Taylor expansion in $y$, the coordinate of $Y$ in a NCS at $Z$, into a Taylor expansion in $z$, the coordinate of $Z$ in a NCS at $Y$. To express $y$ with respect to $z$, we consider the geodesic $\gamma_{ZY}(t)$ from $Z$ to $Y$ and the geodesic $\gamma_{YZ}(t)$ from $Y$ to $Z$. When parameterized by the arclength $s$, they are related as follows:
\begin{align*}
\gamma_{ZY}(s) & = \text{Exp}_Z\left( s \overrightarrow{ZY} \right)\\
& = \gamma_{YZ}(1-s)\\
& = \text{Exp}_Y\left((1-s)\overrightarrow{YZ}\right)
\end{align*}
Differentiating this relation gives:
\begin{align*}
D\text{Exp}_{Z}|_{s\overrightarrow{ZY}}.\overrightarrow{ZY} = - D\text{Exp}_{Y}|_{(1-s)\overrightarrow{YZ}}.\overrightarrow{YZ}
\end{align*}
Taking the relation at $s=0$:
\begin{align*}
D\text{Exp}_{Z}|_{0}.\overrightarrow{ZY} = - D\text{Exp}_{Y}|_{\overrightarrow{YZ}}.\overrightarrow{YZ}
\end{align*}
where $D\text{Exp}_{Z}|_{0} = Id$, so that:
\begin{align*}
y & = - D\text{Exp}_{Y}|_{z}.z
\end{align*}

\noindent We use the definition of the NCS at $Y$. When $\text{Exp}_Y(u) = U$, then $U$ has coordinates $u$ in the NCS at $Y$. So that: $D\text{Exp}_Y|_u.u = u$. This gives, with $u =z$:
\begin{align*}
y & = - z
\end{align*}
and we get the Taylor expansion of $F_0(Z)$ expressed in the coordinate $z$:
\begin{align*}
F_0(Z) & = \left( 1 -\frac{1}{2}z_c H^c(Y) +\mathcal{O}(||z||^2) \right)
\end{align*}

And we identify the term $F_{01}^a(Y)$ needed:
\begin{align*}
F_{01}^a(Y) & = -\frac{1}{2}H^a(Y)
\end{align*}
\subsection{Final result: Taylor expansion of $\text{Bias}(\hat Y, Y)$}

Replacing $F_{01}^a(Y)$ by its value computed above:
\begin{align*}
\text{Bias}(Y,\hat{Y})^a & = -\frac{1}{2}.H^a(Y)\sigma^2  +\mathcal{O}(\sigma^4)+ \epsilon(\sigma)
\end{align*}

\bibliographystyle{siamplain}
\makeatletter
\renewcommand\@biblabel[1]{#1. }
\makeatother
